\documentclass[10pt,twocolumn,letterpaper]{article}

\usepackage[pagenumbers]{cvpr} % To force page numbers, e.g. for an arXiv version
\usepackage{caption}
\usepackage{multirow}
\usepackage{url}
\usepackage{graphicx}
\usepackage{makecell}
\usepackage{array}
\usepackage{bbding}
\usepackage{algorithm,algorithmic}
\usepackage{xcolor,colortbl}
\usepackage{soul}
\usepackage{color}
\usepackage{adjustbox}
\usepackage{amsthm}
\usepackage{subcaption}
\usepackage[percent]{overpic}

\usepackage{etoc}
\etocdepthtag.toc{mtchapter}
\etocsettagdepth{mtchapter}{subsection}
\etocsettagdepth{mtappendix}{none}

\PassOptionsToPackage{[],dvipsnames}{xcolor}
\usepackage[dvipsnames]{xcolor}

\newtheorem{theorem}{Theorem}
\newtheorem{remark}{Remark}
\newtheorem{lemma}{Lemma}

\definecolor{cvprblue}{rgb}{0.21,0.49,0.74}
\usepackage[pagebackref,breaklinks,colorlinks,citecolor=cvprblue]{hyperref}

\def\paperID{832} % *** Enter the Paper ID here
\def\confName{CVPR}
\def\confYear{2024}
\title{\includegraphics[width=0.035\textwidth]{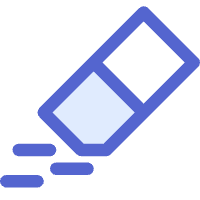} ERASE: Error-Resilient Representation Learning on Graphs\\ for Label Noise Tolerance \vspace{-1.5em}}
\author{
Ling-Hao Chen$^{1, 2 \dag *}$, Yuanshuo Zhang$^{2, 3 *}$, Taohua Huang$^{3}$, Liangcai Su$^{1}$, Zeyi Lin$^{2, 3}$, Xi Xiao$^{1\#}$,\\ Xiaobo Xia$^{4 \dag\#}$, Tongliang Liu$^{4}$\\
$^1$Tsinghua University, $^2$SwanHub.co, $^3$Xidian University, $^4$The University of Sydney \vspace{-0.3em}\\
{\small $^{\dag}$Project lead \ \ $^{*}$Equal Contribution \ \  $^{\#}$Corresponding author} \vspace{-0.4em}\\
{\small Email: }\ {\tt\small \{thu.lhchen, youngsoul0731, work.thhuang\}@gmail.com \ xiaoboxia.uni@gmail.com}\vspace{-0.4em}\\
{\small \centering \url{https://eraseai.github.io/ERASE-page}}
\vspace{-1.7em}
}

\newcommand{\myPara}[1]{\vspace{.05in}\noindent\textbf{#1}}
\begin{document}
\maketitle
\begin{abstract}
\vspace{-1.3em}
Deep learning has achieved remarkable success in graph-related tasks, yet this accomplishment heavily relies on large-scale high-quality annotated datasets. However, acquiring such datasets can be cost-prohibitive, leading to the practical use of labels obtained from economically efficient sources such as web searches and user tags. Unfortunately, these labels often come with noise, compromising the generalization performance of deep networks.
To tackle this challenge and enhance the robustness of deep learning models against label noise in graph-based tasks, we propose a method called ERASE (\underline{E}rror-\underline{R}esilient representation learning on graphs for l\underline{A}bel noi\underline{S}e toleranc\underline{E}). The core idea of ERASE is to learn representations with error tolerance by maximizing coding rate reduction. Particularly, we introduce a decoupled label propagation method for learning representations. Before training, noisy labels are pre-corrected through structural denoising. During training, ERASE combines prototype pseudo-labels with propagated denoised labels and updates representations with error resilience, which significantly improves the generalization performance in node classification. The proposed method allows us to more effectively withstand errors caused by mislabeled nodes, thereby strengthening the robustness of deep networks in handling noisy graph data. Extensive experimental results show that our method can outperform multiple baselines with clear margins in broad noise levels and enjoy great scalability. Codes are released at \url{https://github.com/eraseai/erase}. 
\vspace{-0.3em}
\end{abstract}    
\vspace{-1.5em}
\section{Introduction}
\label{sec:intro}

\begin{figure}[h]
\captionsetup{}
    \centering
    \includegraphics[width = 0.45\textwidth]{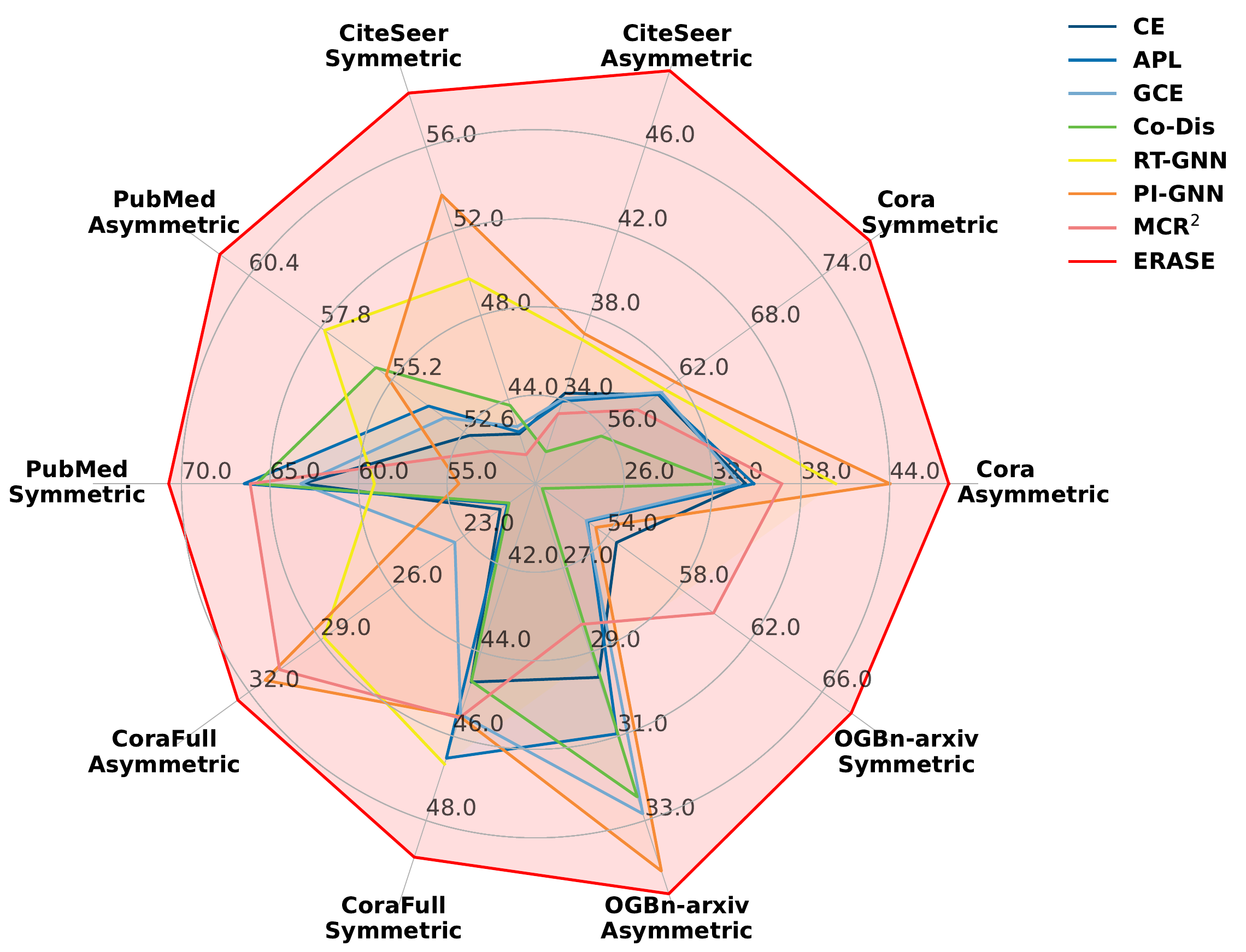}
    \vspace{-1.1em}
    \caption{ERASE beats everything on 5 node classification datasets in large ratio label noise scenarios ($\phi = 0.5$).}
    \label{fig:teaser}
    \vspace{-1.8em}
\end{figure}

Graphs are a pervasive structure found in various data analysis scenarios~\cite{pan2021autostg,li2020just,qiu2020gcc,chen2023anomman, li2021detectornet,he2020lightgcn,wang2019heterogeneous, shi2018heterogeneous,ling2021deep, you2020graph,sun2023think}. They appear in multiple real-world contexts such as social networks and biological networks~\cite{xu2018representation}. These graphs can encapsulate valuable information that goes beyond individual entities~\cite{xu2019powerful}. Accordingly, tasks involving the analysis of graphs, such as node classification, wield substantial influence in practical applications~\cite{du2023noise}. The success of deep learning models on node classification is largely attributed to the collection of large-scale datasets with high-quality annotated labels~\cite{zhou2019meta}. Particularly, with the data in such datasets, deep learning models first extract high-level features from nodes as representations and then complete classification with the representations~\cite{xie2022simmim,he2022masked,hou2022graphmae,liu2022graph,zeng2021contrastive}. Nevertheless, acquiring extensive high-quality annotated labels at a large scale is prohibitively expensive. Alternatively, labels acquired through web searches and user tags present a cost-effective solution~\cite{li2022selective,yi2019probabilistic,wu2020topological}. However, they come with the inherent drawback of being noisy~\cite{dai2021nrgnn,li2021unified}.

Noisy labels detrimentally impact the generalization performance of deep networks on graphs. This is because, for a mislabeled node, its label provides incorrect signals during the process of inducing latent representations for nodes~\cite{du2023noise}. Besides, the incorrect signals can be propagated along the topological edges, which influences the representation learning of other nodes. As a result, these corrupted representations subsequently lead to inaccurate decisions in subsequent node classification, adversely affecting generalization~\cite{dai2022towards,qian2023robust}. Therefore, it is imperative to learn robust representations of nodes for label noise tolerance, which is also the central focus of our study.

Recent works have explored diverse strategies to learn robust representations of nodes, \textit{e.g.}, involving a
small set of clean nodes for assistance~\cite{xia2020towards}, operating with class
labels derived from clustering node embeddings~\cite{zhang2020adversarial}, and selecting
clean labeled nodes from noisy ones~\cite{qian2023robust}. These works achieved commendable progress in enhancing the robustness of representation learning against label corruption. 
However, essentially, their representation learning is not designed or modeled by an \textit{not error-resilient} principle, which means that the representations are hard to gracefully withstand errors caused by mislabeled nodes. 
Specifically, as shown in~\cref{tab:training_correction}, it is hard to relieve the side effects of mislabeled nodes when combating label noise in node classification. Therefore, the learned representations from these works will be significantly influenced by the incorrect signals of mislabeled nodes. If the error resilience modeling in representation learning is not considered, the harm of those unprocessed incorrect signals will extremely jeopardize network robustness, which is never our desideratum. As a result, learning representations on graphs with an error-resilient objective is necessary. 

\begin{table}[t]
\centering
\footnotesize
\setlength\tabcolsep{2.5pt}
% \vspace{-1.4em}
% \hspace{-9em}
\begin{tabular}{c|ccc|ccc}
\toprule \specialrule{0em}{0.0pt}{0.0pt}
                       & \multicolumn{3}{c|}{Asymmetric}                  & \multicolumn{3}{c}{Symmetric}                    \\ \cline{2-7} 
\multirow{-2}{*}{Cora} & 0.3            & 0.4            & 0.5            & 0.3            & 0.4            & 0.5            \\ \hline
CE             & 39.71  & 26.96	& 20.26	& 48.21	& 36.90 & 31.58\\
\rowcolor[HTML]{FFFFFF} 
APL            & 39.71	& 27.50	& 19.74	& 50.71	& 37.62	& 31.23\\
GCE            & 39.14	& 30.89	& 22.89	& 50.00	& 38.57	& 29.65\\
CoDis          & 44.29	& 31.43	& 19.21	& 37.32	& 40.54	& 37.26\\
MCR$^2$        & 16.57	& 16.07	& 14.08	& 39.64	& 21.19	& 14.56\\
RT-GNN          & \underline{57.91}   &\underline{53.32}   & \underline{40.35} & 43.75 & 40.55   & \underline{39.82}  \\
PI-GNN         &41.46   & 37.03 & 19.69 & \underline{57.69} & \underline{46.47} & 28.04\\
\rowcolor[HTML]{E7FBFF} 
ERASE          &\makecell{\textbf{84.86} \\(\textcolor{red}{+26.95})}   &\makecell{\textbf{75.71}  \\(\textcolor{red}{+22.39}) }  &\makecell{\textbf{54.61} \\(\textcolor{red}{+14.26})}   &\makecell{\textbf{84.29}  \\(\textcolor{red}{+26.60}) }    &\makecell{\textbf{74.29}  \\(\textcolor{red}{+27.82})}     &\makecell{\textbf{62.46}  \\(\textcolor{red}{+22.64})} \\ 
\specialrule{0em}{0.0pt}{0.0pt} \bottomrule
\end{tabular}

\caption{Correction rate of mislabeled nodes. ERASE enjoys a significant margin over the best baseline.}
\vspace{-3.7em}
\label{tab:training_correction}
\end{table}

In this paper, to improve the robustness of representation learning in node classification, and address the issue of previous works, we propose error-resilient representation learning on graphs for label noise tolerance (\textit{aka} \textbf{ERASE}).
In essence, ERASE can enable representation learning on graphs to greatly withstand those unprocessed incorrect signals brought by mislabeled nodes. The improved robustness of representations and subsequent better generalization in node classification is hence achieved. Empirically, as shown in~\cref{tab:training_correction}, ERASE enjoys a significant margin over the best baseline on the correction rate of mislabeled nodes, indicating a better error resilience ability than previous attempts. 

Technically, we learn the robust representations by maximizing the coding rate between the whole dataset and each class (coding rate reduction), whose measurement is error-resilient~\cite{Ma_Derksen_Hong_Wright_2007} under mild conditions. As the existence of mislabeled nodes makes it hard to estimate the coding rate reduction precisely, we propose to use a decoupled label propagation method to tackle the problem.
 
In ERASE, the label propagation is decoupled into two phases. 1) \textit{Before training}, to avoid the negative impact of label noise, label propagation is used for structural denoising. In this phase, we pre-correct the noisy training labels via the topology of the graph by label propagation. 2) \textit{During training}, to make full use of error-resilient representations, we combine the prototype pseudo labels obtained by the representations and propagated structural denoised labels as semantic labels, which are used for the second phase of label updating. With the decoupled label propagation, our method will learn the error-resilient node representations via estimating precise coding rate reduction. Besides, enjoying the property of maximizing coding rate reduction, learned representations between different classes are approximately orthogonal, which can be easily used to predict labels via both linear or nonlinear classifiers and achieve great robustness.

The main contributions of this paper are summarized in three aspects. 
1) We provide a novel perspective to learning label-noise-tolerance presentations on graphs via optimizing an error-resilient training objective for the first time. 
2) We propose a decoupled label propagation method to provide denoised labels and semantic labels with graph structural prior. The decoupled label propagation helps to provide trustworthy labels to learn error-resilient node representations. 
3) Extensive experiments show our method outperforms baselines and enjoys great scalability, especially when the label noise ratio is large. 

\iffalse
\begin{itemize}
    \item {
        We provide a novel perspective to learning label-noise-tolerance presentations on graphs via optimizing an error-tolerant training objective. 
    }
    \item {
        We propose a two-stage decoupled label propagation method to provide denoised labels and semantic labels with graph structural prior. The two-stage decoupled label propagation helps to provide trustworthy labels to learn error-resilient node representations. 
    }
    \item {
        Extensive experiments show our method outperforms baselines and enjoys great scalability, especially when the noise ratio is large.

    }
\end{itemize}
\fi

% \input{sec/2_relatedwork}
\section{Preliminaries}
\label{sec:preliminaries}

\myPara{Notations.}
We begin by clarifying notations. We denote a graph as $\mathcal{G} = \left\{ \mathcal{V},\mathcal{E}\right\}$, where $\mathcal{V}$ and $\mathcal{E}$ represent the node set and edge set respectively. In the graph $\mathcal{G}$ with $N$ nodes, $\boldsymbol{A}\in \left\{0, 1\right\}^{N \times N}$ is the adjacency matrix. For an attributed graph, $\boldsymbol{X}=[\boldsymbol{x}_1,\boldsymbol{x}_2,\cdots ,\boldsymbol{x}_N]^{\top}\in \mathbb{R}^{N\times d_0}$ contains $N$ nodes with $d_0$-dimension features. We denote $\boldsymbol{Z} = \mathtt{Enc}_{\Theta}(\boldsymbol{X},\mathcal{E})=[\boldsymbol{z}_1,\boldsymbol{z}_2,\cdots ,\boldsymbol{z}_N]^{\top}\in \mathbb{R}^{N\times d}$ as the latent representations, where $\mathtt{Enc}_{\Theta}(\cdot)$ is a graph neural network encoder parameterized by learnable parameters $\Theta$.

\myPara{Problem formulation.} To evaluate the robustness of our method, we take the semi-supervised node classification on the graph as the pretext task which can be defined as follows. 
We split all the nodes $\mathcal{V}$ into three sets $\mathcal{V} ^{\text{train}}$, $\mathcal{V}^{\text{valid}}$, and $\mathcal{V}^{\text{test}}$ for training, validation, and testing respectively. The ground truth labels and the corrupted labels are denoted as $\boldsymbol{Y}$ and $\boldsymbol{\tilde{Y}}$, all of which are divided as $K$ classes. When learning representations in noisy label scenarios, only corrupted labels $\tilde{\boldsymbol{Y}}^{\text{train}}$ of $\mathcal{V}^{\text{train}}$ are available during the training process. 
Besides, all the features $\boldsymbol{X}$ and adjacency matrix $\boldsymbol{A}$ are also available to update parameters $\Theta$ of $\mathtt{Enc}_{\Theta}(\cdot)$. The validation set and test set are used for early stopping and comparison respectively.
% \liangcai{This sentence may be unnecessary}
% To evaluate the robustness of our method, we take the semi-supervised node classification as the pretext task and perform the comparison on test set classification accuracy. \liangcai{I suppose that we should start by saying that we intend to evaluate on a semi-supervised task and then introduce it.}

% \textbf{Graph Neural Networks}

\myPara{Coding rate reduction.} 
The coding rate is defined as the
average number of bits needed per example~\cite{Ma_Derksen_Hong_Wright_2007}. Given representations $\boldsymbol{Z}$ and a precision $\epsilon$, the average coding rate~\cite{Ma_Derksen_Hong_Wright_2007} of each example is estimated by,
\vspace{-0.5em}
\begin{equation}
    R(\boldsymbol{Z},\epsilon)\doteq \frac{1}{2}\displaystyle \log \det (\boldsymbol{I} + \frac{d}{N\epsilon^2}\boldsymbol{Z}^{\top}\boldsymbol{Z}),
\vspace{-0.5em}
\end{equation}
where $d$ is the dimension of learned representations. Here, we can also use the volume $vol(\boldsymbol{Z})$ to measure how large the space spanned by these vectors, where the $R$ and $vol(\boldsymbol{Z})$ are with a positive correlation. For a detailed definition of the volume $vol(\cdot)$, please refer to~\cref{sec:definition of NVTR}.
% The label corruption can be treated as the problem that the label of an example remains unchanged and a shift of representation on $\boldsymbol{z}$ occurs. That is to say $\tilde{\boldsymbol{z}} = \boldsymbol{z} + \delta$, where $\delta$ is the shift vector. Following the proof in~\cite{Ma_Derksen_Hong_Wright_2007}, we have:

\section{How Can Error-resilient Objective Help Label Noise Tolerance?}

Without loss of generality, the label corruption can be treated as the problem that a shift on $\boldsymbol{z}$ occurs when we utilize the corrupted labels for training. That is to say $\tilde{\boldsymbol{z}} = \boldsymbol{z} + \delta$, where $\delta$s are the shift vectors. Following the proof in~\cite{Ma_Derksen_Hong_Wright_2007}, we have, 
\vspace{-0.5em}
\begin{remark}
    
    If the error $\delta$ of $\tilde{\boldsymbol{z}} = \boldsymbol{z} + \delta$ and $\boldsymbol{w} \sim \mathcal{N}(\boldsymbol{0}, \frac{\epsilon^2}{d}\boldsymbol{I})$ satisfies the inequality $vol(\delta)\leq vol(\boldsymbol{w})$, $R(\tilde{\boldsymbol{Z}}, \epsilon)$ measures the coding rate subject to distortion $\epsilon$. 
    \label{rm1}

\end{remark}

For data with $K$ classes, the average coding rate per example of all classes is, 
\vspace{-0.4em}
\begin{equation}
\footnotesize
    R^C(\boldsymbol{Z},\epsilon \vert \mathbf{\Pi} )\doteq \sum_{j=1}^{K} \frac{\mathrm{tr}(\mathbf{\Pi}_j)}{2N}\displaystyle \log \det (\boldsymbol{I} + \frac{d}{\mathrm{tr}(\mathbf{\Pi}_j)\epsilon^2}\boldsymbol{ Z}^{\top} \mathbf{\Pi}_j \boldsymbol{Z}),
\vspace{-0.4em}
\end{equation}
where $\mathbf{\Pi}_j \in \mathbb{R}^{N\times N}$ is a diagonal matrix representing the membership 
$\mathbf{\Pi}_{j,i,i}\in \{0,1\}$. Here, $\mathbf{\Pi}$ lies in $\Omega \doteq\left\{\mathbf{\Pi} \mid \mathbf{\Pi}_j \geq \boldsymbol{0}, \mathbf{\Pi}_1+\cdots+\mathbf{\Pi}_K=\boldsymbol{I}\right\}$. $\mathbf{\Pi}_{j,i,i}=1$ if example $i$ belongs to the class $j$, otherwise $\mathbf{\Pi}_{j,i,i}=0$. In contrast, smaller $R^C$ means more compact representations.

According to~\cite{NEURIPS2020_6ad4174e}, more discriminative and orthogonal representations require larger $R$ and smaller $R^C$. Based on this design principle, the volume between the different classes should be large enough, which can be measured by the coding rate reduction, 
\vspace{-0.5em}
\begin{equation}
\label{eq:coding rate within a class}
    \Delta R(\boldsymbol{Z}, \mathbf{\Pi}, \epsilon) \doteq R(\boldsymbol{Z}, \epsilon)-R^C(\boldsymbol{Z}, \epsilon \mid \mathbf{\Pi}). 
\vspace{-0.7em}
\end{equation}
The larger the coding rate reduction value is, the more discriminative and orthogonal the learned representations are. To obtain better representations, we train the GNN encoder to maximize the coding rate reduction with $\boldsymbol{\Pi}$ and $\epsilon$ given, 
\vspace{-0.5em}
\begin{equation}
    \label{eq:mcr2}
    \small
    \max _{\Theta, \mathbf{\Pi}} \Delta R(\boldsymbol{Z}(\Theta), \mathbf{\Pi}, \epsilon) = R(\boldsymbol{Z}(\Theta), \epsilon)-R^C(\boldsymbol{Z}(\Theta), \epsilon \mid \mathbf{\Pi}),
\vspace{-0.7em}
\end{equation}
where $\Theta$ is the parameters of the encoder $\mathtt{Enc}_{\Theta}(\cdot)$. Note that only training labels are available during training, which makes the membership matrix $\mathbf{\Pi}$ as an optimization term in the training process.

 In~\cref{eq:mcr2}, the $R(\boldsymbol{Z}(\Theta), \epsilon)$ term measures the coding rate of full data, which is irrelevant to label noise. However, $R^C(\boldsymbol{Z}(\Theta), \epsilon| \mathbf{\Pi})$ is dependent on the partition $\mathbf{\Pi}$, which is calculated by the given noisy labels. We denote  $\boldsymbol{Z}=\mathtt{Enc}_{\Theta}(\boldsymbol{X},\boldsymbol{A})$ and $\tilde{\boldsymbol{Z}}=\mathtt{Enc}_{\tilde{\Theta}}(\boldsymbol{X},\boldsymbol{A})$ as the latent representation learned by $\boldsymbol{Y}^{\text{train}}$ and $\tilde{\boldsymbol{Y}}^{\text{train}}$ respectively. According to~\cref{rm1}, if $vol(\delta)\leq vol(\boldsymbol{w})$ holds, $R^C(\boldsymbol{Z}(\Theta), \epsilon \mid \mathbf{\Pi})$ can be measured precisely in noisy label scenarios and the coding rate reduction also can be measured robustly. The experimental supports are shown in~\cref{sec:Abalation_study}. 
\vspace{-0.6em}
\section{Methodology}
\label{section 4}

Before delving into technical details, we present the system overview of ERASE, which is shown in~\cref{fig:system}. In this section, we will introduce the two-phase decoupled label propagation method that is designed to reduce the negative effect of label noise at first. Before training~(\cref{fig:system}(b)), the label propagation~\cite{zhu2002learning,10.1145/3442381.3449927,10.1145/3490478,zhou2003learning,Gong_Tao_Liu_Liu_Yang_2017,Zhang_Zhang_Zhao_Ye_Zhang_Wang_2020,Park_Jeong_Kim_Kim_2020,ma2021unified} method is introduced for label denoising. During training~(\cref{fig:system}(c)), we obtain semantic labels by combining denoised labels and prototype pseudo labels. Benefiting from the propagated semantic labels, we can estimate the coding rate reduction more precisely. \textbf{Importantly, the coding rate reduction is error-resilient, whose design principle is different from the traditional classification training objective.}

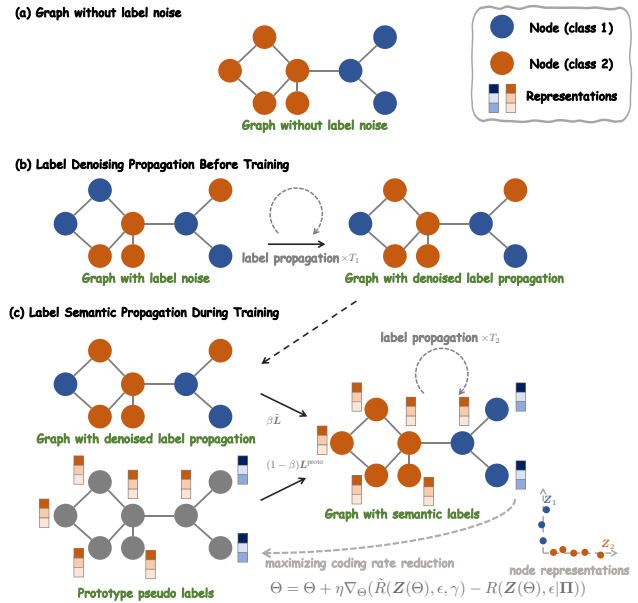
\begin{figure}[h]
    \vspace{-0.7em}
    \centering
    \begin{overpic}[scale=0.248]{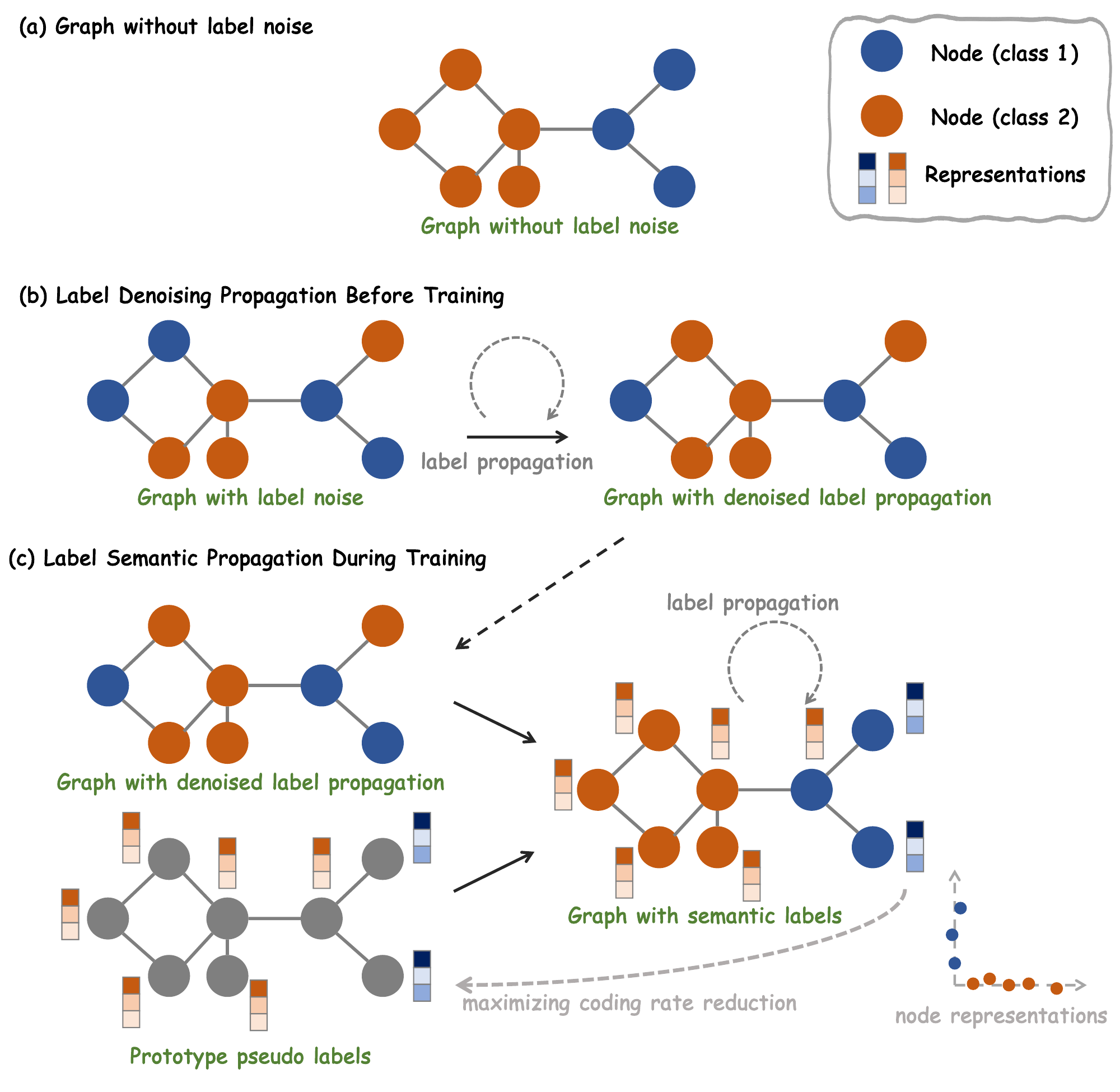}
        \put(52.2, 55.8){
            \textcolor[RGB]{107,107,107}{\scalebox{.4}{$\times T_1$ }}
        }
        \put(41, 3.2){
            \textcolor[RGB]{107, 107, 107}{\scalebox{.6}{$\Theta = \Theta + \eta \nabla_{\Theta}(\tilde{R}(\boldsymbol{Z}(\Theta), \epsilon,\gamma) 
- R(\boldsymbol{Z}(\Theta), \epsilon|\mathbf{\Pi}))$}}
            }
        \put(74, 43.1){
            \textcolor[RGB]{107,107,107}{\scalebox{.4}{ $\times T_2$ }}
        }
        \put(40, 23){
            \textcolor[RGB]{107,107,107}{\scalebox{.4}{ $(1-\beta)\boldsymbol{L}^{\text{proto}}$ }}
}
            \put(40, 30){
            \textcolor[RGB]{107,107,107}{\scalebox{.4}{ $\beta \tilde{\boldsymbol{L}}$ }}
        }
        \put(83.8, 17.5){
            \textcolor[RGB]{53,85,150}{\scalebox{.4}{ $\boldsymbol{Z}_1$ }}
        }
        \put(93, 10.5){
            \textcolor[RGB]{194,90,26}{\scalebox{.4}{ $\boldsymbol{Z}_2$ }}
        }
    \end{overpic}
    \vspace{-1.5em}
    \caption{The procedure of decoupled label propagation and model training. (a) A graph without label noise. (b) For a graph with label noise, before training, we perform label denoising on all training nodes. (c) During training, we estimate the prototype labels with learned representations $\boldsymbol{Z}$ and combine them with denoised labels to obtain the semantic labels. We update representations with semantic labels via maximizing coding rate reduction.}
    %language prior}
    \vspace{-1.8em}
    \label{fig:system}
\end{figure}

\vspace{-0.4em}
\subsection{Decoupled Label Propagation}
\label{sec:estimation of senmantic-structural label}

% \myPara{Decoupled label propagation.} 
In ERASE, the label propagation is decoupled into two phases. 
\begin{itemize}[leftmargin=*]
    \item \textit{Before training}, to avoid the negative impact of label noise on unlabeled nodes (\textit{aka} validation and testing nodes), label propagation is used for structural denoising, in which phase we pre-correct the labels with noise via the topology of the graph, \textit{e.g.}, the adjacency matrix.
    \item \textit{During training}, label propagation is conducted on the whole graph to provide pseudo-labels for all nodes. Besides, we additionally introduce cosine similarity logits learned by maximizing coding rate reduction to provide semantic pseudo labels for updating representations. 
    % \vspace{-1em}
\end{itemize}

\subsubsection{Label Denoising Propagation Before Training} 
\label{sec:dp}

\vspace{-0.6em}
 We firstly initialize the class distribution $\tilde{\boldsymbol{L}}^{(0)}$ of all nodes with training label $\tilde{\boldsymbol{Y}}^{\text{train}}$ and unlabelled data are zero-padded. To avoid the negative impact of label noise on unlabeled nodes, we introduce the masked adjacency matrix to restrict labels to be propagated only among training nodes. Keep it simple, we denote $\mathbf{m}\in \{0, 1\}^{N}$ as the masking vector for distinguishing training nodes and others, where $1$ and $0$ in $\mathbf{m}$ mean training nodes and others respectively. The masked adjacency matrix is obtained by $\boldsymbol{A}^{M} = \boldsymbol{A}\odot \mathbf{m}\mathbf{m}^{\top}$, where $\odot$ denotes the Hadamard product of matrices. Following~\cite{zhou2003learning}, we normalize $\boldsymbol{A}^{M}$ as $\tilde{\boldsymbol{A}}^{M}=\boldsymbol{D}^{-\frac{1}{2}}\boldsymbol{A}^{M}\boldsymbol{D}^{-\frac{1}{2}}$, where $\boldsymbol{D}$ is a diagonal matrix whose diagonal element $ \boldsymbol{D}_{ii}=\sum_{j}{\boldsymbol{A}}^{M}_{ij}$. We perform  structural denoising propagation as,  
\vspace{-0.4em}
 \begin{equation}
 \label{eq:denoising propagation}
 \fontsize{8.5}{1}\selectfont
 \tilde{\boldsymbol{L}}^{(t+1)} = \left((1-\alpha_{1})\tilde{\boldsymbol{A}}^{M}+\alpha_{1}\boldsymbol{I}\right) \tilde{\boldsymbol{L}}^{(t)}, \ t=0,1,\cdots,T_1-1,
\vspace{-0.4em}
 \end{equation}
where $T_1$ is the total propagation depth, and $\alpha_1$ is the coefficient controlling the usage of original labels. 
In the sequel, to keep it simple, we denote $\tilde{\boldsymbol{L}}^{(T_1)}_{ij}$ as $\tilde{\boldsymbol{L}}_{ij}$.

\subsubsection{Label Semantic Propagation During Training} 
\label{sec:sp}

For the semi-supervision setting in learning node representations on the graph, we cannot obtain labels of all nodes. As a result, when calculating the coding rate for each class $R^{C}$, we cannot fully utilize node $\boldsymbol{X}$ features, making it hard to estimate $R^{C}$ precisely. 

Fortunately, inspired by pseudo-labeling techniques, we assign pseudo-labels to unlabeled nodes. Accordingly, available nodes can be scaled from $\boldsymbol{X}^{\text{train}}$ to $\boldsymbol{X}$, estimating coding rate reduction more precisely.
% \vspace{-0.4em}
\begin{lemma}
\label{lam1}
    If the ambient space is large enough, the optimal solution $\boldsymbol{Z}^{*} = \{\boldsymbol{Z}_{1}^{*}, \boldsymbol{Z}_{2}^{*}, \cdots, \boldsymbol{Z}_{K}^{*}\}$ of $K$ classes in~\cref{eq:mcr2} satisfies $\boldsymbol{Z}_{i}^{*\top} \boldsymbol{Z}_{j}^{*}=\boldsymbol{0}$ for $i \neq j$ (c.f. the proof in~\cite{NEURIPS2020_6ad4174e}).
\end{lemma}
As shown in~\cref{lam1}, the training objective is to make the learned representations discriminative, revealing that cosine similarity is an appropriate measurement of node similarity. Therefore, we introduce the cosine similarity logits as the prototype pseudo labels during training.

\myPara{Prototype pseudo labels.} 
The prototype~\cite{PCL} is described as ``an exemplar that represents a collection of instances with similar semantic attributes''. For all nodes in the graph, we estimate the prototypes $\mathcal{C} = \left\{ \ \boldsymbol{c}_{j} \right\}_{j=1}^{K}$ of each class as, 
\vspace{-0.5em}
\begin{equation}
\label{eq:prototype}
    \boldsymbol{c}_{j} =  {\textstyle \sum_{\boldsymbol{v}_i \in \mathcal{S}_j}}\boldsymbol{z}_i / {\mid \mathcal{S}_j \mid},
\vspace{-0.5em}
\end{equation}
where $\small \mathcal{S}_j =  \{\boldsymbol{v}_i|\mathop{\arg\max}\limits_{k}\tilde{\boldsymbol{L}}_{ik} = j\}$ is the subset of $\mathcal{V}^{\text{train}}$ split by membership of the class $k$ and $\tilde{\boldsymbol{L}}$ is the label obtained by label propagation in~\cref{eq:denoising propagation}. Given the representation $\boldsymbol{z}$ and the prototypes $\mathcal{C}$, the prototype pseudo label $\boldsymbol{L}^{\text{proto}}$ is calculated by, 
\vspace{-0.5em}
\begin{equation}
\label{L_proto_eq}
    \boldsymbol{L}^{\text{proto}}_{ij} = \exp(s(\boldsymbol{z}_{i},\boldsymbol{c}_{j}))/ {\textstyle \sum_{k=1}^{K}}\exp(s(\boldsymbol{z}_{i},\boldsymbol{c}_{k})),
\vspace{-0.5em}
\end{equation}
where $s(\boldsymbol{z}_{i},\boldsymbol{c}_{j})=\frac{\boldsymbol{z}_{i}^{\top} \boldsymbol{c}_{j}}{\Vert \boldsymbol{z}_{i} \Vert_2 \Vert \boldsymbol{c}_{j} \Vert_2}$ is the cosine similarity between node representation $\boldsymbol{z}_{i}$ and prototype $\boldsymbol{c}_{j}$. The semantic pseudo label $\boldsymbol{L}^{\text{proto}}$ contains class semantics of $\boldsymbol{Z}$. If the value of $\boldsymbol{L}^{\text{proto}}_{ij}$ is closer to 1,  $\boldsymbol{z}_i$ is with larger probability to lie in the class $j$. 

\begin{theorem}
    \label{cTc}
    If the ambient space is large enough, the clustering center $\boldsymbol{c}^{*} = \{\boldsymbol{c}_{1}^{*}, \boldsymbol{c}_{2}^{*}, \cdots, \boldsymbol{c}_{K}^{*}\}$ of the optimal solution $\boldsymbol{Z}^{*}$ in~\cref{eq:mcr2} satisfies $\boldsymbol{c}_{i}^{*\top} \boldsymbol{c}_{j}^{*}=0$ for $i \neq j$.
\end{theorem}
\begin{proof}
    If the ambient space is large enough, $\boldsymbol{Z}^{*}$ satisfies $\boldsymbol{Z}_{i}^{*\top} \boldsymbol{Z}_{j}^{*}=\boldsymbol{0}$ for $i \neq j$ (\cref{lam1}). Therefore, $\boldsymbol{c}_{i}^{*\top}\boldsymbol{c}_{j}^{*} = (\sum_{s=1}^{m_j}\boldsymbol{z}_{s}^{* \top}) \sum_{s=1}^{m_j}\boldsymbol{z}_{s}^{*} = \sum_{s=1}^{m_j}\sum_{l=1}^{m_j}\boldsymbol{z}_{s}^{* \top} \boldsymbol{z}_{l}^{*} = 0$ holds when $i \neq j$.
\end{proof}

\begin{remark}
    According to~\cref{cTc}, when coding rate reduction reaches the optimal value, the mean vectors of each class are orthogonal. This makes the representations discriminative enough between different classes, leading to higher confidence of learned prototype pseudo labels $\boldsymbol{L}^{\text{proto}}$. 
\end{remark}

\myPara{Semantic labels.} The prototype label $\boldsymbol{L}^{\text{proto}}$ provides the prototype similarity knowledge, and the $\tilde{\boldsymbol{L}}$ provides the propagated training labels. During training, to make full use of both labels, we introduce the semantic logits ${\boldsymbol{L}}^{\text{s}}$ which combines class prototypes and denoised labels together for semantic estimation,
\vspace{-0.4em}
\begin{equation}
    {\boldsymbol{L}}^{\text{s}} = (1-\beta)\boldsymbol{L}^{\text{proto}} + \beta \tilde{\boldsymbol{L}},
\vspace{-0.4em}
\end{equation}
where $\boldsymbol{L}^{\text{proto}}$ is obtained by ~\cref{L_proto_eq}. Accordingly, we define the semantic estimation operation process (phase 2) as $\boldsymbol{L}^{\text{s}}=\texttt{SE}(\boldsymbol{Z},\tilde{\boldsymbol{L}})$. 

\textit{During training}, the label matrix of label propagation is initialized by $\boldsymbol{L}^{(0)}=\boldsymbol{L}^{\text{s}}$. We further propagate the semantic labels as,
\begin{equation}
    \small
     {\boldsymbol{L}}^{(t+1)} = (1-\alpha_{2})\tilde{\boldsymbol{A}}{\boldsymbol{L}}^{(t)}+\alpha_{2} {\boldsymbol{L}}^{(t)}, \quad t=0,1,\cdots,T_2-1,
\end{equation}
where $\tilde{\boldsymbol{A}}$ is the normalized matrix of adjacency matrix $\boldsymbol{A}$, and the $T_2$ is the propagation depth. The propagated label matrix $\boldsymbol{L}^{(T_2)} \in \mathbb{R}^{N\times K}$ contains reliable semantic information covering the whole graph, and we can easily obtain semantic labels of all nodes by converting $\boldsymbol{L}=\boldsymbol{L}^{(T_2)}$ to one-hot encoding (\cref{eq:argmax}).

\begin{figure}[!t]
\vspace{-1em}
\begin{algorithm}[H]
    \renewcommand{\algorithmicrequire}{\textbf{Input:}}
	\renewcommand{\algorithmicensure}{\textbf{Output:}}
	\caption{Training procedure of ERASE.}
    \label{algorithm:Training procedure of ERASE}
    \begin{algorithmic}
        \REQUIRE A graph $\mathcal{G} = \left\{ \mathcal{V},\mathcal{E}\right\}$ with adjacency matrix $\boldsymbol{A}$,  features $\boldsymbol{X}$, noisy labels $\tilde{\boldsymbol{Y}}^{\text{train}}$; A GNN encoder initialized $\mathtt{Enc}_{\Theta}(\cdot)$; learning rate $\eta$; scaling weight $\gamma$; coding rate precision $\epsilon$, steps of label propagation $T_1, T_2$; coefficients $\alpha_1, \alpha_2 $ and $\beta$ and maximum iterations $I_{\max}$.
        \ENSURE Robust encoder $\mathtt{Enc}_{\Theta}(\cdot)$.
        \STATE Initialize $\tilde{\boldsymbol{L}}^{0};$\\
        \texttt{\small \textcolor{teal}{// label propagation before training}} 
        \FOR {$t=0,1,\cdots,T_1-1$} 
            \STATE $\tilde{\boldsymbol{L}}^{(t+1)} = (1-\alpha_{1})\tilde{\boldsymbol{A}}^{M}\tilde{\boldsymbol{L}}^{(t)}+\alpha_{1} \tilde{\boldsymbol{L}}^{(t)};$
        \ENDFOR\\
        \texttt{\small \textcolor{teal}{// training iterations}} 
        \FOR {$I=1,2,\cdots,I_{\max}$}
            \STATE $\boldsymbol{Z} = \mathtt{Enc}_{\Theta}(\boldsymbol{X},\mathcal{E});$
            \STATE $\boldsymbol{L}^{s}=\texttt{SE}(\boldsymbol{Z},\tilde{\boldsymbol{L}}^{T_1});$ \ \ \texttt{\small \textcolor{teal}{// get semantic labels}}
             \STATE Initialize ${\boldsymbol{L}}^{0}=\boldsymbol{L}^{s};$\\
            \texttt{\small \textcolor{teal}{// label propagation during training}} 
            \FOR{$t=0,1,\cdots,T_2-1$}
            \STATE     $ {\boldsymbol{L}}^{(t+1)} = (1-\alpha_{2})\tilde{\boldsymbol{A}}{\boldsymbol{L}}^{(t)}+\alpha_{2} {\boldsymbol{L}}^{(t)};$
            \ENDFOR
            \STATE $\boldsymbol{L} = \boldsymbol{L}^{T_2};$\\
            \texttt{\small \textcolor{teal}{// Update $\Theta$ by Eq.~(\ref{eq:training objective})}}
            \STATE $\Theta = \Theta + \eta \nabla_{\Theta}(\tilde{R}(\boldsymbol{Z}(\Theta), \epsilon,\gamma) 
- R(\boldsymbol{Z}(\Theta), \epsilon|\mathbf{\Pi}))$. 
        \ENDFOR
    \end{algorithmic} 
\end{algorithm}
\vspace{-3.0em}
\end{figure}

\vspace{-0.5em}
\subsection{Training Objective and Node Classification}
\vspace{-0.5em}
\label{sec:traininginf}

\myPara{Model training.} Given semantic labels $\boldsymbol{L}$, we can obtain the semantic membership of all nodes as follows, 
    \vspace{-0.2em}
\begin{equation}
\label{eq:argmax}
\mathbf{\Pi}_{j,i,i} = 
\begin{cases}
    1, \mathop{\arg\max}\limits_{k}{\boldsymbol{L}}_{ik} = j\\
    0, \mathop{\arg\max}\limits_{k}{\boldsymbol{L}}_{ik} \neq j
\end{cases}.
    \vspace{-0.2em}
\end{equation}
However, attempting to derive a sound representation from unlabelled data via optimizing $\Delta R$ directly will be an excessively ambitious undertaking.
Therefore, we scale the coding rate of all nodes during training. Here we use simplified strategy for scaling\footnote{We simplify $\tilde{R}(\boldsymbol{Z},\epsilon,\gamma_1,\gamma_2)\doteq \frac{1}{2\gamma_1}\displaystyle \log \det (\boldsymbol{I} + \frac{d\gamma_2}{N\epsilon^2}\boldsymbol{ZZ}^{\top}) $ by only controlling the weight of $\Delta R(\boldsymbol{Z},\epsilon)$.} as, 
    \vspace{-0.4em}
\begin{equation}
\label{eq:revised average coding rate}
    \tilde{R}(\boldsymbol{Z},\epsilon,\gamma)=\frac{\gamma}{2}\displaystyle \log \det (\boldsymbol{I} + \frac{d}{N\epsilon^2}\boldsymbol{Z}^{\top}\boldsymbol{Z}),
    \vspace{-0.4em}
\end{equation}
where $\boldsymbol{Z} = \mathtt{Enc}_{\Theta}(\boldsymbol{X},\mathcal{E})$, $\gamma$ is the hyper-parameter, and $\mathtt{Enc}_{\Theta}(\cdot)$ is a GNN encoder. We train our GNN following the objective, 
    \vspace{-0.2em}
\begin{equation}
\label{eq:training objective}
    \small
    \max _{\Theta} \Delta R(\boldsymbol{Z}(\Theta), \mathbf{\Pi}, \epsilon) = \tilde{R}(\boldsymbol{Z}(\Theta), \epsilon,\gamma) 
- R(\boldsymbol{Z}(\Theta), \epsilon \mid \mathbf{\Pi}), 
    \vspace{-0.2em}
\end{equation}
where $\tilde{R}(\boldsymbol{Z}(\Theta), \epsilon,\gamma)$ and $R^C(\boldsymbol{Z}, \epsilon|\mathbf{\Pi})$ are calculated
by~\cref{eq:revised average coding rate} and~\cref{eq:coding rate within a class} respectively. The training process is summarized in~\cref{algorithm:Training procedure of ERASE}.

\vspace{-0.5em}
\myPara{Node classification.} In the inference stage, after obtaining learned representations $\boldsymbol{Z}$, we perform the node classification task via logistic regression (LogReg)~\cite{buitinck2013api,scikitlearn}. Other linear or nonlinear classifiers are optional. More discussions about the choice of classifiers are provided in~\cref{sec:classificationoption}.

\subsection{Time Complexity Analysis of Training}
We take a $L$-layer GAT as an example. The time complexity of the model inference is $\mathcal{O}(L\vert \mathcal{V} \vert d_0d + L\vert \mathcal{V} \vert d^{2} + 2L\vert \mathcal{E}\vert d)$,  where $d$ is the dimension of the representations, and $|\mathcal{V}|$ and $|\mathcal{E}|$ denote the number of nodes and edges respectively. Without considering operations such as addition with low time consumption, the time complexity of decoupled label propagation is $\mathcal{O}(T_2K{\vert \mathcal{V}\vert}^{2} )$, where $T_1, T_2$ and $K$ denote the depth of label propagation and the number of classes respectively. The time complexity of computing the coding rate reduction is $\mathcal{O}(2Kd {\vert \mathcal{V}\vert}^{2}+Kd^3 + d^2\vert \mathcal{V} \vert+d^3 )$. In~\cref{sec:scalability of Algorithm}, we scale our algorithm by subgraph sampling~\cite{NIPS2017_5dd9db5e} to tackle the challenge of large-scale graphs.
The experimental comparison is attached in~\cref{sec:appTrainDe}.

\begin{table*}[bt]
\centering
\setlength{\tabcolsep}{2.9pt}
\renewcommand\arraystretch{1.0}
\footnotesize
\vspace{-1.2em}
\begin{tabular}{c|ccccc|ccccc}
\toprule
\specialrule{0em}{0pt}{0pt}
\textbf{Cora} &Asym-0.1 &   Asym-0.2 &   Asym-0.3 &  Asym-0.4 &   Asym-0.5 &   Sym-0.1 &   Sym-0.2 &  Sym-0.3 &   Sym-0.4 & Sym-0.5 
  \\ \hline
CE &   75.72~(0.57)	&  70.50~(1.26) &  62.52~(1.83)	& 48.26~(1.98) &34.23~(2.47)	&78.65~(0.83) &76.72~(1.08)	&73.72~(1.38)	& 67.30~(1.50)	&  60.28~(1.77) 
    \\
APL &75.41~(1.09)	&71.23~(1.23)	& 64.05~(1.64)	& 52.02~(2.90)	& 34.79~(3.10)	&  78.56~(0.49)	&  76.53~(0.68)	& 73.01~(0.93)	& 67.68~(1.45)	& 60.31~(2.00)\\ 
GCE &
    75.19~(0.81)	& 71.49~(1.28)	&  65.10~(1.96)	& 52.30~(2.46)	&33.84~(4.68)	&79.14~(0.78)	&76.67~(0.88)	& \underline{73.81~(1.08)}	&67.89~(1.47)	&60.54~(1.86)
   \\
Co-Dis &
    77.52~(1.04)	& 
    74.12~(2.48)	&
    64.48~(4.17)	&
    49.76~(3.69)	&
    32.76~(2.32)	&
    \underline{80.27~(0.94)}	&
    73.66~(2.26)	&
    72.04~(3.94)	 &
    68.36~(4.44)	&
    55.50~(5.47)
   \\
MCR${}^{2}$  &74.73~(0.70) &66.37~(0.88)	&59.80~(1.20)	&52.30~(1.26)	&36.69~(1.19)	&77.25~(0.84)	&74.32~(1.04)	&71.15~(1.13)	&68.06~(0.95)	&58.53~(1.70)
   \\
RT-GNN &    77.89~(2.74)   & 71.64~(1.29)                & 65.48 (2.69)               &52.71 (1.44)
    & 40.34 (3.37)   & 78.53 (2.57)       &74.48 (1.53)
    &72.50 (0.62) &67.91 (1.07) & 60.84 (0.86)     \\
PI-GNN &\underline{78.41~(1.12)}	&\underline{74.73~(3.07)}	&\underline{67.04~(4.81)}	&\underline{56.92~(4.52)}	&\underline{44.01~(5.64)}  &79.12~(1.42)	& \underline{78.11~(1.70)}	& 73.33~(4.96)	& \underline{68.81~(6.20)}	& \underline{61.71~(5.92)}
   \\  \hline
\rowcolor[HTML]{E7FBFF} ERASE &
    \textbf{81.43~(0.90)}	&
    \textbf{80.46~(1.00)}	&
    \textbf{79.52~(1.13)}	&
    \textbf{75.36~(2.32)}	&
    \textbf{48.00~(2.52)}	&
    \textbf{81.58~(0.80)}	&
    \textbf{81.97~(0.79)}	&
    \textbf{81.61~(0.95)}	&
    \textbf{80.13~(1.07)}	&
    \textbf{78.01~(1.05)} 
    \\ \hline
\specialrule{0em}{0.5pt}{0.5pt}
\hline
\textbf{CiteSeer} &
    Asym-0.1 &
    Asym-0.2 &
    Asym-0.3 &
    Asym-0.4 &
    Asym-0.5 &
    Sym-0.1 &
    Sym-0.2 &
    Sym-0.3 &
    Sym-0.4 &
    Sym-0.5 
    \\ \hline
CE &
    67.16~(2.98) &
    62.27~(1.32)	&
    58.71~(1.39)	&
    41.48~(1.22)	&
    34.30~(1.84)	&
    68.51~(0.42)	&
    64.65~(0.75)	&
    63.15~(1.00)	& 
    58.23~(1.12)	&
    42.37~(2.32)
   \\
APL &
    68.04~(1.13)	&
    62.19~(1.23)	&
    59.03~(2.26)	&
    41.01~(1.38)	&
    33.92~(1.64)	&
    68.51~(0.55)	&
    63.86~(1.11)	&
    62.73~(0.94) &
    57.85~(1.16)	&
    42.46~(1.78)
   \\
GCE &
    68.07~(3.36)	&
    62.60~(2.60)	&
    58.59~(3.15)	&
    41.66~(2.97)	&
    34.05~(2.33) &
    67.93~(3.01)	&
    63.76~(2.60)	&
    62.77~(2.77) &
    57.69~(2.76)	&
    42.71~(2.20)
   \\
Co-Dis &
    68.92~(1.19)	&
    63.56~(4.47)	&
    59.31~(3.66)	&
    40.40~(3.33)	&
    31.52~(3.77) &
    67.36~(1.27)	&
    67.75~(1.97)	&
    64.00~(3.66) &
    57.16~(3.66)	&
    43.72~(4.21)
   \\
MCR${}^{2}$ &
    68.07~(0.96)	&
    64.88~(1.54)	&
    {62.41}~(1.43)&
    42.57~(1.30) &
    33.33~(1.31)	&
    67.36~(0.74)	&
    63.76~(0.71) &
    61.89~(0.91) &
    57.52~(0.91)	&
    41.38~(1.31)
   \\
RT-GNN &\textbf {71.52~(0.66)}    &\underline{67.23~(1.35)}                 
    &\underline{64.84~(2.02)}  &\underline{49.59~(1.49)}
    &36.84~(1.51)           & 70.26 (1.93)       &
     67.50 (1.43) &66.04 (1.56) &\underline{63.57 (1.48)} & 49.74 (3.28)     \\
PI-GNN &
    {68.52}~(3.40)	&
    {66.33}~(3.47)	&
    56.71~(5.87)	&
    48.72~(6.28)	&
    \underline{37.13~(5.32)}	&
    \underline{70.62~(0.81)}	&
    \underline{68.23~(2.55)}	&
    \underline{67.23~(2.97)}	&
    61.93~(3.42)	&
    \underline{53.72~(5.66)}
   \\  \hline
\rowcolor[HTML]{E7FBFF} ERASE &
    \underline{70.70~(1.60)} &
    \textbf{69.91~(1.79)} &
    \textbf{69.45~(1.84)} &
    \textbf{59.32~(4.69)} &
    \textbf{49.62~(2.20)} &
    \textbf{70.81~(1.34)} &
    \textbf{69.85~(2.82)} &
    \textbf{69.08~(2.98)} &
    \textbf{68.37~(3.03)} &
    \textbf{58.56~(3.30)} 
   \\ \hline
\specialrule{0em}{0.5pt}{0.5pt}
\hline
\textbf{PubMed} &
    Asym-0.1 &
    Asym-0.2 &
    Asym-0.3 &
    Asym-0.4 &
    Asym-0.5 &
    Sym-0.1 &
    Sym-0.2 &
    Sym-0.3 &
    Sym-0.4 &
    Sym-0.5 
    \\ \hline
CE &
    77.17~(0.76) &
    {75.79}~(0.82)	&
    71.72~(1.75) &
    65.60~(2.84) &
    52.41~(4.51)	&
    77.32~(0.60) &
    77.08~(0.89) &
    71.07~(1.38)	&
    67.84~(1.20)	&
    63.19~(7.47)
   \\
APL &
    77.12~(0.84)	&
    \textbf{76.73~(0.84)}	&
    73.07~(1.94)	&
    \underline{66.18~(3.12)}	&
    53.88~(4.10)	&
    \underline{77.79~(1.16)}	&
    \underline{76.81~(0.84)}	&
    \underline{71.57~(1.73)}	&
    \underline{68.41~(1.61)}	&
    \underline{66.47~(2.04)}
   \\
GCE &
    {77.70}~(1.11) &
    76.05~(1.84)	&
    72.97~(2.11)	&
    65.30~(4.94)	&
    53.30~(5.62)	&
    \underline{77.79~(1.43)}	&
    77.42~(1.25)	&
    67.25~(8.38) &
    67.20~(2.39)	&
    63.27~(6.79) 
   \\
Co-Dis &
    \textbf{79.00~(1.66)} &
    76.24~(5.56)	&
    72.00~(6.46)	&
    62.55~(7.19)	&
    55.80~(9.53)	&
    77.64~(2.34)	&
    {75.28~(1.81)}	&
    72.92~(2.95) &
    66.76~(6.38)	&
    65.76~(7.98) 
   \\
MCR${}^{2}$  &
    76.90~(0.75)	&
    \underline{76.70~(0.89)}	&
    69.02~(1.21)	&
    59.50~(1.12)	&
    51.63(1.33)	&
    75.03~(1.32)	&
    73.63~(1.29)	&
    70.10~(1.17)	&
    65.72~(2.48)	&
    66.12~(1.04) 
   \\ 
RT-GNN & 77.62 (2.41)      & 76.25 (1.38)                &  \underline{73.19 (1.77)}       
    &65.22 (1.61)
    &  \underline{57.66 (1.93)}             &  76.72 (1.12)      &74.79 (1.04)&69.67 (1.07) &65.20 (1.72) &  59.10 (1.58)  \\
PI-GNN &
    76.01~(2.02)	& 
    73.83~(2.21)	& 
    70.62~(3.23)	& 
    63.33~(5.36)	& 
    55.42~(9.60)  & 	
    75.63~(1.71)	& 
    72.52~(3.62)	& 
    {68.13}~(4.79)	& 
    {62.82}~(5.10)	& 
    54.33~(5.73)
   \\ \hline
\rowcolor[HTML]{E7FBFF} ERASE &
    \underline{78.01~(1.10)}	&
    \underline{76.70~(1.07)}	&
    \textbf{76.71~(1.55)}	&
    \textbf{73.00~(2.62)}	&
    \textbf{61.46~(2.90)}	&
    \textbf{78.16~(0.88)}	&
    \textbf{77.86~(1.07)}	&
    \textbf{75.60~(0.85)}	&
    \textbf{72.72~(0.79)}	&
    \textbf{70.72~(1.08)}
  \\  \hline
\specialrule{0em}{0.5pt}{0.5pt}
\hline
\textbf{CoraFull} &
    Asym-0.1 &
    Asym-0.2 &
    Asym-0.3 &
    Asym-0.4 &
    Asym-0.5 &
    Sym-0.1 &
    Sym-0.2 &
    Sym-0.3 &
    Sym-0.4 &
    Sym-0.5 \\ \hline
CE &
    55.42~(0.21)	&
    50.76~(0.26)	&
    44.32~(0.32)	&
    33.76~(0.50)	&
    21.48~(0.29)	&
    56.48~(0.31)	&
    55.04~(0.33)	&
    52.57~(0.41)	&
    49.19~(0.40)	&
    44.71~(0.36) 
    \\
APL &
    {56.33}~(0.16)	&
    51.75~(0.38)	&
    45.73~(0.23)	&
    34.64~(0.43)	&
    21.17~(0.31)	&
    57.36~(0.26)	&
    56.36~(0.25)	&
    45.73~(0.23)	&
    50.68~(0.35)	&
    46.52~(0.44)
    \\
GCE &
    54.60~(0.33)	&
    50.41~(0.60)	&
    47.17~(0.51)	&
    37.00~(0.60)	&
    23.38~(0.91)	&
    55.37~(0.39)	&
    54.69~(0.33)	&
    51.82~(0.45)	&
    49.81~(0.51)	&
    45.48~(0.65)
    \\
Co-Dis &
    \textbf{58.96~(1.69)}	&
    52.78~(1.54)	&
    45.56~(1.49)	&
    34.82~(1.76)	&
    21.11~(2.87)	&
    \textbf{59.22~(1.05)}	&
    \textbf{56.71~(1.33)}	&
    \underline{54.20~(1.93)}	&
    48.90~(2.46)	&
    44.69~(1.14)
    \\
MCR${}^{2}$  &
    51.11~(0.37)	&
    45.75~(0.55)	&
    40.97~(0.58)	&
    35.22~(0.72)	&
    30.72~(1.26) &
    52.31~(0.33)	&
    50.83~(0.47)	&
    48.91~(0.58)	&
    46.86~(0.52) &
    45.56~(0.78)
    \\
    RT-GNN & 56.54 (0.30)      &52.15 (0.79)                 &      44.73 (0.46)             &38.65 (0.84)

    &  28.86 (0.56)            & \underline{57.87 (0.42)}       &55.57 (0.30)
    &53.07 (0.34) &\underline{50.84 (0.70)} & \underline{46.66 (0.61)}    \\
PI-GNN &
    \underline{56.62~(0.84)}	&
    \underline{53.81~(1.37)}	&
    \underline{48.83~(1.73)}	&
    \underline{40.02~(2.86)}	&
    \underline{31.33~(2.43)}	&
    56.71~(0.77)	&
    55.02~(0.99)	&
    52.93~(1.17)	&
    49.22~(1.67)	&
    45.53~(1.71)
    \\  \hline
\rowcolor[HTML]{E7FBFF} ERASE &
    55.60~(0.45)	&
    \textbf{53.96~(0.50)}	&
    \textbf{50.18~(0.49)}	&
    \textbf{41.79~(1.47)}	&
    \textbf{32.47~(1.11)}	&
    55.83~(0.54)	&
    \underline{55.33~(0.62)}	&
    \textbf{54.64~(0.63)}	&
    \textbf{52.07~(0.81)}	&
    \textbf{48.87~(0.64)}
    \\ \specialrule{0em}{0pt}{0pt}
  \bottomrule
\end{tabular}
\vspace{-1.2em}
\caption{Comparison with baselines in test accuracy ($\pm $ std) $\%$ with symmetric noise and asymmetric noise. Asym-0.1 means asymmetric noise type with a
0.1 noise rate and Sym-0.1 means symmetric noise type
with a 0.1 noise rate. }
\vspace{-2em}
\label{tab:baseline_main_results}
\end{table*}

\vspace{-0.8em}
\section{Experiments}
\vspace{-0.4em}

In this section, we present experiments on multiple graph node clarification benchmarks to show the robustness of our ERASE against label noise. Besides, we perform more empirical evaluations and ablations to verify key designs and insights of our method. 
\vspace{-0.4em}

\subsection{Experiment Setups}
\label{exp setting} 
\vspace{-0.5em}
\myPara{Datasets.} We conduct experiments on four graph datasets: Cora, CiteSeer, PubMed, and CoraFull. For Cora, CiteSeer, and PubMed, we apply the default split in~\cite{Yang_Cohen_Salakhutdinov_2016}. For CoraFull, we randomly select 20 nodes from each class for training, 30 nodes for validation, and the rest nodes for testing. Moreover, to explore the scalability of our method, we perform experiments on a large-scale graph dataset, OGBn-arxiv~\cite{NEURIPS2020_fb60d411}. Dataset statistics are provided in~\cref{appendix:dataset}. 

\myPara{Label noise generation.} We manually corrupt labels following two standard ways~\cite{patrini2017making, van2015learning, du2023noise}. 1) \textit{Symmetric flipping}: labels in each class are flipped to other classes with equal probability. 2) \textit{Asymmetric pair flipping}: labels in each class are flipped to only one certain class. Details of the label corruption ways can be found in~\cref{appendix:noise setting}. We set the noise rate $\phi \in \{0.1, 0.2, 0.3, 0.4, 0.5\}$ for both symmetric and asymmetric noise.

\myPara{Implementation details.} We use a 2-layer GAT as the backbone whose hidden dimension, attention heads in the first and last layers are set to 256, 8, and 1 respectively. Adam optimizer~\cite{kingma2014adam} with a learning rate of 0.001 and weight decay of $5\times 10^{-4}$ are utilized. The model is trained with 400 epochs on a single NVIDIA RTX-3090 GPU. The main results are reported with mean and standard values in 10 runs. More details are in~\cref{sec:appTrainDe}.

\myPara{Baselines.} We compare ERASE with some powerful baselines: Cross-entropy~(CE) Loss with GAT~\cite{Liu_Zhou_2020} backbone, Robust loss GCE~\cite{zhang2018generalized}, Robust loss APL~\cite{pmlr-v119-ma20c}, Vanilla MCR${}^{2}$~\cite{NEURIPS2020_6ad4174e}, Co-Dis~\cite{codis}, RT-GNN~\cite{qian2023robust}, and PI-GNN~\cite{du2023noise}, which include traditional optimization loss, label noise robust loss, latest label noise learning methods, and latest label noise learning methods on graph. We set the hyper-parameters to be consistent for a fair comparison and more details about baselines are provided in~\cref{sec:Details of baselines}.

\subsection{Comparisons with Baselines}
\vspace{-0.5em}
\label{sec:comparisons with baselines}
The main results comparing with baselines are shown in~\cref{tab:baseline_main_results}, showing that ERASE outperforms baseline methods on existing graph datasets. Particularly, ERASE shows significant gains in prediction accuracy on larger label noise ratio scenarios. As shown in~\cref{tab:baseline_main_results}, when $\phi=0.5$, ERASE outperforms baselines \textcolor{red}{$+3.99, +12.49, +3.04, +1.14$} in the asymmetric noise scenario on four datasets respectively. Besides, ERASE outperforms baselines \textcolor{red}{$+10.31, +4.84, +4.25, +2.21$} in the symmetric noise scenario respectively on four datasets when $\phi=0.5$. The comparison shows that our method is more robust against label noise than baselines. Figures of performance comparison are shown in~\cref{appendix:performance comparison}.

\vspace{-0.5em}
\subsection{Ablation Study}
\vspace{-0.7em}
\label{sec:Abalation_study}

\myPara{Label noise tolerance controller $\epsilon^{2}$ and empirical supports.} 
As shown in the definition of the coding rate and~\cref{rm1}, we explore how $\epsilon^{2}$ helps learn error-resilient representations. 
(1) As shown in~\cref{tab:eps}, the $\epsilon^{2}$ in~\cref{rm1} shows the error tolerance on the label noise. In~\cref{tab:eps}, if there is a larger number of mislabeled nodes, it needs a larger $\epsilon^{2}$ to tolerate label noise. When the noise rate is low, if the noise tolerance $\epsilon^{2}$ is too high, over-tolerance of the model is harmful to estimate the coding rate precisely and makes the model hard to learn, which indicates the trade-off behind the tolerance $\epsilon^{2}$. 
(2) As discussed in the last paragraph of~\cref{sec:preliminaries}, we compare the volume~\cite{Ma_Derksen_Hong_Wright_2007} of the space of noise $\delta$ and tolerance $\mathcal{N}(\boldsymbol{0}, \frac{\epsilon^2}{d}\boldsymbol{I})$ via $NTVR$ (Noise/Tolerance Volume Ratio) defined in~\cref{sec:definition of NVTR}. As shown in~\cref{tab:eps_vol}, the value of $NTVR$ is always less than 1, which means the volume of error $\delta$ caused by label noise is within the tolerance. This provides the experimental evidence of the assumption in~\cref{rm1}. With respect to the trade-off discussed in (1), the $NTVR\approx 0.3$ is an ideal judgment for choosing $\epsilon^{2}$ in larger label noise ratio scenarios and $NTVR\approx 0.1$ in lower label noise ratio scenarios. Here, $\epsilon^{2}=0.05$ is the design choice for two datasets. 
\vspace{-0.4em}

\begin{table*}[!t]
\begin{adjustbox}{width=1.0\textwidth,center}
    \begin{minipage}{0.6\linewidth}
\begin{table}[H]
\tiny
\setlength\tabcolsep{1.5pt}
\begin{tabular}{c|ccccc|ccccc}
\toprule
\textbf{Cora}   & Asym-0.1    & Asym-0.2    & Asym-0.3    & Asym-0.4    & Asym-0.5     & Sym-0.1     & Sym-0.2     & Sym-0.3     & Sym-0.4     & Sym-0.5     \\ \hline
0.01            & \textbf{81.96~(0.42)} & \textbf{81.28~(1.19)} & 77.70~(1.39) & 67.58~(3.36) & 32.98~(1.48) & \textbf{81.22~(0.80)} & 80.62~(0.87) & 79.62~(0.80) & \textbf{80.20~(1.41)} & 72.40~(3.09) \\
 \rowcolor[HTML]{E7FBFF} 0.05          & 81.42~(0.89) & 80.12~(0.84) &\textbf{79.58~(0.90)} & \textbf{75.00~(1.63)} & \textbf{48.34~(2.83)} & \textbf{81.22~(0.48)} & \textbf{80.84~(0.56)} & \textbf{80.90~(0.62)}  & 79.70~(0.97) & \textbf{77.74~(0.69)} \\
0.1             & 80.28~(0.91) & 79.24~(0.92) & 78.78~(0.89) & 75.40~(2.00) & 46.54~(3.41) & 78.86~(1.44) & 80.34~(0.47) & 80.32~(0.56) & 79.28~(1.29) & 75.90~(1.62) \\
0.5             & 74.44~(0.75) & 73.42~(2.03) & 68.78~(3.53) & 56.08~(2.64) & 41.98~(5.50) & 76.16~(0.66) & 75.84~(1.15) & 74.40~(2.20) & 70.10~(2.35) & 65.18~(3.73) \\ \hline
\specialrule{0em}{0.5pt}{0.5pt} \hline
\textbf{PubMed} & Asym-0.1    & Asym-0.2    & Asym-0.3    & Asym-0.4    & Sym-0.5     & Sym-0.1     & Sym-0.2     & Sym-0.3     & Sym-0.4     & Sym-0.5     \\ \hline
0.01            & 77.42~(1.15) & \textbf{76.68~(0.93)} & 73.96~(1.56) & 68.94~(2.49) & 56.76~(1.87) & 77.70~(1.26) & 77.24~(0.87) & 73.12~(0.88) & 69.86~(1.16) & 65.86~(1.39) \\
 \rowcolor[HTML]{E7FBFF}  0.05            & \textbf{78.48~(0.84)} & 76.00~(1.17) & \textbf{77.40~(0.77)} & \textbf{72.08~(2.35)} & 62.08~(1.03) & \textbf{78.44~(0.83)} & 75.96~(0.49) & \textbf{76.10~(0.51)} & \textbf{72.76~(0.42)} & 70.78~(1.24) \\
0.1 & 77.00~(2.02) & 76.08~(1.77) & 75.08~(2.24) & 71.18~(3.57) & \textbf{62.36~(2.30)} & 77.44~(1.64) & \textbf{77.84~(1.26)} & 75.72~(1.09) & 72.56~(1.48) & \textbf{71.90~(0.41)} \\
0.5             & 72.82~(1.44) & 72.74~(1.64)
&  70.38~(1.34) & 65.76 ~(2.51)            & 57.68~(2.72)            &  73.14~(0.75)           &  71.92~(0.40)           &  71.60~(1.47)           &  66.60~(1.25)           & 67.54~(0.68)            \\ \bottomrule
\end{tabular}
\vspace{-1.8em}
\caption{Mean test accuracy when changing $\epsilon^{2}$.}
\vspace{-2.3em}
\label{tab:eps}
\end{table}
\end{minipage}
\hspace{5pt}
\begin{minipage}{0.36\linewidth}  
\vspace{-1.4em}
\begin{table}[H]
\tiny
\renewcommand\arraystretch{0.9}
\setlength\tabcolsep{1.5pt}
\begin{tabular}{c|ccccc|ccccc}
\toprule
\multirow{2}{*}{\textbf{Cora}}   & \multicolumn{5}{c|}{Asymmetric}       & \multicolumn{5}{c}{Symmetric}         \\ \cline{2-11} 
                                 & 0.1   & 0.2   & 0.3   & 0.4   & 0.5   & 0.1   & 0.2   & 0.3   & 0.4   & 0.5   \\ \hline
0.01                             & 0.183 & 0.302 & 0.401 & 0.539 & 0.548 & 0.153 & 0.198 & 0.233 & 0.292 & 0.467 \\
 \rowcolor[HTML]{E7FBFF}  0.05                             & 0.112 & 0.168 & 0.387 & 0.122 & 0.342 & 0.081 & 0.100 & 0.132 & 0.157 & 0.307 \\
0.1                              & 0.062 & 0.130 & 0.194 & 0.291 & 0.301 & 0.076 & 0.100 & 0.100 & 0.105 & 0.200 \\
0.5                              & 0.019 & 0.027 & 0.078 & 0.108 & 0.111 & 0.036 & 0.031 & 0.048 & 0.033 & 0.089 \\ \hline
\specialrule{0em}{0.5pt}{0.5pt} \hline
\multirow{2}{*}{\textbf{PubMed}} & \multicolumn{5}{c|}{Asymmetric}       & \multicolumn{5}{c}{Symmetric}         \\ \cline{2-11} 
                                 & 0.1   & 0.2   & 0.3   & 0.4   & 0.5   & 0.1   & 0.2   & 0.3   & 0.4   & 0.5   \\ \hline
0.01                             & 0.334 & 0.339 & 0.398 & 0.487 & 0.479 & 0.365 & 0.354 & 0.366 & 0.435 & 0.431 \\
 \rowcolor[HTML]{E7FBFF}  0.05                             & 0.111 & 0.111 & 0.145 & 0.219 & 0.235 & 0.115 & 0.113 & 0.133 & 0.186 & 0.176 \\
0.1                              & 0.050 & 0.055 & 0.087 & 0.148 & 0.181 & 0.055 & 0.053 & 0.062 & 0.096 & 0.101 \\
0.5                              & 0.009 & 0.010 & 0.017 & 0.039 & 0.063 & 0.011 & 0.010 & 0.013 & 0.020 & 0.021\\ \bottomrule
\end{tabular}
\vspace{-1.8em}
\caption{$NTVR$ value when changing $\epsilon^{2}$.}
\vspace{-6em}
\label{tab:eps_vol}
\end{table}
\end{minipage}
\end{adjustbox}
\end{table*}

\vspace{-0.5em}
\myPara{Ablation study on decoupled label propagation.} In~\cref{sec:estimation of senmantic-structural label}, we split label propagation on graphs into label denoising propagation before training (DP, \cref{sec:dp}) and label semantic propagation during training (SP, \cref{sec:sp}). As shown in~\cref{tab:impact of deoupled label propagation}, both label propagation phases contribute significantly to robustness, especially in high label noise ratio scenarios.

\begin{table*}[!t]
\begin{adjustbox}{width=1.0\textwidth,center}
\begin{minipage}{0.39\linewidth}

\begin{table}[H]
\centering
\scriptsize
\setlength\tabcolsep{4.4pt}
\vspace{-1.2em}
\begin{tabular}{c|ccc|ccc}
\toprule \specialrule{0em}{0.0pt}{0.0pt}
\textbf{Cora}     & \multicolumn{3}{c|}{Asymmetric}                  & \multicolumn{3}{c}{Symmetric}                    \\ \hline
\textbf{DP\quad SP}    & 0.3            & 0.4            & 0.5            & 0.3            & 0.4            & 0.5            \\ \hline
\XSolidBrush\quad\XSolidBrush                & 78.50          & 60.58          & 29.16          & 80.52          & 77.84          & 73.92          \\
\XSolidBrush\quad\Checkmark               & \underline{79.82}    & 70.92          & 39.80          & \textbf{81.00} & 79.12          & 76.58          \\
\Checkmark\quad\XSolidBrush               & \textbf{79.84} & \underline{73.36}    & \underline{42.84}    & 80.84          & \textbf{79.96} & \underline{77.16}    \\ \hline
\rowcolor[HTML]{E7FBFF}
\Checkmark\quad\Checkmark               & 79.58          & \textbf{75.00} & \textbf{48.34} & \underline{80.90}    & \underline{79.70}    & \textbf{77.74} \\ \hline
\specialrule{0em}{0.5pt}{0.5pt}
\hline
\textbf{PubMed} & \multicolumn{3}{c|}{Asymmetric}                  & \multicolumn{3}{c}{Symmetric}                    \\ \hline
\textbf{DP\quad SP}    & 0.3            & 0.4            & 0.5            & 0.3            & 0.4            & 0.5            \\ \hline
\XSolidBrush\quad\XSolidBrush &74.84& 67.46 & 53.18      & 75.02 &71.20& 62.10  \\
\XSolidBrush\quad\Checkmark    &    \underline{75.90}&	\underline{67.94}	&53.50  &      75.52	&\underline{71.94}&	64.90 \\
\Checkmark\quad\XSolidBrush               & 75.72	&67.82&	\underline{54.16}    &\underline{75.66}	&71.50	&\underline{64.96}     \\ \hline
\rowcolor[HTML]{E7FBFF}
\Checkmark\quad\Checkmark               & \textbf{77.40} & \textbf{72.08} & \textbf{62.08} & \textbf{76.10 }    & \textbf{72.76 }          & \textbf{70.78 }    \\
\specialrule{0em}{0.0pt}{0.0pt} \bottomrule
\end{tabular}
\vspace{-1.2em}
\caption{\centering Ablation on w/ or w/o denoising propagation~(DP) and semantic propagation (SP).}
% \vspace{-2.2em}
\label{tab:impact of deoupled label propagation}
\end{table}
\end{minipage}

\begin{minipage}{0.38\linewidth}
\begin{table}[H]\centering
\scriptsize
\setlength\tabcolsep{2.8pt}
\renewcommand\arraystretch{0.95}
\vspace{-1.0em}
\begin{tabular}{c|ccccc}
\toprule  \specialrule{0em}{0.0pt}{0.0pt}
                                   \multicolumn{6}{c}{\textbf{Cora}}  \\ \hline
Asym.   & 0.1            & 0.2            & 0.3            & 0.4            & 0.5            \\ \hline
CE                                & 74.16          & 73.26          & 65.04          & 59.06          & 34.86          \\
\rowcolor[HTML]{E7FBFF}
ERASE                             & \textbf{81.42} & \textbf{80.12} & \textbf{79.58} & \textbf{75.00} & \textbf{48.34} \\ \hline \specialrule{0em}{0.5pt}{0.5pt} \hline
Sym.   & 0.1            & 0.2            & 0.3            & 0.4            & 0.5            \\ \hline
CE                                & 75.46          & 73.16          & 71.32          & 70.94          & 62.20          \\
\rowcolor[HTML]{E7FBFF}
ERASE                             & \textbf{81.22} & \textbf{80.84} & \textbf{80.90} & \textbf{79.70} & \textbf{77.74} \\ \hline
\specialrule{0em}{0.5pt}{0.5pt}
\hline
                                   \multicolumn{6}{c}{\textbf{PubMed}}  \\ \hline

Asym.  & 0.1            & 0.2            & 0.3            & 0.4            & 0.5            \\ \hline
CE                                & 77.06          & 74.92          & 74.26          & 69.04          & 55.64          \\
\rowcolor[HTML]{E7FBFF}
ERASE                             & \textbf{78.48} & \textbf{76.00} & \textbf{77.40} & \textbf{72.08} & \textbf{62.08} \\ \hline
\specialrule{0em}{0.5pt}{0.5pt}
Sym.  & 0.1            & 0.2            & 0.3            & 0.4            & 0.5            \\ \hline
CE                                & 77.44          & 74.86          & 74.42          & 65.86          & 63.80          \\
\rowcolor[HTML]{E7FBFF}
ERASE                             & \textbf{78.44} & \textbf{75.96} & \textbf{76.10} & \textbf{72.76} & \textbf{70.78}\\ \specialrule{0em}{0.0pt}{0.0pt} \bottomrule
\end{tabular}
\vspace{-1.2em}
\caption{\centering Ablation of how to estimate semantic labels (semantic labels learned by CE v.s. ERASE).}
% \vspace{-2.3em}
    \label{tab:impact of coding rate reduction on semantic information}
\end{table}
\end{minipage}

\begin{minipage}{0.4\linewidth}
\vspace{-0.7em}
\begin{table}[H]
    \centering
    \scriptsize
    \setlength\tabcolsep{3pt}
    \begin{tabular}{c|ccc|ccc}
    \toprule \specialrule{0em}{0.0pt}{0.0pt}
                                        & \multicolumn{3}{c|}{Asymmetric} & \multicolumn{3}{c}{Symmetric} \\ \cline{2-7} 
    \multirow{-2}{*}{\textbf{Cora}}     & 0.3       & 0.4      & 0.5      & 0.3      & 0.4      & 0.5     \\ \hline
    GCN        &   78.74   &  \textbf{75.73}    &   47.31 &   80.37   &   80.07   & \textbf{ 78.19}      \\ 
    GraphSAGE     & 79.14     &72.63  &  45.12   &  80.38    &78.53   & 75.75   \\ \hline
    \rowcolor[HTML]{E7FBFF}
    GAT  &  \textbf{79.52}& 75.36&\textbf{48.00}& \textbf{81.61}	&  \textbf{80.13}	& 78.01    \\ \hline 
    \specialrule{0em}{0.5pt}{0.5pt}
    \hline
    & \multicolumn{3}{c|}{Asymmetric} & \multicolumn{3}{c}{Symmetric} \\ \cline{2-7} 
    
    \multirow{-2}{*}{\textbf{CiteSeer}} & 0.3       & 0.4      & 0.5      & 0.3      & 0.4      & 0.5     \\ \hline
    GCN    & 68.79    &  58.97       &   48.19    &  \textbf{70.69}       &  \textbf{69.64}   & 58.88  \\
    GraphSAGE   &66.12     & 52.61 &  43.39 & 67.47   & 52.61    & \textbf{59.17}     \\ \hline
    \rowcolor[HTML]{E7FBFF}
    GAT                              &    \textbf{ 69.45}      &  \textbf{ 59.32}    &  \textbf{49.62}     &  69.08      &   68.37   &  58.56 \\  \hline   \specialrule{0em}{0.5pt}{0.5pt}
    \hline
        & \multicolumn{3}{c|}{Asymmetric} & \multicolumn{3}{c}{Symmetric} \\ \cline{2-7} 
    \multirow{-2}{*}{\textbf{PubMed}} & 0.3       & 0.4      & 0.5      & 0.3      & 0.4      & 0.5     \\ \hline
    GCN    &74.70    &  68.83     &   58.32    &  74.64       &  72.40  & 69.30  \\
    GraphSAGE       & 76.18  & 70.33 &  59.91 &  75.34 & 72.38&69.64     \\ \hline
    \rowcolor[HTML]{E7FBFF}
    GAT                              &    \textbf{76.71}     &   \textbf{73.00}    &  \textbf{61.46}     &  \textbf{75.60}     &  \textbf{ 72.72}  &  \textbf{70.72} \\ \specialrule{0em}{0.0pt}{0.0pt}
    \bottomrule  
    \end{tabular}
    \vspace{-1.2em}
    \caption{
      \centering Ablation on different backbones.}
    \label{tab:backbone}
    \end{table}
\end{minipage}
\end{adjustbox}
\vspace{-1.4em}
\end{table*}

\myPara{Ablation study on the error-resilient training objective design.}
In~\cref{{sec:sp}}, we claimed that representations learned by coding rate reduction can well estimate semantic labels in noisy label scenarios. To verify whether coding rate reduction plays a crucial role in estimating semantic labels, we compare our method with the cross-entropy~(CE) training objective. The comparisons are presented in~\cref{tab:impact of coding rate reduction on semantic information}. The results show that maximizing the coding rate reduction training objective provides a more robust estimation of semantic labels when meeting label noise. This verifies the basic motivation on why learning semantic labels via CE loss is not a wise choice.

\vspace{0.2em}
\myPara{Ablation study on classification methods.}
\label{sec:classificationoption}
Here, we would like to explore whether a linear classification method is enough for the node classification after training. We compare the Logistic Regression~(LogReg, our choice) with some linear~(Linear SVM) and non-linear classification~(Polynomial SVM, and MLP) methods. The implementation details are shown in ~\cref{sec:implementation details of classfication methods}. As shown in~\cref{tab:evalution methods}, both linear and non-linear methods perform well when performing classification, and LogReg usually performs better than non-linear methods. Besides, the linear algorithms are more efficient than non-linear methods. The key explanation on why a linear algorithm is sufficient for final step node classification is based on~\cref{lam1}, where the error-resilient training objective aims to make the representations between different classes orthogonal. This objective is also supported by the empirical results in~\cref{sec:visrep}.

\begin{table*}[!t]\vspace{-1em}
\begin{adjustbox}{width=1.0\textwidth,center}
\begin{minipage}{0.37\linewidth}
\begin{table}[H]
\centering
\scriptsize
\setlength\tabcolsep{4pt}
\renewcommand\arraystretch{0.9}
\begin{tabular}{c|ccc|ccc}
\toprule \specialrule{0em}{0.0pt}{0.0pt}
                                    & \multicolumn{3}{c|}{Asymmetric} & \multicolumn{3}{c}{Symmetric} \\ \cline{2-7} 
\multirow{-2}{*}{\textbf{Cora}}     & 0.3       & 0.4      & 0.5      & 0.3      & 0.4      & 0.5     \\ \hline
MLP                                 &   \textbf{79.66}   &   73.90    &   \textbf{50.12}   &   81.04    &   \underline{80.08}   & \underline{77.46}        \\ 
SVM (P)                                 &  79.64         &   73.26   &  47.52     &   81.00   &  80.04   &   \underline{77.46}  \\
SVM (L)                                 &    80.34       &    \underline{74.16}      &    47.58      &     \underline{81.50}     &    79.84      &   77.40      \\ \hline
\rowcolor[HTML]{E7FBFF} 
LogReg  &  \underline{79.52}	&  \textbf{75.36}	&  \underline{48.00}	&  \textbf{81.61}	&  \textbf{80.13}	& \textbf{78.01}    \\ \hline 
\specialrule{0em}{0.5pt}{0.5pt}
\hline
& \multicolumn{3}{c|}{Asymmetric} & \multicolumn{3}{c}{Symmetric} \\ \cline{2-7} 

\multirow{-2}{*}{\textbf{CiteSeer}} & 0.3       & 0.4      & 0.5      & 0.3      & 0.4      & 0.5     \\ \hline
MLP                                 &  \underline{68.50}         &   \underline{60.56}    &   \underline{45.18}    &   \underline{68.56}       &   68.07       &  \textbf{63.12}  \\
SVM (P)                                 &   66.68   &   59.20  &  42.78   &  67.20  &   66.42   &   \underline{61.82}   \\ 
SVM (L)                                 &    66.76       &   \textbf{61.34}   &   44.32   & 67.72   &   \underline{68.16}   &   61.80  \\ \hline
\rowcolor[HTML]{E7FBFF} 
LogReg                              &     \textbf{69.45}      &   59.32    &  \textbf{49.62}     &  \textbf{69.08}      &   \textbf{68.37}   &  58.56 \\  \specialrule{0em}{0.0pt}{0.0pt} \bottomrule  
\end{tabular}
\vspace{-1.4em}
\caption{
  \centering Ablation on linear and non-linear classification methods.}
\label{tab:evalution methods}
\end{table}
\end{minipage}
% \hspace{-4pt}
\begin{minipage}{0.6\linewidth}
% \hspace{-3em}

\begin{table}[H]\centering
\setlength{\tabcolsep}{2pt}
\renewcommand\arraystretch{1.1}
\scriptsize
\label{table:ogbarxiv}
\begin{tabular}{c|ccc|ccc}
\toprule   \specialrule{0em}{0.0pt}{0.0pt} 

{\textbf{OGBn-}} & \multicolumn{3}{c|}{Asymmetric} & \multicolumn{3}{c}{Symmetric} \\ \cline{2-7} 
{\textbf{arxiv}} &
  \multicolumn{1}{c}{0.3} &
  \multicolumn{1}{c}{0.4} &
  \multicolumn{1}{c|}{0.5} &
  \multicolumn{1}{c}{0.3} &
  \multicolumn{1}{c}{0.4} &
  \multicolumn{1}{c}{0.5} \\ \hline
CE          & 
50.21~(0.59)        &
44.32~(1.36)            &
29.60~(2.46)        &
54.24~(0.32)         &
53.70~(0.45)        &
54.52~(0.58)       
\\
APL         &
55.88~(2.92)    &
52.88~(1.53) &
30.94~(2.09)     &
57.63~(0.54)  &
52.88~(1.53)       &
52.90~(2.17)
\\
GCE   &      
52.17~(0.51)&
47.66~(0.49)&
32.84~(2.66)&
53.89~(0.41)&
53.27~(0.36)&
52.83~(0.89)
 \\
Co-Dis    &  
51.59~(0.78)&
50.44~(0.70)&
32.44~(4.91)& 
51.57~(0.94)&
50.72~(0.83)&
50.37~(0.71)
 \\
MCR${}^{2}$ &  57.62~(0.75)       & 48.39~(1.34)             & 28.34~(1.45)       & 61.26~(0.22)         &60.84~(0.70)          &  59.94~(0.32)       \\
RT-GNN      & -       & -    & -      &  -  & -  & -  \\  
PI-GNN      & 51.56 (0.17)       &  47.47 (0.56)     &34.20 (1.19)        &  53.70~(0.45)       &53.64 (0.29)          &53.36 (0.23)        \\   \hline
\rowcolor[HTML]{E7FBFF} 
ERASE     &  \textbf{67.21~(0.21)}      &\textbf{62.40 (0.16) }      & \textbf{34.74~(0.68)}       &  \textbf{ 68.23~(0.06)}       & \textbf{67.96 (0.11)}         &\textbf{ 67.63 (0.21)}        \\   \specialrule{0em}{0.0pt}{0.0pt}  \bottomrule
\end{tabular}
\vspace{-1.4em}
\caption{
  \centering Comparison with baselines on OGBn-arxiv. All experiments are run 5 times. The best results are in bold. (RT-GNN fails on OGBn-arxiv.)}
\end{table}
\end{minipage}
\end{adjustbox}
\vspace{-0.9em}
\end{table*}

% \subsubsection{Sensitivity analysis}
% \label{sec:sensitivity analysis}

\vspace{0.2em}
\myPara{Depth of label propagation.} 
As shown in~\cref{tab:impact of deoupled label propagation}, both phases of decoupled label propagation contribute to the final results. Here, we set $T_1=T_2=T$ to explore how the propagation depth affects the result. As shown in~\cref{fig:ablation_study_T}, the classification accuracy grows when $T$ is increasing at the beginning. This shows that aggregating labels from neighbors is quite essential for label noise scenario. When $T$ is too large, the model will perform poorly, which is mainly due to the over-smoothing problem of labels on the graph. We provide more detailed results in~\cref{sec:results of propagation steps}. 
\vspace{-0.5em}
\myPara{Different network backbones.}
To verify the generalization ability of the algorithm, we ablated the performance of the ERASE under different GNN backbones. To this end, we compare the GAT backbone with GCN~\cite{kipf2016semi} and GraphSAGE~\cite{NIPS2017_5dd9db5e}. Results are shown in~\cref{tab:backbone}. Results show that our ERASE method enjoys different GNN backbones on these datasets. For a fair comparison, we take the GAT backbone to compare with baselines in~\cref{tab:baseline_main_results}. For detailed results, please refer to~\cref{appendix:backbone}.
Results in ablations are reported with the mean of 5 runs. Due to the page limits, we leave more discussions about our technical design in~\cref{sec:more ablation}.
% \begin{figure}[htb]
%     \centering
%     \captionsetup{}
%     \includegraphics[width=\linewidth]
%     {img/ablation_study_T.pdf}
%     \caption{Test accuracy of ERASE with different $T$ on Cora and CiteSeer.}
%     \label{fig:ablation_study_T}
% \end{figure}

\vspace{-0.8em}
\begin{figure}[htbp]
    \vspace{-0.4em}
    \captionsetup{}
  \centering
  \begin{subfigure}{0.22\textwidth}
    \includegraphics[width=\textwidth]{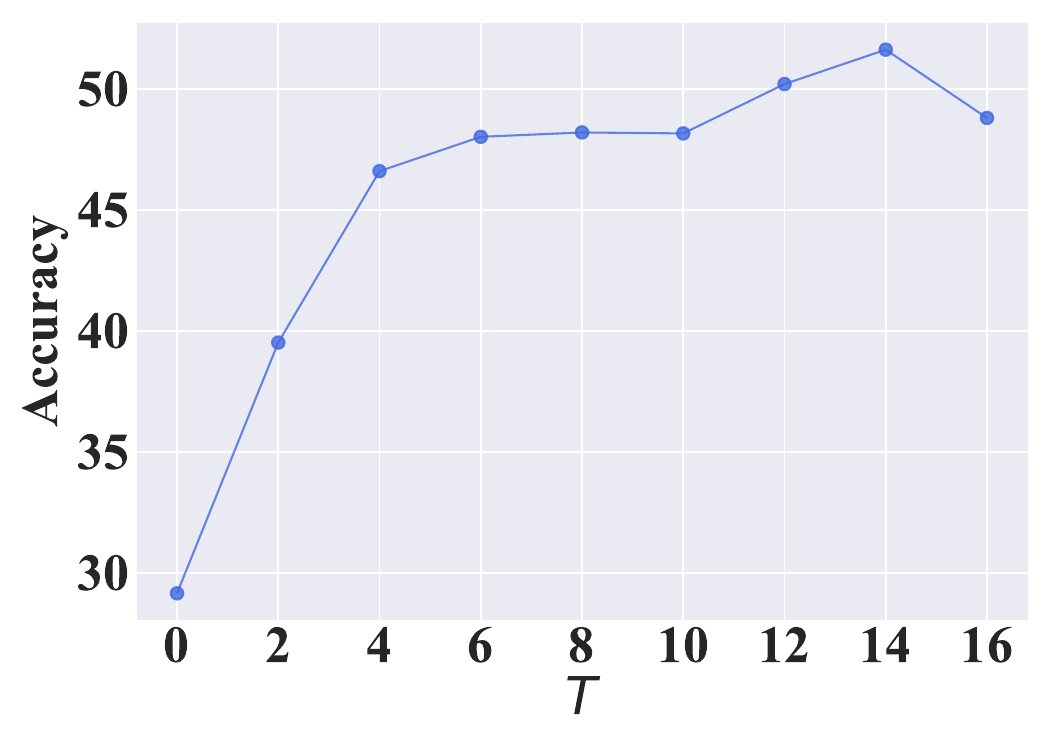}

\vspace{-0.4em}
    \caption{Asymmetric $\phi = 0.5$.}
    \label{fig:ablation_study_T_Cora_asymm}
  \end{subfigure}
  % \newline
  \begin{subfigure}{0.22\textwidth}
    \includegraphics[width=\textwidth]{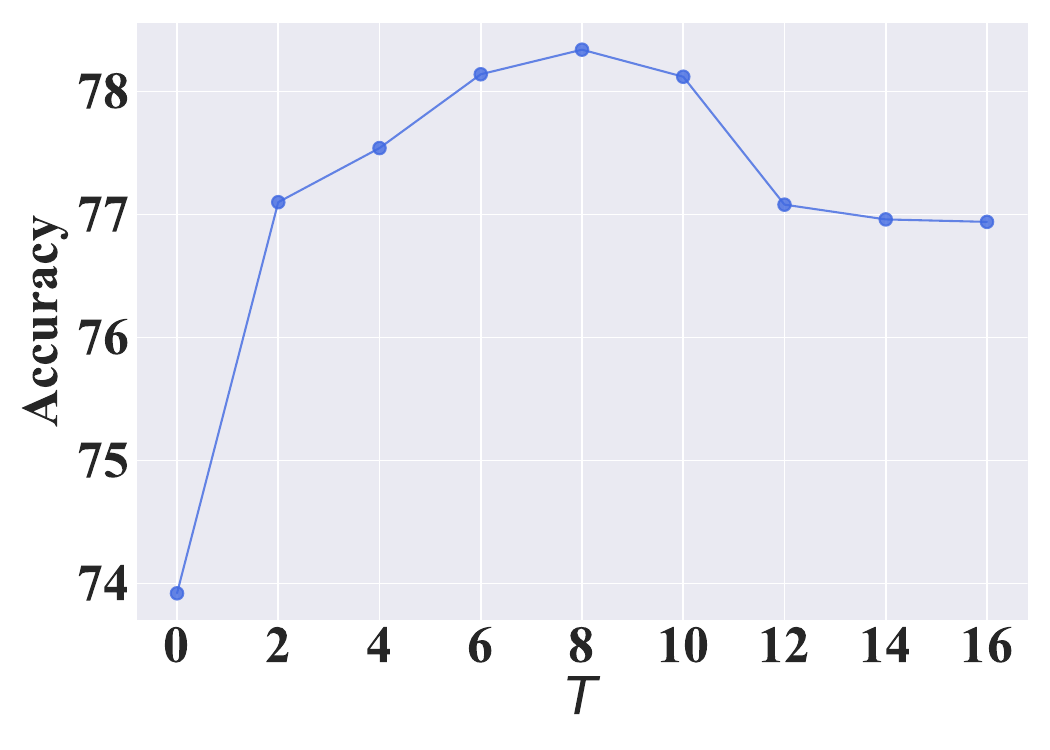}
    
\vspace{-0.4em}
    \caption{Symmetric $\phi = 0.5$.}
     \label{fig:ablation_study_T_Cora_symm}
  \end{subfigure}
  \newline
  \centering
  \begin{subfigure}{0.22\textwidth}
    \includegraphics[width=\textwidth]{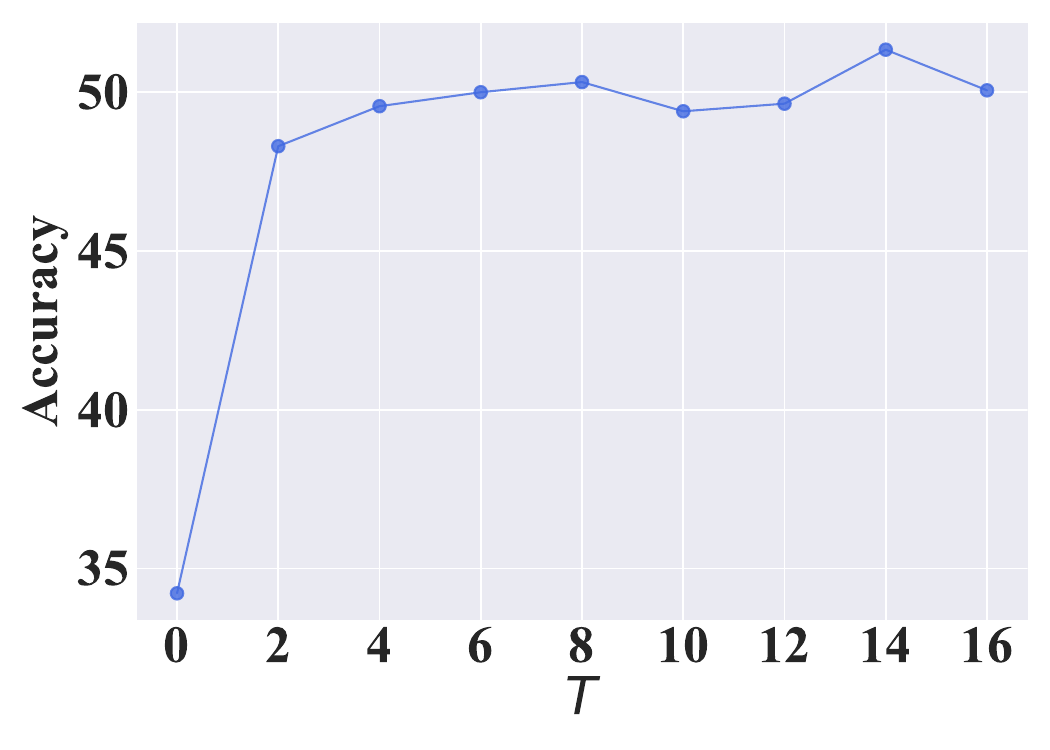}
\vspace{-1.6em}
    \caption{Asymmetric $\phi = 0.5$}
    \label{fig:ablation_study_T_CiteSeer_asymm}
  \end{subfigure}
  \begin{subfigure}{0.22\textwidth}
    \includegraphics[width=\textwidth]{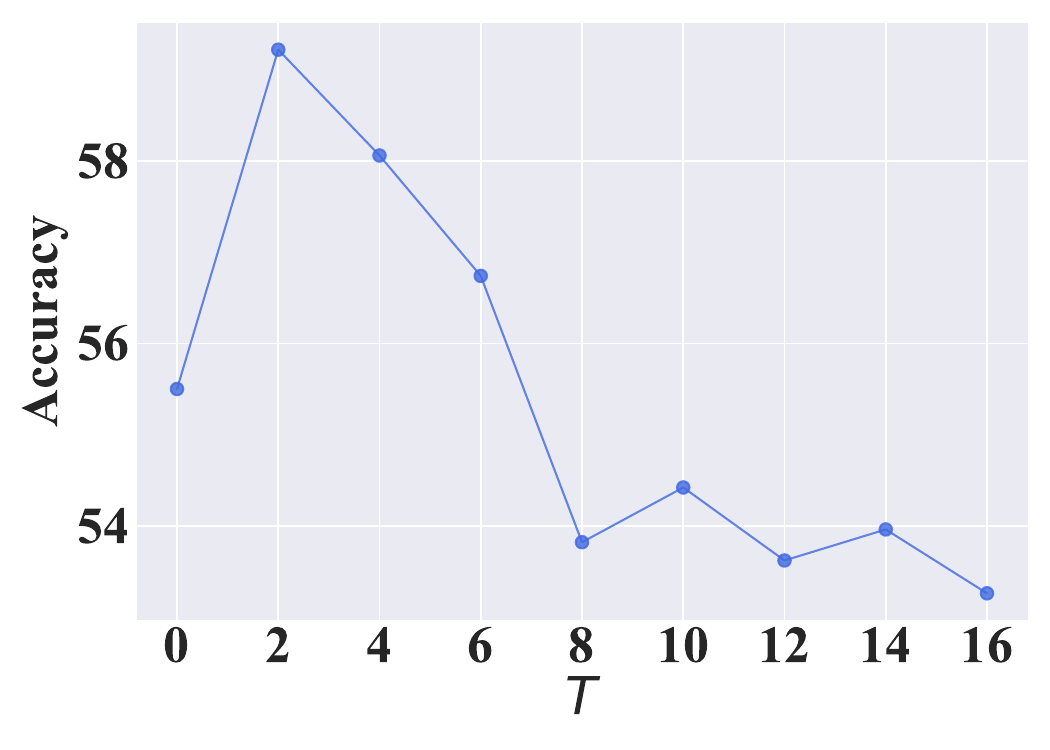}
\vspace{-1.6em}
    \caption{Symmetric $\phi = 0.5$.}
     \label{fig:ablation_study_T_CiteSeer_symm}
  \end{subfigure}
\vspace{-1em}
  \caption{Test accuracy of ERASE with 
  different $T$s on Cora (a, b) and CiteSeer (c, d) respectively.}
  \vspace{-1em}
  \label{fig:ablation_study_T}
\end{figure}

\begin{figure*}[!ht]
    \vspace{-0.3em}
    \small
    \captionsetup{}
    \centering
    \begin{subfigure}{0.49\textwidth}
    \includegraphics[width=\textwidth]{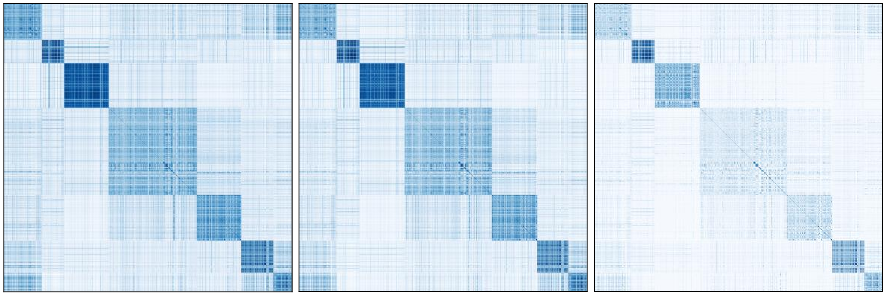}
    \caption{Results of Cora. Asymmetric $\phi = 0.1, 0.3, 0.5$ respectively.}
    \end{subfigure}
    \hfill
    \begin{subfigure}{0.49\textwidth}
    \includegraphics[width=\textwidth]{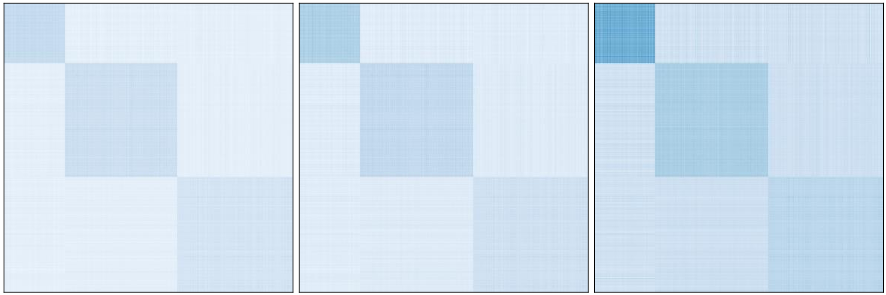}
    \caption{Results of PubMed. Asymmetric $\phi = 0.1, 0.3, 0.5$ respectively.}
    \end{subfigure}
    \vspace{-1.2em}
    \caption{Cosine similarity between sorted representation pairs.}
    \label{fig:visualize_cos_mat}
    \vspace{-1.3em}
\end{figure*}

\begin{figure*}[!ht]
    \captionsetup{}
  \centering
  \begin{subfigure}{0.49\textwidth}
  \centering
    \includegraphics[width=\textwidth]{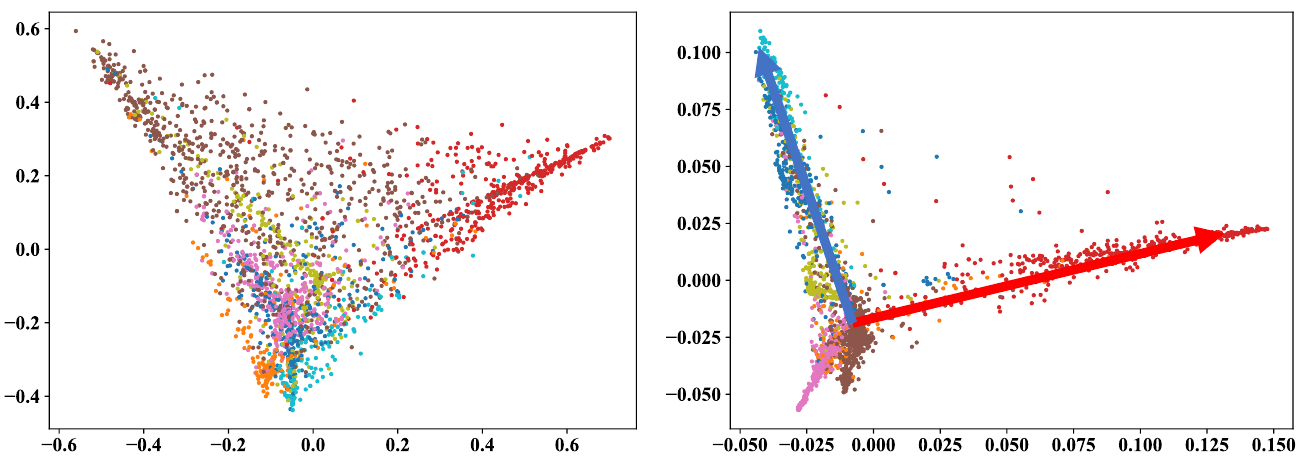}
    \caption{ 
  \centering Results of Cora (Asymmetric $\phi = 0.3$) obtained by CE~(left) v.s. ERASE~(right) respectively.}
  \end{subfigure}
  % \newline
  \centering
  \begin{subfigure}{0.49\textwidth}
  \centering
    \includegraphics[width=\textwidth]{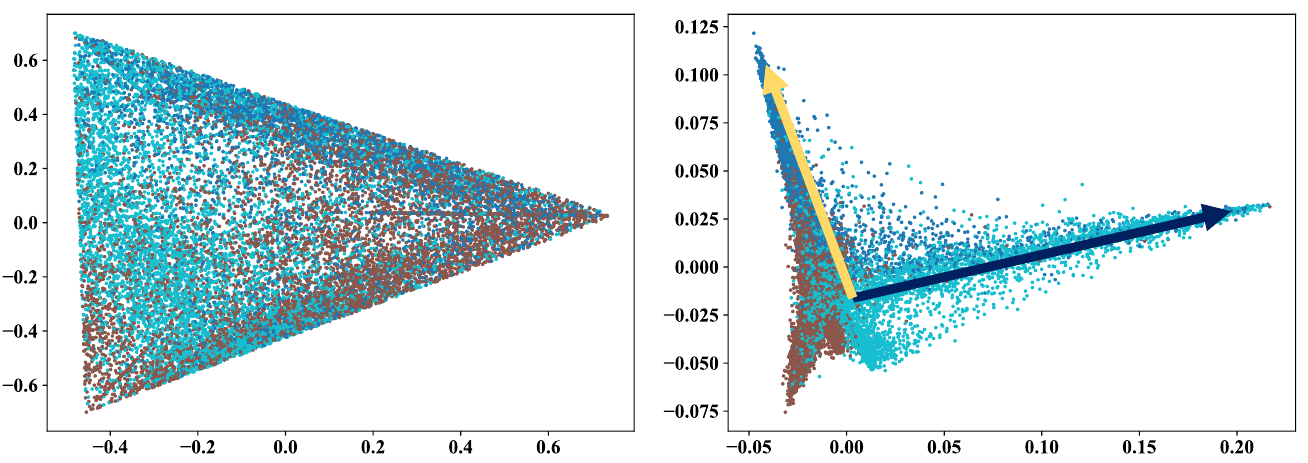}
   \caption{
  \centering Results of PubMed (Asymmetric $\phi = 0.5$) obtained by CE~(left) v.s. ERASE~(right) respectively.}
  \end{subfigure}
    \vspace{-0.9em} 
  \caption{PCA visualization of learned representations (CE v.s. ERASE).}
    \vspace{-1.5em}
  \label{fig:visualize_pca}
\end{figure*}

\subsection{Scalability of Algorithm}
\vspace{-0.4em} 
\label{sec:scalability of Algorithm}

To verify the scalability of our proposed method, we provide the experimental comparison with baselines on a larger scale graph benchmark~\cite{NEURIPS2020_fb60d411}, OGBn-arxiv, whose statistics are provided in~\cref{appendix:dataset}. The experimental results are shown in~\cref{table:ogbarxiv}. The comparison with baselines shows the good scalability of our method on large-scale datasets.

\vspace{-0.4em} 
\subsection{Visualization of Learned Representations}
\vspace{-0.5em} 
\label{sec:visrep}

As claimed in~\cref{sec:dp}, the optimal solution for maximizing coding rate reduction is that the representations between different classes are orthogonal. As shown in~\cref{tab:evalution methods}, a linear classification~(LogReg) method is sufficient for the final classification. (a) We provide visualization of the confusion matrix of learned representations. As shown in~\cref{fig:visualize_cos_mat}, learned representations of ERASE between different classes are approximately orthogonal even in high noise ratio scenarios. (b) To further verify the claim, we use the \textbf{linear} dimensionality reduction algorithm (PCA) to observe the orthogonality of different representations. The results are visualized in~\cref{fig:visualize_pca}. Compared with the CE optimization goal, representations learned by our method show better orthogonality between different classes, which is more discriminative for a simple classifier to perform classification. This phenomenon provides the experimental foundation of why a linear classification algorithm is good for the final prediction, which is shown in~\cref{tab:evalution methods}. More comparisons and t-SNE visualization~\cite{van2008visualizing} are provided in~\cref{appendix:MoreVisRes}.

\vspace{-0.5em}
\subsection{Why is ERASE more error-resilient than baselines?}
\vspace{-0.5em}
As claimed in~\cref{sec:intro}, our method is designed by an error-resilient principle. To verify this, we compare the correction rate of the mislabeled nodes with baselines. The results are shown in~\cref{tab:training_correction} (on Cora) and~\cref{tab:training_correction2} (On CiteSeer) respectively. The results show that ERASE significantly corrects the labels of mislabeled nodes, which means that it is more error-resilient than baselines. These results validate the motivation of our error-resilient design principle.

\begin{table}[H]
    \centering
    \footnotesize
    \setlength\tabcolsep{3pt}
\begin{tabular}{c|ccc|ccc}
\toprule \specialrule{0em}{0.0pt}{0.0pt}
                       & \multicolumn{3}{c|}{Asymmetric}                  & \multicolumn{3}{c}{Symmetric}                    \\ \cline{2-7} 
\multirow{-2}{*}{Cora} & 0.3            & 0.4            & 0.5            & 0.3            & 0.4            & 0.5            \\ \hline
CE             & 39.71  & 26.96	& 20.26	& 48.21	& 36.90 & 31.58\\
\rowcolor[HTML]{FFFFFF} 
APL            & 39.71	& 27.50	& 19.74	& 50.71	& 37.62	& 31.23\\
GCE            & 39.14	& 30.89	& 22.89	& 50.00	& 38.57	& 29.65\\
CoDis          & 44.29	& 31.43	& 19.21	& 37.32	& 40.54	& 37.26\\
MCR$^2$        & 16.57	& 16.07	& 14.08	& 39.64	& 21.19	& 14.56\\
RT-GNN          & \underline{57.91}   &\underline{53.32}   & \underline{40.35} & 43.75 & 40.55   & \underline{39.82}  \\
PI-GNN         &41.46   & 37.03 & 19.69 & \underline{57.69} & \underline{46.47} & 28.04\\
\rowcolor[HTML]{E7FBFF} 
ERASE          &\textbf{84.86}   &\textbf{75.71}    &\textbf{54.61} &\textbf{84.29}     &\textbf{74.29}     &\textbf{62.46}\\ 
\specialrule{0em}{0.0pt}{0.0pt} \bottomrule
\end{tabular}
    % \vspace{-1.2em}
\caption{Correction rate of mislabeled nodes~(on CiteSeer).}
\label{tab:training_correction2}
    \end{table}

\vspace{-0.9em} 
\section{Conclusion}
\vspace{-0.5em} 

This work proposes a novel framework, namely ERASE, to learn error-resilient presentations on graphs. The learned representations of ERASE enjoy good orthogonal properties and robustness against mislabeled nodes, both of which benefit from the decoupled label propagations and the maximizing coding rate reduction goal. Extensive experiments show the effectiveness, scalability, and great theoretical guarantees of our ERASE algorithm.

\vspace{-0.9em} 
\section*{Limitation and Future Work}
\vspace{-0.5em} 

Although this work enjoys a significant improvement margin for the accuracy in large ratio label noise scenarios, it fails to predict the true labels of some nodes, compared with small ratio label noise scenarios. 
ERASE is still a semi-supervised framework for node classification, which cannot resolve the label noise problem on graphs thoroughly. An unsupervised or self-supervised method with very limited high-quality selected annotations may be a more promising way. However, these techniques are still challenging for noise on graphs. 
We leave this as future work and believe this can also be improved by follow-up research. 

\section*{Acknowledgment}

The author team would sincerely acknowledge Haotian Zheng~(\url{https://hao-tian-zheng.github.io}), Zhiyang Liang, Yu-Kun Zhou from Xidian University, Qiongyan Wang from the University of Copenhagen for providing significant discussions and organizing the team membership. Part of this work was done when Ling-Hao Chen was at Xidian University. Ling-Hao Chen is partially supported by IDEA Research. This work is partially supported by Inspur Group Co., Ltd. and SwanHub.co. 
\vspace{-1em}
% \clearpage
% \newpage
{
    \small
    \bibliographystyle{ieeenat_fullname}
    \bibliography{main}
}
\newpage

\appendix
\clearpage
\etocdepthtag.toc{mtappendix}
\etocsettagdepth{mtchapter}{none}
\etocsettagdepth{mtappendix}{subsection}
\maketitlesupplementary
\onecolumn

\section{Related Work}
\label{Appendix:relate}

\subsection{Learning with Label Noise}

To enhance the robustness of models during training, there is a series of works proposed to handle the noisy labels. These methods can be categorized as robust loss function~\cite{patrini2017making,ghosh2017robust,feng2021can,wang2021learning,xia2023regularly}, sample selection~\cite{wei2022self,cheng2020learning,patel2023adaptive}, learning transition matrix~\cite{xue2022investigating,yao2020dual,xia2020part,xia2019anchor,li2022estimating,wu2021class2simi}, and label corrections~\cite{yi2019probabilistic,zheng2021meta,gong2022class}. Besides, some works~\cite{li2022selective,liu2020early,li2019dividemix,wei2020combating} take the combination of these approaches. However, the training objectives of these works do not model the error tolerance explicitly, which makes it hard to perform well in large ratio label noise scenarios. In ERASE, we introduce the maximizing coding rate reduction training objective to learn error-resilient representations. Moreover, these methods are not carefully designed for non-euclidean structural data, like graphs~\cite{wu2020comprehensive}. Therefore, in this paper, we introduce the decoupled label propagation algorithm to update the semantic labels of nodes with maximizing coding rate reduction for robust training.

\subsection{Graph Neural Network and Label Noise on Graph}

Graph Neural Network (GNN)~\cite{kipf2016semi,NIPS2017_5dd9db5e,wu2019simplifying} is a family of powerful models for modeling attributed graphs and is widely used in a series of downstream tasks, such as node prediction~\cite{kipf2016semi,rong2019dropedge}, link prediction~\cite{zhang2018link,rossi2021knowledge,cai2021line}, recommendation systems~\cite{zhu2022bars,su2023beyond, he2020lightgcn}, graph generation~\cite{liao2019efficient}, etc. However, for node classification, these works cannot enjoy robustness~\cite{zhou2022towards,zhuo2021training, de2020analysis, dai2022towards,dai2023unnoticeable,liu2022adversarial, lingam2023pitfalls, zhuang2022robust,dai2022comprehensive, li2022recent,dai2021nrgnn,Park_Jeong_Kim_Kim_2020,omg,10.1145/3539597.3570369,du2023noise,qian2023robust,xia2020towards,Zhang_2021_CVPR} when meeting the label noise on graph. In particular, this problem is more serious in semi-supervised node classification tasks. Accordingly, \cite{dai2021nrgnn,dai2022towards} exploited cosine similarity to connect labeled and unlabeled nodes, and~\cite{omg} takes the cosine similarity and label mix-up strategy for label denoising on graphs. PI-GNN~\cite{du2023noise} leverages adaptive confidence-aware pairwise interaction estimation and a decoupled training approach to learning with label noise. RT-GNN~\cite{qian2023robust} introduces self-reinforcement and consistency regularization as additional supervision, drawing inspiration from the memorization effects observed in deep neural networks. However, the training objectives of these methods are not designed under an error-resilient training principle. In ERASE, we introduce maximizing the coding rate reduction training objective~\cite{NEURIPS2020_6ad4174e} with a decoupled label propagation algorithm to train the model, which enjoys good properties.

\hspace{3em}
\section{Reproducibility}

We ensure our work is reproducible. We provide the \textcolor{blue}{\textbf{explanation video}}\footnote{\url{https://youtu.be/w5HwGb8bElA}} to make our method as clear as possible and easy to follow for the community. We make \textcolor{blue}{\textbf{codes}}\footnote{\url{https://github.com/eraseai/erase}} public and provide a detailed \texttt{README} tutorial document to help the community reproduce our results. Moreover, a \textcolor{blue}{\textbf{project page}}\footnote{\url{https://eraseai.github.io/ERASE-page}} is also public. We additionally provide the details of baselines, datasets, training details, and more analysis of experiments and technical designs in the following parts of the appendix.

\newpage

\section{Implementation Details}
\label{appendix:detail}

\subsection{Datasets Statistics}
\label{appendix:dataset}

The statistics of the five datasets are shown in~\cref{tabel 3}, including the number of nodes, edges, features, classes, and the dataset splits (training, validation, and testing).
\begin{table}[h]\centering
    \small
    \setlength\tabcolsep{3pt}
    \begin{tabular}{cccccccc}
    \toprule
              & \# nodes   & \# edges     & \# features & \# classes & \# training  & \# validation  & \# testing   \\ \hline
    Cora      & 2,708   & 10,556    & 1,433    & 7       & 140                        & 500                        & 1,000                       \\
    CiteSeer  & 3,327   & 9,104     & 3,703    & 6       & 120                        & 500                        & 1,000                       \\
    PubMed    & 19,717  & 88,648    & 500      & 3       & 60                         & 500                        & 1,000                       \\
        CoraFull   &19,793  & 63,421    & 8,710      & 70       & 1,400                         & 2,100                      & 1,6293                       \\
    OGBn-arxiv & 169,343 & 1,166,243 & 128      & 40      & 90,941 & 29,799 & 48,603 \\ \bottomrule
    \end{tabular}
        \caption{Statistics of datasets, including the number of nodes, edges, features, classes, and the dataset splits.}
    \label{tabel 3}
\end{table}

\subsection{Noise Setting}
\label{appendix:noise setting}
We use a transition matrix $Q_{ij}=P(\Bar{y}=j|y=i)$ given that noisy $\Bar{y}$ is flipped from clean $y$ to corrupt the labels. The definition of $Q$ is as follows.

\myPara{Symmetric noise.} The transition matrix of symmetric noise is 
\[
Q = \begin{bmatrix}
1-\phi & \frac{\phi}{n-1} & \cdots & \frac{\phi}{n-1} & \frac{\phi}{n-1} \\
\frac{\phi}{n-1} & 1-\phi & \cdots & \frac{\phi}{n-1} & \frac{\phi}{n-1} \\
\vdots & \vdots & \ddots & \ddots & \vdots \\
\frac{\phi}{n-1} & \frac{\phi}{n-1} & \cdots & \frac{\phi}{n-1} & 1-\phi
\end{bmatrix}.
\]

\myPara{Asymmetric noise.} The transition matrix of asymmetric noise is 
\[
Q = \begin{bmatrix}
1-\phi & \phi & \cdots & 0 & 0\\
\phi & 1-\phi & \cdots & 0 & 0\\
0 & \phi & 1-\phi & \cdots & 0\\
\vdots & \vdots & \ddots & \ddots & \vdots \\
\phi & 0 & 0 & \cdots & 1-\phi
\end{bmatrix}.
\]

\subsection{Details of Baselines}
\label{sec:Details of baselines}
These baselines include traditional classification loss, label noise robust loss, the latest label noise learning methods, and the latest label noise learning methods on graphs. The details of these methods are listed as follows. 
\begin{itemize}
    \item CE: We use a vanilla GAT~\cite{Liu_Zhou_2020} as the backbone and minimize the cross-entropy (CE) loss, which is widely used in node classification tasks but not robust enough to label noise.
    \item APL~\cite{pmlr-v119-ma20c} (ICML-20): It combines two robust loss functions that mutually boost each other to solve the underfitting issue of previous robust loss. We choose NCE + RCE which has good performance against label noise~\cite{pmlr-v119-ma20c}.
    \item GCE~\cite{zhang2018generalized} (NeurIPS-18): It is a theoretically grounded set of noise-robust loss functions that can be seen as an extension encompassing MAE and CE loss.
    \item MCR${}^{2}$~\cite{NEURIPS2020_6ad4174e} (NeurIPS-20): It constitutes an information-theoretic metric designed to optimize the difference in coding rates between the entirety of the dataset and the cumulative coding rates of each distinct class.
    \item Co-Dis~\cite{codis} (ICCV-23): This effective method selects data that is likely to be clean, characterized by having substantial differences in prediction probabilities between two networks.
    \item RT-GNN~\cite{qian2023robust} (WSDM-23): It exploits the mnemonic capabilities inherent in neural networks to discern nodes with accurate labels. Subsequently, pseudo-labels are derived from these identified nodes, aiming to reduce the impact of nodes with noisy labels during the training phase.
    \item PI-GNN~\cite{du2023noise} (TMLR-23): This method introduces Pairwise Intersection (PI) labels derived from evaluating feature similarity among nodes.  PI labels tend to be less corrupted than label noise and are employed to alleviate the detrimental effects of label noise which is expected to improve the robustness of the model.
\end{itemize}

\subsection{Training Details}
\label{sec:appTrainDe}
\myPara{Hyper-parameters.} We take a 2-layer GAT as the backbone. Here, $d_1$ is the dimension of the hidden layer and $d$ is the dimension of the output layer. For the sake of reproducibility, we list the settings of parameters here. The parameters such as learning rate, weight decay, and $d_1$ remain consistent when running baselines. However, it is worth noting that the dimension $d$ must be set to $K$ except running MCR${}^{2}$ and ERASE. For convenience, we set $T_1 = T_2 = T$ and $\alpha_1 = \alpha_2 = \alpha$ in our experiments. 
\begin{table}[h]\centering
    \footnotesize
    \setlength\tabcolsep{3pt}
\begin{tabular}{c|cccc|ccccc}
\toprule
         & learning rate & weight decay      & $d_1$ & $d$ & $\epsilon^{2}$ & $\gamma$ & $T$ & $\alpha$ & $\beta$ \\ \hline
Cora     & $1\times 10^{-3}$         & $5\times 10^{-4}$ & 256   & 512 & 0.05           & 2        & 5   & 0.6      & 0.6     \\
CiteSeer & $1\times 10^{-3}$         & $5\times 10^{-4}$ & 256   & 512 & 0.4            & 2        & 4   & 0.6      & 0.7     \\
PubMed   & $1\times 10^{-3}$         & $5\times 10^{-4}$ & 256   & 512 & 0.05           & 3        & 3   & 0.6      & 0.6     \\
CoraFull & $1\times 10^{-3}$         & $5\times 10^{-4}$ & 256   & 512 & 0.01           & 1        & 2   & 0.7      & 0.7     \\ \bottomrule
\end{tabular}
        \caption{Settings of parameters.}
    \label{tab:hyparameter}
\end{table}

\myPara{Model training.}
To save training time, we adopt an early stop strategy. We use the accuracy of the validation labels to determine whether training stops or not and set the full step as 150, guaranteeing that the model is fully trained. The hyper-parameters $\epsilon^2, \gamma, T, \alpha,$ and $\beta$ are tuned according to the performance of the validation set. For the OGBn-arxiv dataset, we train the model for 10 epochs for our model as the dataset is large.

\myPara{Comparison of time consumption.}
To provide the experimental time consumption of ERASE, we observed the training time consumption of an iteration, which includes forward and backward processes. We compared ERASE with two robust learning methods on graph~(PI-GNN and RT-GNN). The comparison is conducted on a machine with 14 Intel CPU kernels, 1 RTX 3090 GPU, and 80GB memory. PI-GNN, RT-GNN, and ERASE take 65ms, 355ms, and 145ms respectively, which shows that ERASE takes much less time than RT-GNN. Although ERASE takes more time than PI-GNN, ERASE shows significant improvement over PI-GNN in robustness and the extra time consumption is also acceptable. We leave developing a more efficient version of ERASE as future work. 

% The result in ~\cref{tab:comparison of time consumption} shows that ERASE takes much less time than RT-GNN and is more time-consuming than PI-GNN at the expense of outstanding robustness.

% \begin{table*}[htb]\centering
%     \footnotesize
%     \setlength\tabcolsep{1.5pt}
%     \centering
% \begin{tabular}{c|ccccc}
% \toprule
% Dataset & Cora & CiteSeer & PubMed & CoraFull & OGBn-arxiv \\ \hline

% PI-GNN  &65ms      &68ms          & 353ms       &  438ms        &             \\
% RT-GNN   &  355ms    & 463ms         & 2620ms       &   3160ms       &            \\
% \rowcolor[HTML]{E7FBFF}
% ERASE   & 145ms     &   152ms       & 923ms       &          &            \\ \bottomrule
% \end{tabular}
%         \caption{Time consumption Comparison with robust GNN against label noise. }
%         \label{tab:comparison of time consumption}
% \end{table*}

\subsection{Implementation Details of Classification Methods}
\label{sec:implementation details of classfication methods}

In~\cref{sec:Abalation_study}, we compare the implementation of the final classification methods with MLP, SVM~(L), and SVM~(P). Here, we present the implementation details of these methods. (1) \textit{MLP}. The MLP comes up with 2 100-dimension hidden layers, optimized by Adam~\cite{kingma2014adam} with a learning rate of 0.001. (2) \textit{SVM~(L)}. Linear SVM is implanted by the skit-learn~\cite{scikitlearn} official package.  (2) \textit{SVM~(P)}. Polynomial SVM is implanted by the scikit-learn~\cite{scikitlearn} official package with $degree=3$.

\newpage

\section{Performance Comparison}
\label{appendix:performance comparison}
\subsection{ERASE v.s. Baselines on Large Ratio Label Noise}
\label{sec:radar}
As shown in~\cref{fig: radar}, ERASE outperforms all the baselines with a significant margin when the noise ratio is large~($\phi=0.5$).
\begin{figure*}[h]
\captionsetup{}
    \centering

    \includegraphics[width = 0.8\textwidth]{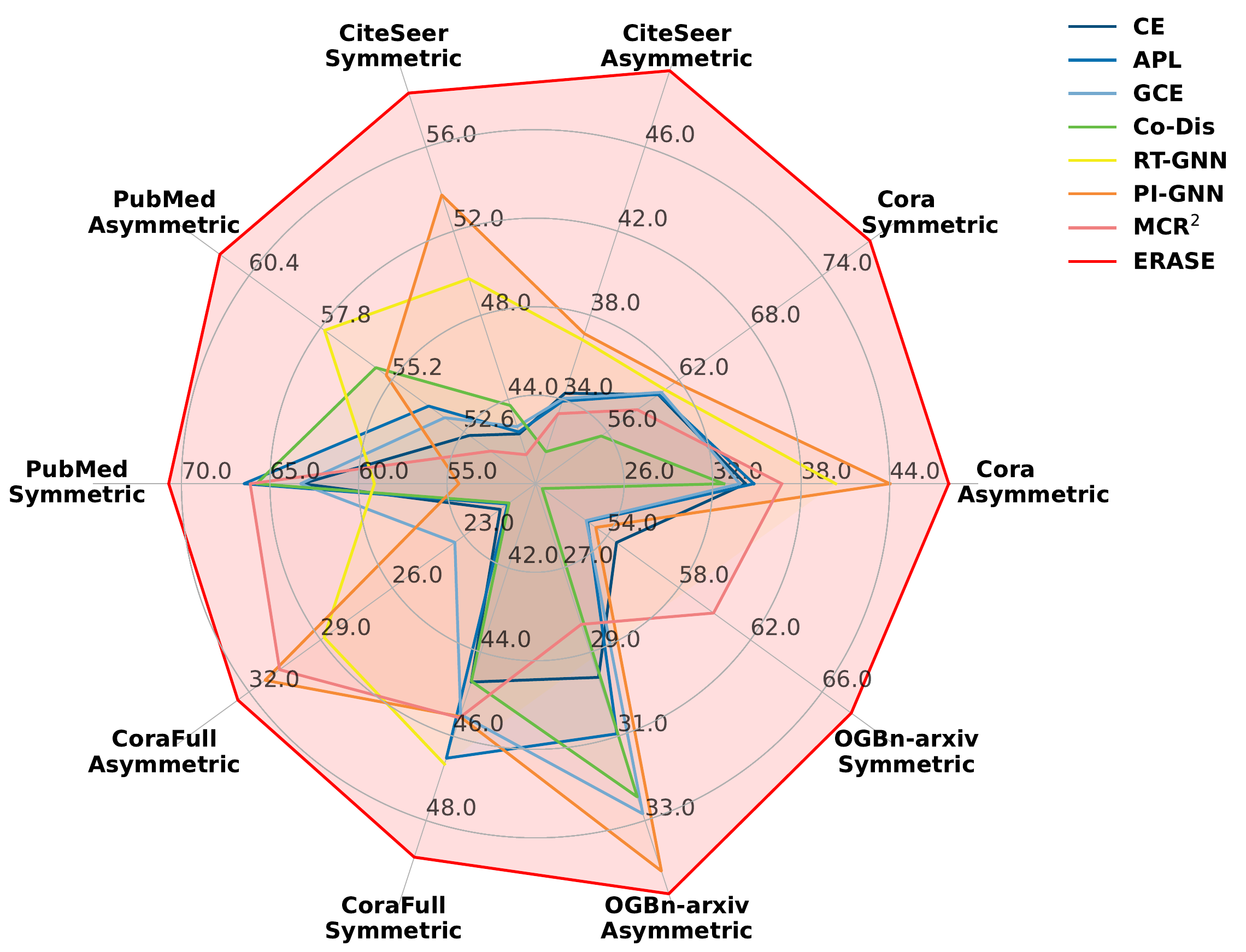}
    \caption{ERASE achieves state-of-the-art performance on 5 node classification datasets in large ratio label noise scenarios ($\phi = 0.5$).}
    \label{fig: radar}
\end{figure*}

\subsection{Gains of ERASE Compared to Baselines}
\label{sec:Gains of ERASE}

We show the performance gain of ERASE compared to the best baseline in~\cref{fig:Gain}. 
As shown in~\cref{fig:Gain}, ERASE enjoys good accuracy gains compared to the best performance of baselines especially in high noise ratio scenarios~($\phi = 0.5$).

\begin{figure}[htbp]
    \captionsetup{}
  \centering
  \begin{subfigure}{0.24\textwidth}
    \includegraphics[width=\textwidth]{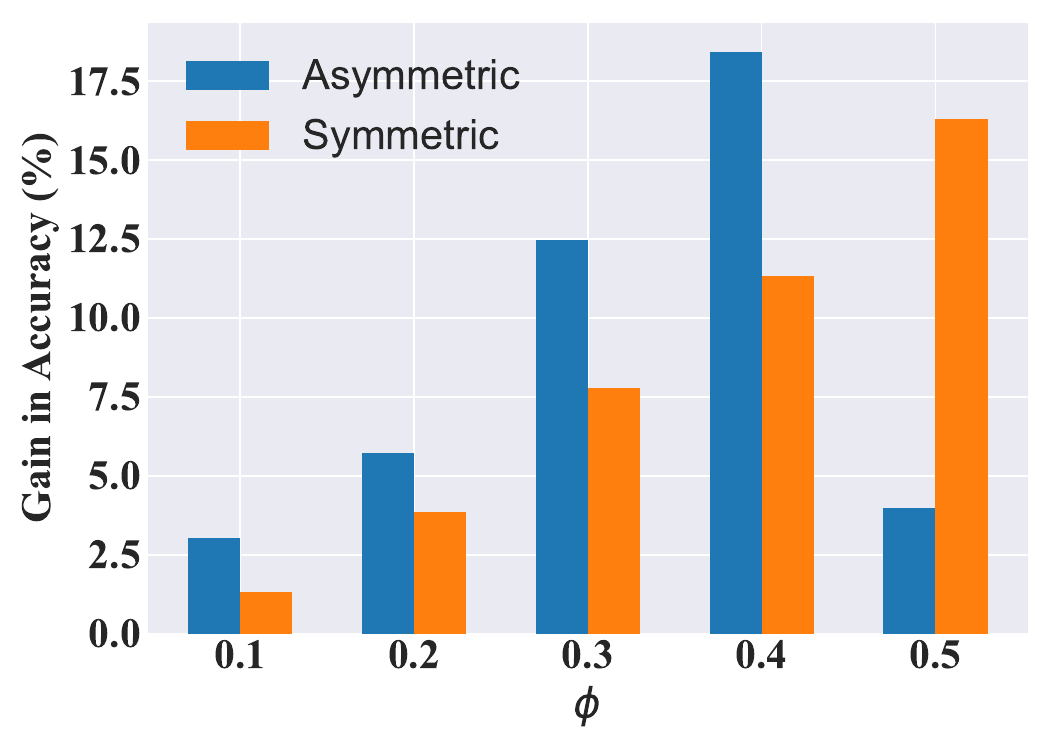}
    \caption{On Cora.}
    \label{fig:Gain_Cora}
  \end{subfigure}
  % \newline
  \begin{subfigure}{0.24\textwidth}
    \includegraphics[width=\textwidth]{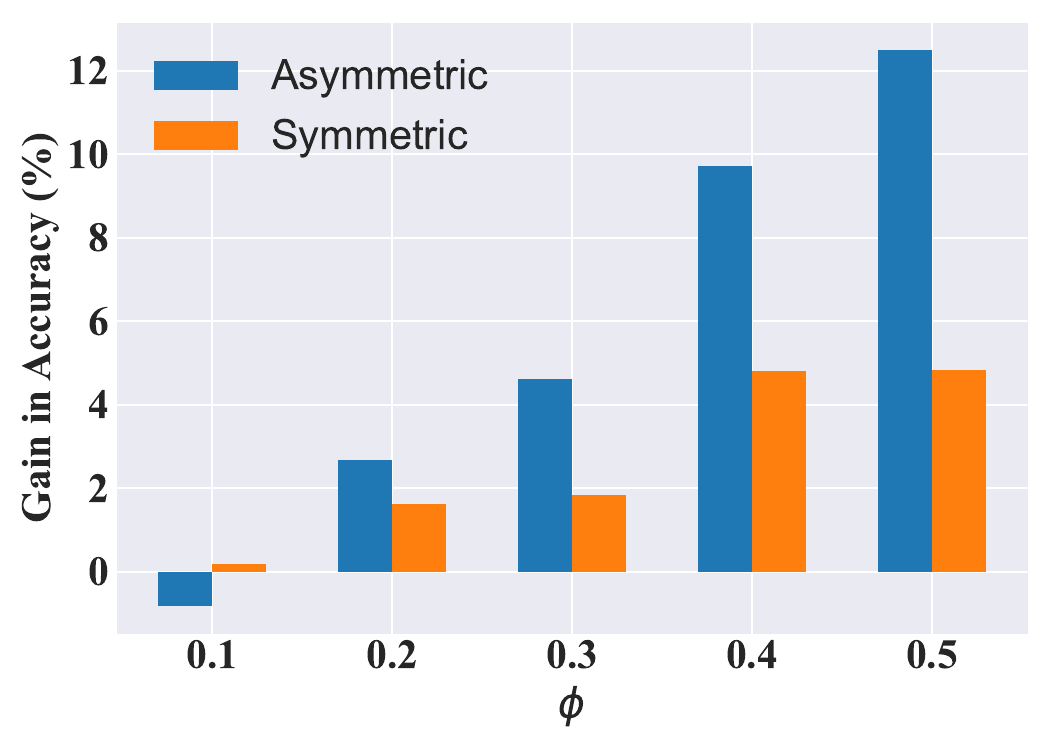}
    
    \caption{On CiteSeer.}
     \label{fig:Gain_CiteSeer}
  \end{subfigure}
  \centering
  \begin{subfigure}{0.24\textwidth}
    \includegraphics[width=\textwidth]{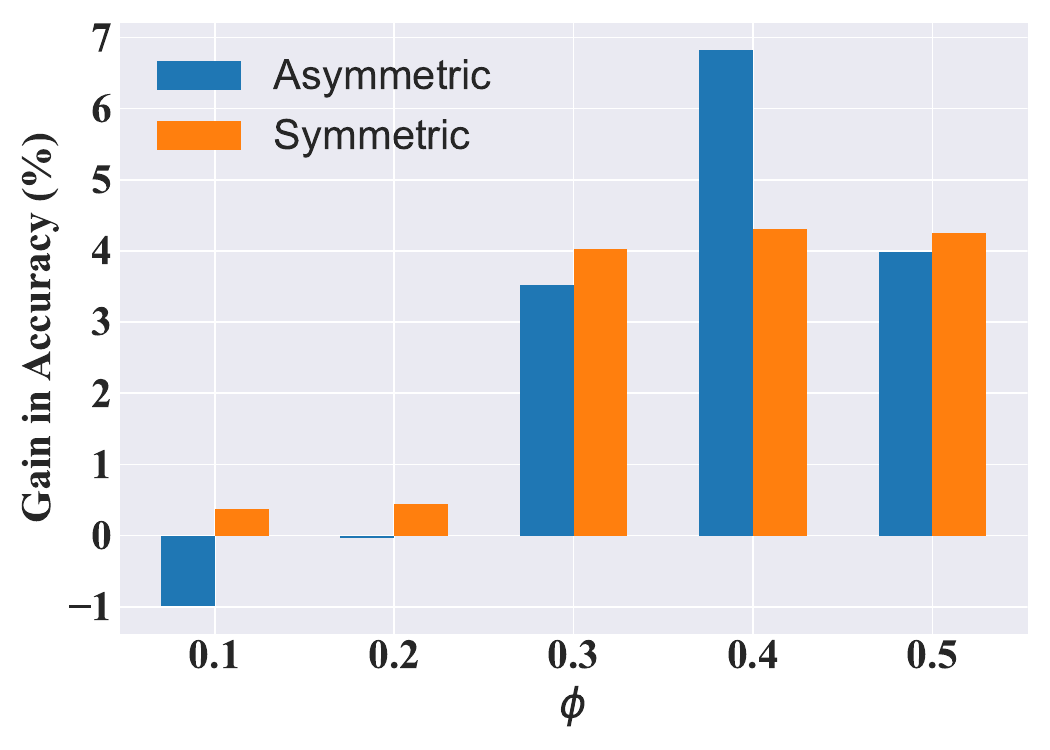}
    \caption{On PubMed.}
    \label{fig:Gain_PubMed}
  \end{subfigure}
  \begin{subfigure}{0.24\textwidth}
    \includegraphics[width=\textwidth]{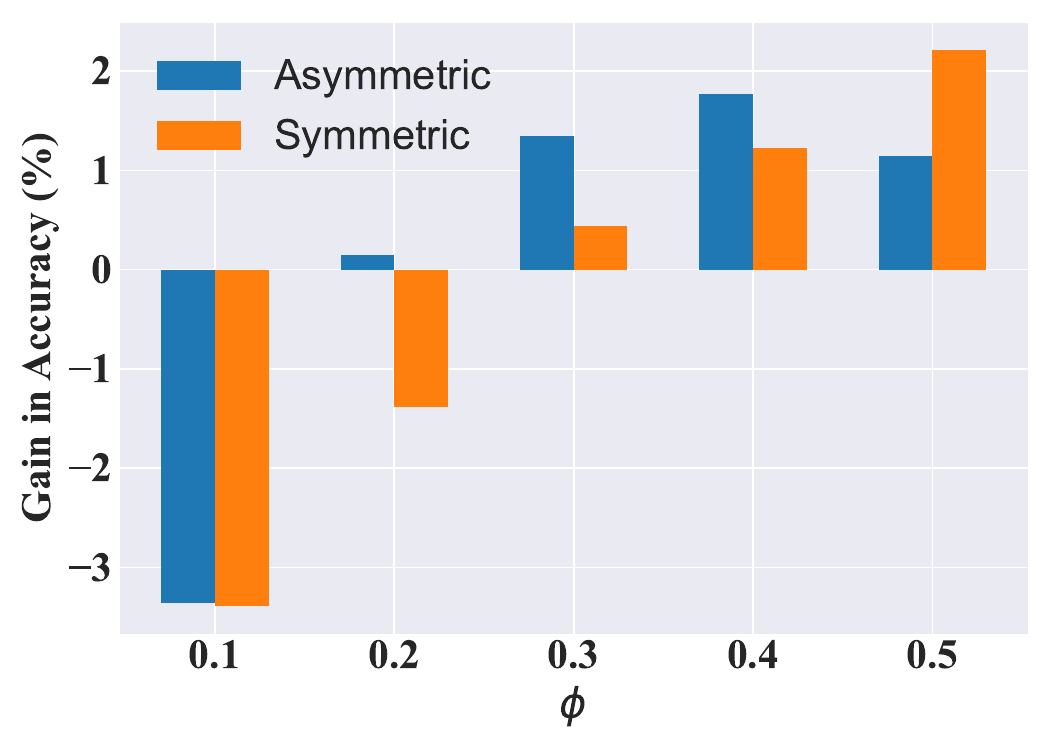}
    \caption{On CoraFull.}
     \label{fig:Gain_CoraFull}
  \end{subfigure}
  \caption{Gains of ERASE compared to the best baselines on 4 datasets.}
  \label{fig:Gain}
\end{figure}

\newpage

\section{More Discussions of Our Design}
\label{sec:more ablation}
\subsection{Definition of the Volume}
\label{sec:definition of NVTR}

\myPara{Definition of $vol(\cdot)$.} As defined in~\cite{Ma_Derksen_Hong_Wright_2007}, volume is a measurement of how large the space spanned vectors and the volume of the space covered by vectors is directly proportional to the square root of the determinant for a matrix. That is, $vol(\boldsymbol{Z}) \propto \sqrt{\det(\frac{1}{N}\boldsymbol{Z}^{\top}\boldsymbol{Z})}$, where $N$ is the number of nodes and $\boldsymbol{Z}$ is normalized to $\mathbf{0}$ ($\mathtt{mean}(\boldsymbol{Z}) = \mathbf{0}$).
According to~\cref{rm1}, if the error $\delta$ of $\tilde{\boldsymbol{z}} = \boldsymbol{z} + \delta$ and $\boldsymbol{w} \sim \mathcal{N}(\boldsymbol{0}, \frac{\epsilon^2}{d}\boldsymbol{I})$ satisfies the inequality $vol(\delta)\leq vol(\boldsymbol{w})$, the error caused by label noise is within the tolerance. For experimental observation, we trained two GNN encoders  $\mathtt{Enc}_{\Theta}(\cdot)$ and $\mathtt{Enc}_{\tilde{\Theta}}(\cdot)$ without noise and with noise respectively. Then the shift $\delta$ is calculated by $\delta = \tilde{\boldsymbol{Z}}-\boldsymbol{Z}$, where $\tilde{\boldsymbol{Z}} = \mathtt{Enc}_{\tilde{\Theta}}(\boldsymbol{X},\mathcal{E})$ and $ \boldsymbol{Z}= \mathtt{Enc}_{{\Theta}}(\boldsymbol{X},\mathcal{E})$. Therefore, the $NVTR$ metric ($\frac{vol(\delta)}{vol(\boldsymbol{w})}$) can be easily induced by the definition of volume.

\subsection{Guidance of Semantic Labels}
In ERASE, we use the semantic label of training nodes to perform the node classification. We explore whether semantic labels play an important role in the classification. 
As shown in~\cref{tab:semantic label guidance,fig:semantic label}, semantic labels are very approximate to the final performance of ERASE, reflecting semantic labels are well estimated and provide a good prior for final node classification. This verifies the soundness of semantic label modeling in ERASE.

\begin{table}[htb]\centering
    \small
    \setlength\tabcolsep{5pt}
\begin{tabular}{c|lllll|lllll}
\toprule
\multicolumn{1}{c|}{} & \multicolumn{5}{c|}{Asymmetric}       & \multicolumn{5}{c}{Symmetric}        \\ \cline{2-11} 
\multicolumn{1}{c|}{\multirow{-2}{*}{Cora}} &
  \multicolumn{1}{c}{0.1} &
  \multicolumn{1}{c}{0.2} &
  \multicolumn{1}{c}{0.3} &
  \multicolumn{1}{c}{0.4} &
  \multicolumn{1}{c|}{0.5} &
  \multicolumn{1}{c}{0.1} &
  \multicolumn{1}{c}{0.2} &
  \multicolumn{1}{c}{0.3} &
  \multicolumn{1}{c}{0.4} &
  \multicolumn{1}{c}{0.5} \\ \hline
Semantic Labels        &80.52 & 79.56 & 78.34 & 73.10 & 44.46 & 79.60 & 79.70 & 79.38 & 78.98 & 75.22 \\
\rowcolor[HTML]{E7FBFF}
ERASE                & \textbf{81.42} & \textbf{80.12} & \textbf{79.58} & \textbf{75.00} & \textbf{48.34} & \textbf{81.22} & \textbf{80.84} & \textbf{80.90} & \textbf{79.70} & \textbf{77.74} \\ \hline
\multicolumn{1}{c|}{} & \multicolumn{5}{c|}{Asymmetric}       & \multicolumn{5}{c}{Symmetric}         \\ \cline{2-11} 
\multicolumn{1}{c|}{\multirow{-2}{*}{CiteSeer}} &
  \multicolumn{1}{c}{0.1} &
  \multicolumn{1}{c}{0.2} &
  \multicolumn{1}{c}{0.3} &
  \multicolumn{1}{c}{0.4} &
  \multicolumn{1}{c|}{0.5} &
  \multicolumn{1}{c}{0.1} &
  \multicolumn{1}{c}{0.2} &
  \multicolumn{1}{c}{0.3} &
  \multicolumn{1}{c}{0.4} &
  \multicolumn{1}{c}{0.5} \\ \hline
Semantic Labels        & 68.02 & 67.22 & 67.08 & 54.80 & 46.70 & 69.38 & 67.56 & 67.34 & 66.42 & 56.58 \\
\rowcolor[HTML]{E7FBFF}
ERASE                & \textbf{70.08} & \textbf{69.00} & \textbf{69.10} & \textbf{58.86} & \textbf{50.00} & \textbf{70.24} & \textbf{69.18} & \textbf{68.90} & \textbf{67.96} & \textbf{58.08} \\ \bottomrule
\end{tabular}
\caption{Mean test accuracy of semantic labels and ERASE.}
\label{tab:semantic label guidance}
\end{table}
\begin{figure}[htbp]
    \captionsetup{}
  \centering
  \begin{subfigure}{0.24\textwidth}
    \includegraphics[width=\textwidth]{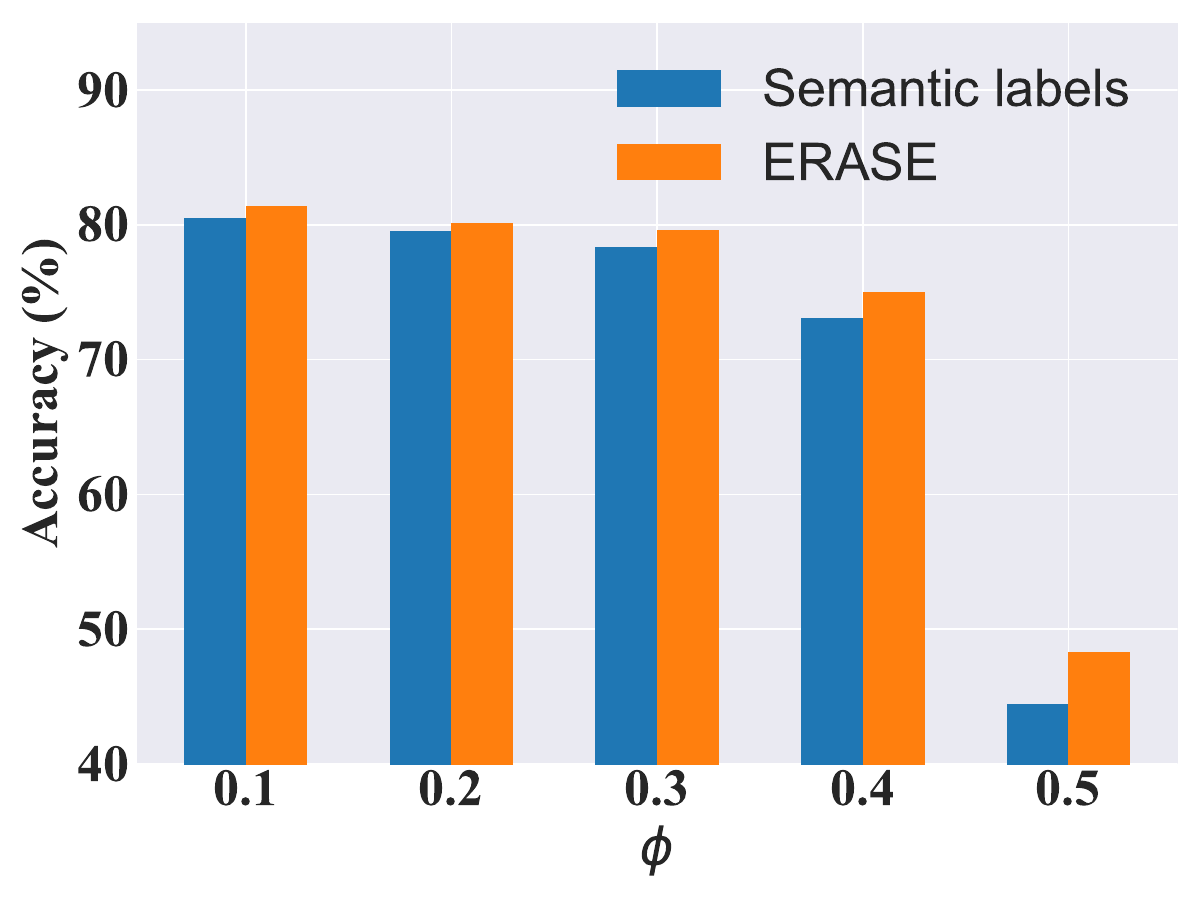}
    \caption{Cora (Asymmetric).}
    \label{fig:Semantic_Label_Guidance_Cora_Asymm}
  \end{subfigure}
  % \newline
  \begin{subfigure}{0.24\textwidth}
    \includegraphics[width=\textwidth]{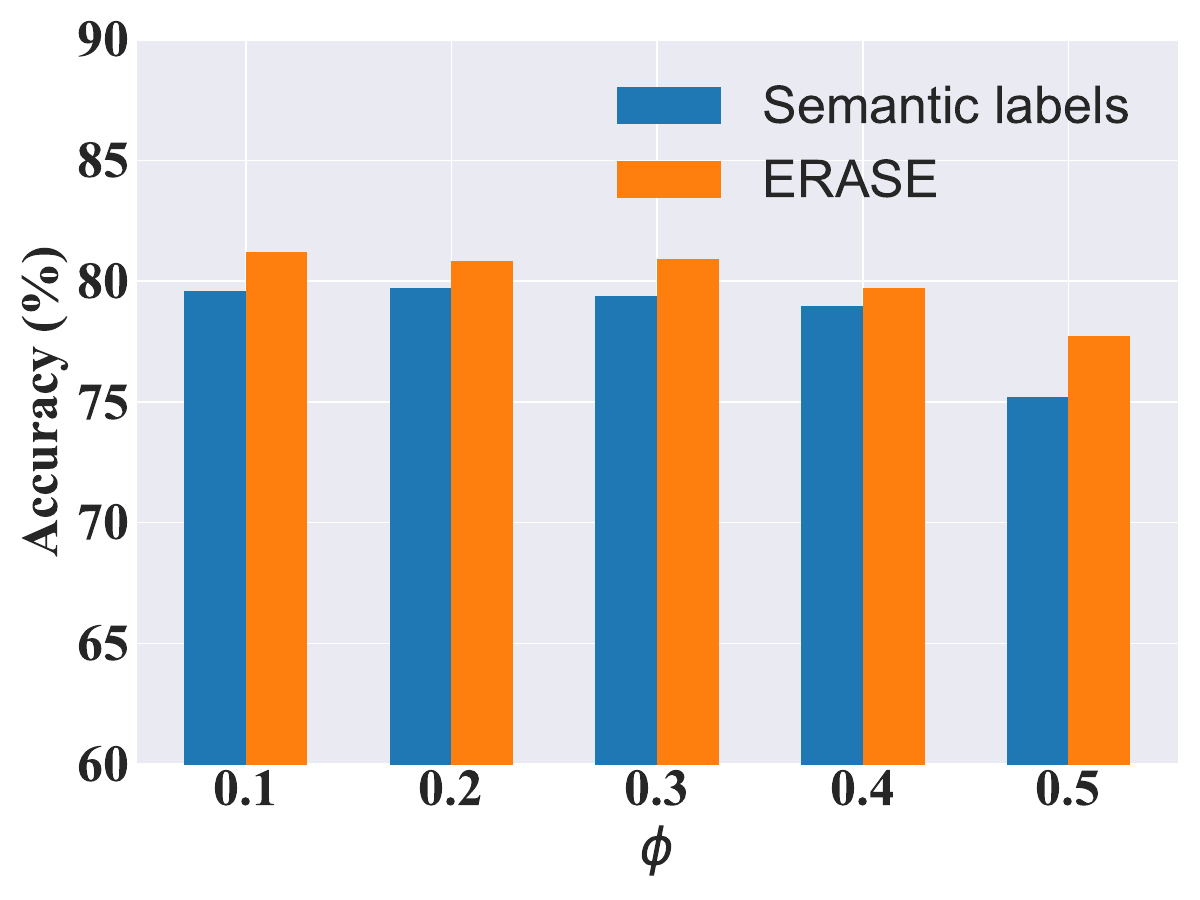}
    
    \caption{Cora (Symmetric).}
     \label{fig:Semantic_Label_Guidance_Cora_Symm}
  \end{subfigure}
  \centering
  \begin{subfigure}{0.24\textwidth}
    \includegraphics[width=\textwidth]{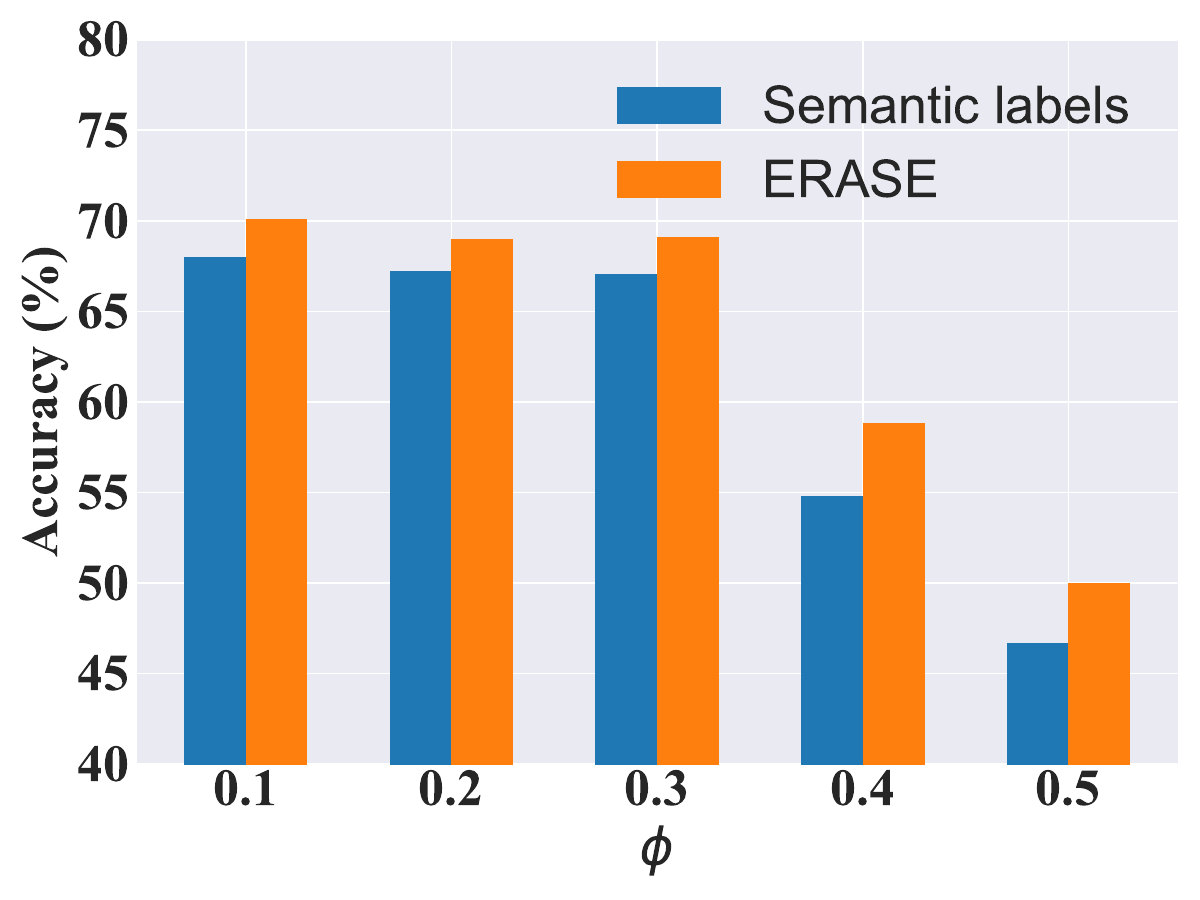}
    \caption{CiteSeer (Asymmetric).}
    \label{fig:Semantic_Label_Guidance_CiteSeer_Asymm}
  \end{subfigure}
  \begin{subfigure}{0.24\textwidth}
    \includegraphics[width=\textwidth]{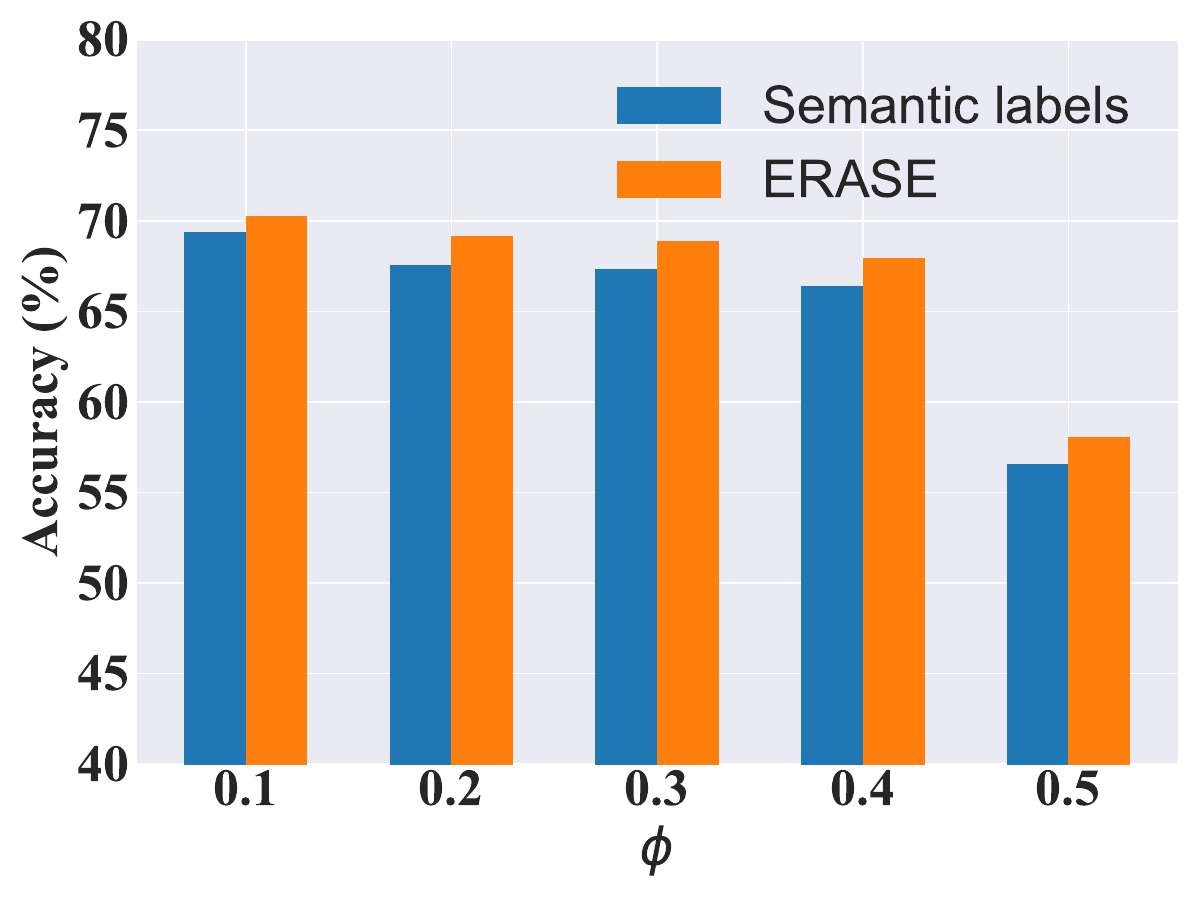}
    \caption{CiteSeer (Symmetric).}
     \label{fig:Semantic_Label_Guidance_CiteSeer_Symm}
  \end{subfigure}
  \caption{Comparison of the semantic labels and ERASE.}
  \label{fig:semantic label}
\end{figure}

\newpage
\subsection{Ablation on Propagation Depth}
\label{sec:results of propagation steps}

As discussed in~\cref{sec:Abalation_study}, our decoupled label propagation design is significant and technically sound. Here, we explore how the depth of label propagation affects the performance. The ablation results are shown in~\cref{tab:propagation steps}. As shown in~\cref{tab:propagation steps}, in low ratio label noise scenarios, the depth of label propagation affects the results marginally. Besides, in large ratio label noise scenarios, the model enjoys the scaling depth of the label propagation at first and performance is degraded when $T$ is very large due to the over-smoothing problem.

\begin{table}[h]\centering
    \small
    \setlength\tabcolsep{6.5pt}
\begin{tabular}{c|ccccc|ccccc}
\toprule
\multirow{2}{*}{\textbf{Cora}}     & \multicolumn{5}{c|}{Asymmetric}       & \multicolumn{5}{c}{Symmetric}         \\ \cline{2-11} 
                                   & 0.1   & 0.2   & 0.3   & 0.4   & 0.5   & 0.1   & 0.2   & 0.3   & 0.4   & 0.5   \\ \hline
$T=0$                                & 81.38 & 79.40 & 78.50 & 60.58 & 29.16 & 80.68 & 81.04 & 80.52 & 77.84 & 73.92 \\
$T=2$                                & \textbf{81.78} & \textbf{80.70} & \textbf{80.28} & 71.78 & 39.52 & \textbf{81.56} & \textbf{81.52} & \textbf{81.20} & 80.06 & 77.10 \\
$T=4$                                & 81.52 & 80.60 & 80.10 & 74.60 & 46.60 & 81.10 & 81.32 & 81.12 & \textbf{80.20} & 77.54 \\
$T=6$                                & 81.06 & 80.26 & 79.56 & 74.82 & 48.02 & 81.34 & 80.94 & 81.04 & 79.88 & 78.14 \\
$T=8$                                & 80.98 & 80.30 & 79.74 & \textbf{75.08} & 48.20 & 80.92 & 81.06 & 80.78 & 79.32 & \textbf{78.34} \\
$T=10$                               & 80.88 & 80.16 & 79.64 & 75.06 & 48.16 & 80.72 & 80.82 & 80.54 & 79.06 & 78.12 \\
$T=12$                               & 80.40 & 80.16 & 79.02 & 76.28 & 50.20 & 80.54 & 80.60 & 79.08 & 77.20 & 77.08 \\
$T=14$                               & 80.50 & 80.12 & 79.20 & 76.70 & \textbf{51.62} & 80.54 & 80.50 & 78.78 & 77.26 & 76.96 \\
$T=16$                               & 80.38 & 80.12 & 79.04 & 76.74 & 48.80 & 80.56 & 80.22 & 78.54 & 76.98 & 76.94 \\ \hline
\multirow{2}{*}{\textbf{CiteSeer}} & \multicolumn{5}{c|}{Asymmetric}       & \multicolumn{5}{c}{Symmetric}         \\ \cline{2-11} 
                                   & 0.1   & 0.2   & 0.3   & 0.4   & 0.5   & 0.1   & 0.2   & 0.3   & 0.4   & 0.5   \\ \hline
$T=0$                                & \textbf{70.32} & 68.62 & 66.92 & 53.76 & 34.22 & 69.54 & 68.54 & 67.38 & 68.34 & 55.50 \\
$T=2$                                & 70.12 & \textbf{69.50} & \textbf{69.18} & 58.22 & 48.30 & 69.96 & \textbf{69.18} & \textbf{69.22} & \textbf{69.92} & \textbf{59.22} \\
$T=4$                                & 70.04 & 69.30 & 68.94 & \textbf{58.86} & 49.56 & \textbf{70.18} & \textbf{69.18} & 68.92 & 67.96 & 58.06 \\
$T=6$                                & 69.46 & 68.66 & 68.62 & 58.80 & 50.00 & 69.82 & 68.60 & 67.40 & 67.92 & 56.74 \\
$T=8$                                & 69.46 & 68.28 & 68.06 & \textbf{58.86} & \textbf{50.32} & 69.78 & 68.50 & 66.76 & 67.42 & 53.82 \\
$T=10$                               & 69.26 & 68.22 & 68.16 & 58.16 & 49.40 & 70.06 & 68.14 & 66.72 & 66.76 & 54.42 \\
$T=12$                               & 68.46 & 67.84 & 68.10 & 55.74 & 49.64 & 69.14 & 67.38 & 67.32 & 66.28 & 53.62 \\
$T=14$                               & 68.68 & 68.04 & 68.08 & 55.78 & 51.34 & 69.60 & 67.38 & 67.44 & 66.30 & 53.96 \\
$T=16$                               & 68.68 & 68.16 & 68.00 & 55.90 & 50.06 & 70.14 & 67.76 & 67.42 & 66.30 & 53.26 \\ \bottomrule
\end{tabular}
\caption{Test accuracy for ERASE with different propagation depths and noise rates. The mean value of 5 runs is displayed.}
\label{tab:propagation steps}
\end{table}

\newpage

\subsection{Ablation on Label Propagation Coefficients}
As shown in~\cref{fig:ablation_study_alpha,tab:alpha}, the performances of ERASE with different $\alpha$s are stable when the noise ratio is low and the accuracy reaches the maximum within an approximate range from 0.5 to 0.8, which provides convenience for application of ERASE. Larger $\alpha$ reduces the role of label propagation and smaller $\alpha$ overtrusts the label propagation and enlarges the over-smoothing problem. 
\begin{figure}[H]
    \captionsetup{}
  \centering
  \begin{subfigure}{0.24\textwidth}
    \includegraphics[width=\textwidth]{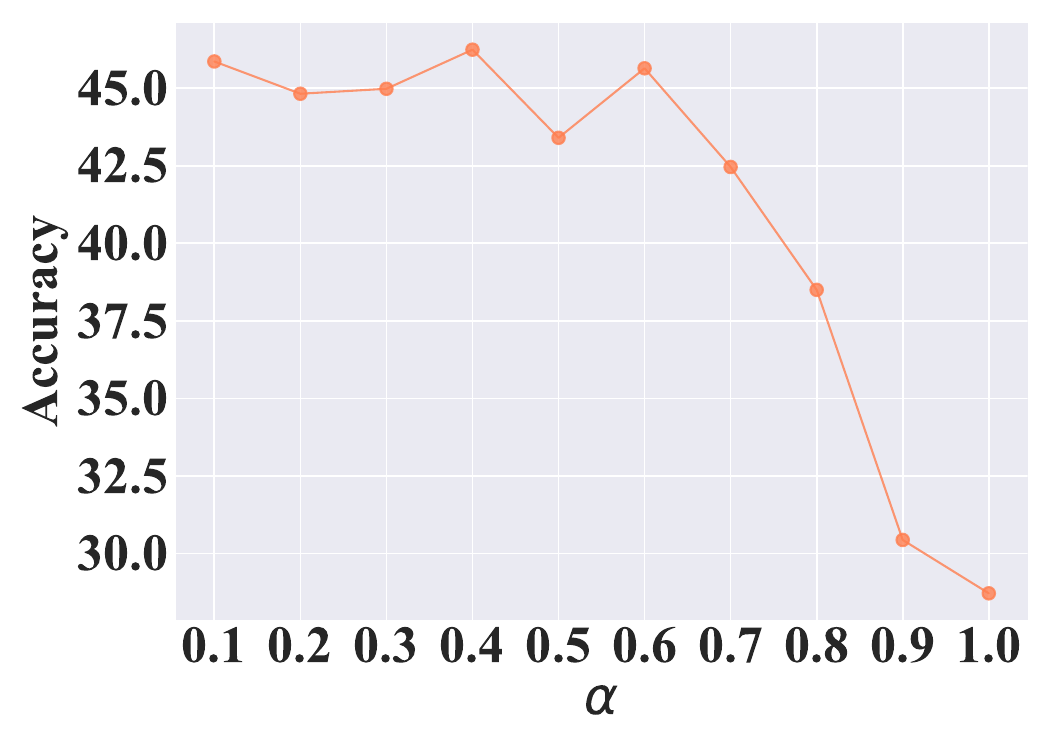}
    \caption{Cora (Asymmetric $\phi = 0.5$).}
    \label{fig:ablation_study_alpha_Cora_asymm}
  \end{subfigure}
  % \newline
  \begin{subfigure}{0.24\textwidth}
    \includegraphics[width=\textwidth]{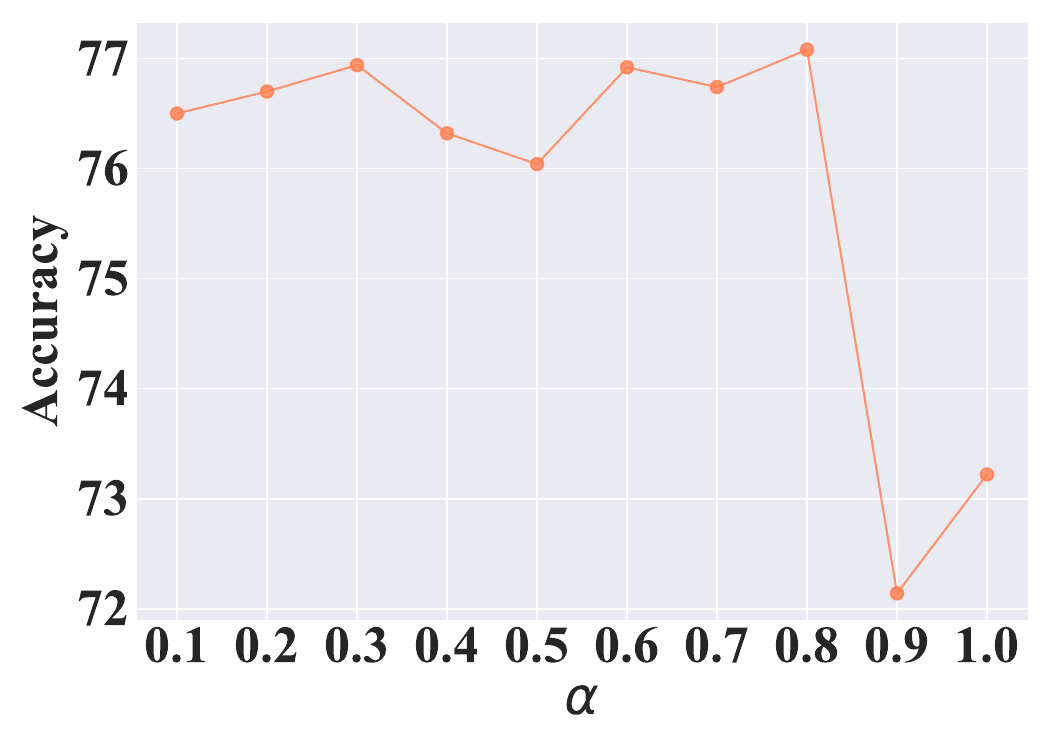}
    
    \caption{Cora (Symmetric $\phi = 0.5$).}
     \label{fig:ablation_study_alpha_Cora_symm}
  \end{subfigure}
  \centering
  \begin{subfigure}{0.24\textwidth}
    \includegraphics[width=\textwidth]{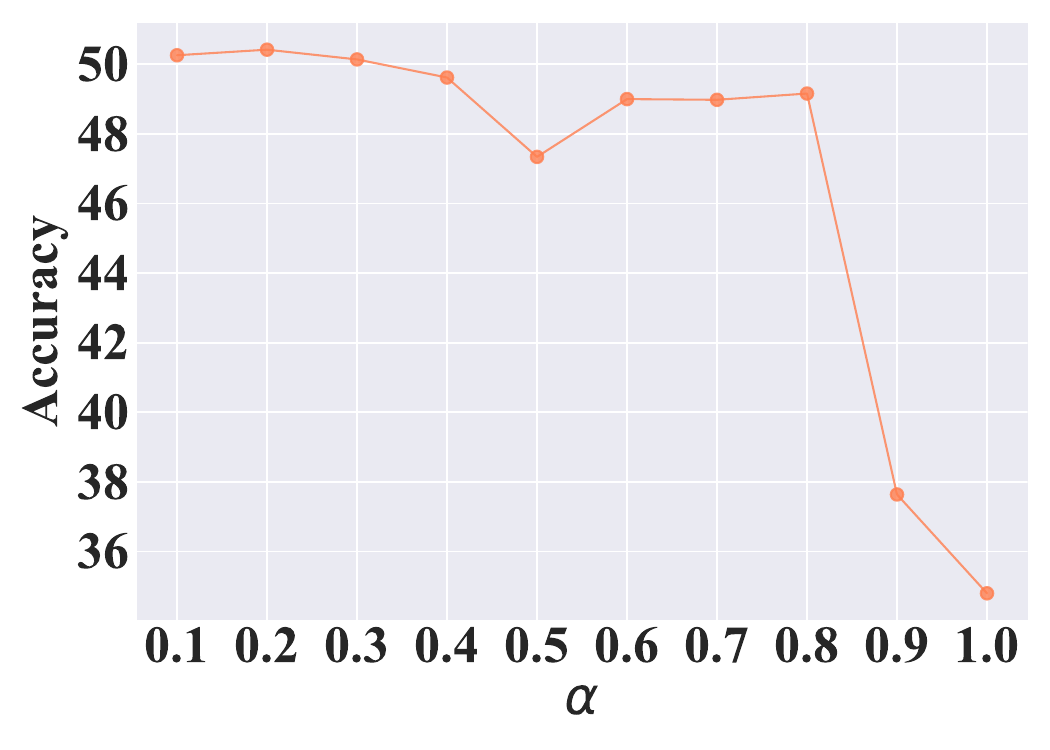}
    \caption{CiteSeer (Asymmetric $\phi = 0.5$).}
    \label{fig:ablation_study_alpha_CiteSeer_asymm}
  \end{subfigure}
  \begin{subfigure}{0.24\textwidth}
    \includegraphics[width=\textwidth]{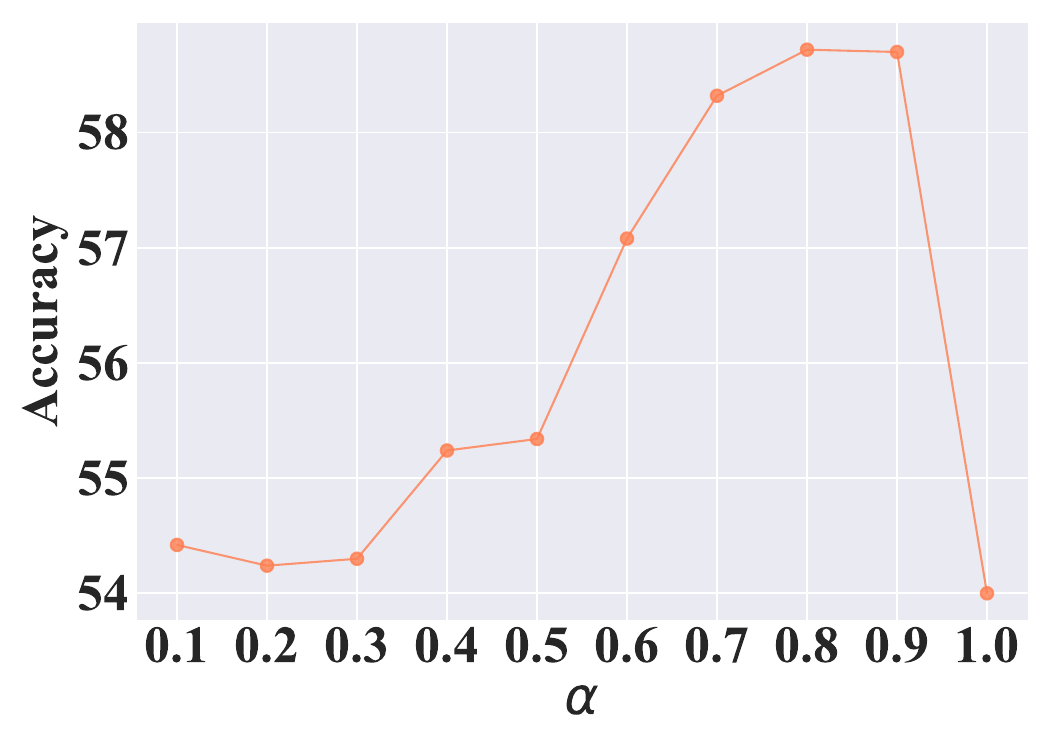}
    \caption{CiteSeer (Symmetric $\phi = 0.5$).}
     \label{fig:ablation_study_alpha_CiteSeer_symm}
  \end{subfigure}
  \caption{ Test accuracy of ERASE with different $\alpha$s on Cora (a, b) and CiteSeer (c, d) respectively.}
  \label{fig:ablation_study_alpha}
\end{figure}

\begin{table}[H]\centering
    \small
    \setlength\tabcolsep{5pt}
\begin{tabular}{c|ccccc|ccccc}
\toprule
\multirow{2}{*}{\textbf{Cora}} &
  \multicolumn{5}{c|}{Asymmetric} &
  \multicolumn{5}{c}{Symmetric} \\ \cline{2-11} 
 &
  0.1 &
  0.2 &
  0.3 &
  0.4 &
  0.5 &
  0.1 &
  0.2 &
  0.3 &
  0.4 &
  0.5 \\ \hline
$\beta=0.1$ &
  78.82 &
  79.30 &
  76.92 &
  74.50 &
  45.86 &
  79.10 &
  79.56 &
  78.90 &
  78.00 &
  76.50 \\
$\beta=0.2$ &
  78.94 &
  79.64 &
  76.66 &
  74.58 &
  44.82 &
  79.32 &
  79.24 &
  79.74 &
  78.04 &
  76.70 \\
$\beta=0.3$ &
  79.08 &
  79.26 &
  77.86 &
  74.14 &
  44.98 &
  79.38 &
  79.70 &
  79.64 &
  78.54 &
  \textbf{76.94} \\
$\beta=0.4$ &
  80.86 &
  \textbf{80.50} &
  78.76 &
  74.04 &
  \textbf{46.24} &
  79.88 &
  80.38 &
  80.20 &
  78.70 &
  76.32 \\
$\beta=0.5$ &
  79.68 &
  79.60 &
  \textbf{79.52} &
  \textbf{75.86} &
  43.40 &
  79.60 &
  80.16 &
  80.44 &
  78.38 &
  76.04 \\
$\beta=0.6$ &
  80.16 &
  79.48 &
  78.46 &
  74.44 &
  45.64 &
  79.48 &
  80.54 &
  80.60 &
  \textbf{79.12} &
  76.92 \\
$\beta=0.7$ &
  80.34 &
  79.86 &
  79.10 &
  72.28 &
  42.46 &
  \textbf{80.00} &
  \textbf{80.98} &
  \textbf{80.82} &
  78.54 &
  76.74 \\
$\beta=0.8$ &
  80.88 &
  80.22 &
  79.38 &
  69.78 &
  38.50 &
  79.90 &
  80.56 &
  80.44 &
  78.52 &
  77.08 \\
$\beta=0.9$ &
  \textbf{81.04} &
  80.14 &
  79.04 &
  62.60 &
  30.44 &
  79.76 &
  80.48 &
  80.26 &
  77.88 &
  72.14 \\
$\beta=1.0$ &
  80.56 &
  79.02 &
  77.32 &
  61.48 &
  28.72 &
  79.68 &
  79.78 &
  79.46 &
  77.64 &
  73.22 \\ \hline
\multirow{2}{*}{\textbf{CiteSeer}} &
  \multicolumn{5}{c|}{Asymmetric} &
  \multicolumn{5}{c}{Symmetric} \\ \cline{2-11} 
 &
  0.1 &
  0.2 &
  0.3 &
  0.4 &
  0.5 &
  0.1 &
  0.2 &
  0.3 &
  0.4 &
  0.5 \\ \hline
$\beta=0.1$ &
  70.12 &
  67.98 &
  68.38 &
  53.28 &
  50.26 &
  69.82 &
  68.92 &
  69.34 &
  67.64 &
  54.42 \\
$\beta=0.2$ &
  69.80 &
  68.44 &
  68.38 &
  53.12 &
  \textbf{50.42} &
  69.60 &
  68.94 &
  69.20 &
  67.82 &
  54.24 \\
$\beta=0.3$ &
  70.00 &
  67.28 &
  68.16 &
  54.70 &
  50.14 &
  69.76 &
  69.04 &
  69.42 &
  68.34 &
  54.30 \\
$\beta=0.4$ &
  69.76 &
  68.48 &
  68.16 &
  54.84 &
  49.62 &
  69.66 &
  68.68 &
  69.30 &
  68.66 &
  55.24 \\
$\beta=0.5$ &
  69.46 &
  67.72 &
  67.32 &
  55.28 &
  47.34 &
  69.60 &
  66.86 &
  69.00 &
  67.52 &
  55.34 \\
$\beta=0.6$ &
  70.14 &
  \textbf{68.70} &
  68.14 &
  \textbf{57.54} &
  49.00 &
  70.70 &
  69.22 &
  \textbf{70.10} &
  69.46 &
  57.08 \\
$\beta=0.7$ &
  70.36 &
  67.98 &
  \textbf{68.52} &
  55.32 &
  48.98 &
  70.16 &
  69.74 &
  70.06 &
  \textbf{69.70} &
  58.32 \\
$\beta=0.8$ &
  70.36 &
  68.34 &
  68.02 &
  56.00 &
  49.16 &
  \textbf{70.86} &
  \textbf{69.98} &
  69.94 &
  69.52 &
  \textbf{58.72} \\
$\beta=0.9$ &
  \textbf{70.42} &
  68.46 &
  67.94 &
  52.96 &
  37.64 &
  70.32 &
  69.66 &
  69.74 &
  69.22 &
  58.70 \\
$\beta=1.0$ &
  69.94 &
  67.52 &
  65.90 &
  50.98 &
  34.80 &
  69.78 &
  68.66 &
  68.06 &
  68.54 &
  54.00 \\ \bottomrule
\end{tabular}
\caption{Test accuracy for ERASE with differnt $\alpha$s and noise rate. The mean value of 5 runs is displayed.}
\label{tab:alpha}
\end{table}

\newpage

\subsection{Ablation on Prototype Pseudo Label}

To verify whether the prototype pseudo label contributes to the robustness of ERASE, we ablate the value of $\beta$ in~\cref{fig:ablation_study_beta,tab:semantic weight}. Results show that the accuracy remains relatively stable under low noise conditions. Besides, under high noise conditions, the accuracy also remains stable within a wide range of $\beta$. When $\beta$ is very close to 1, we observed a sharp drop in accuracy, because the prototype pseudo label is a crucial component.
\vspace{-1.2em}

\begin{figure}[H]
    \captionsetup{}
  \centering
  \begin{subfigure}{0.24\textwidth}
    \includegraphics[width=\textwidth]{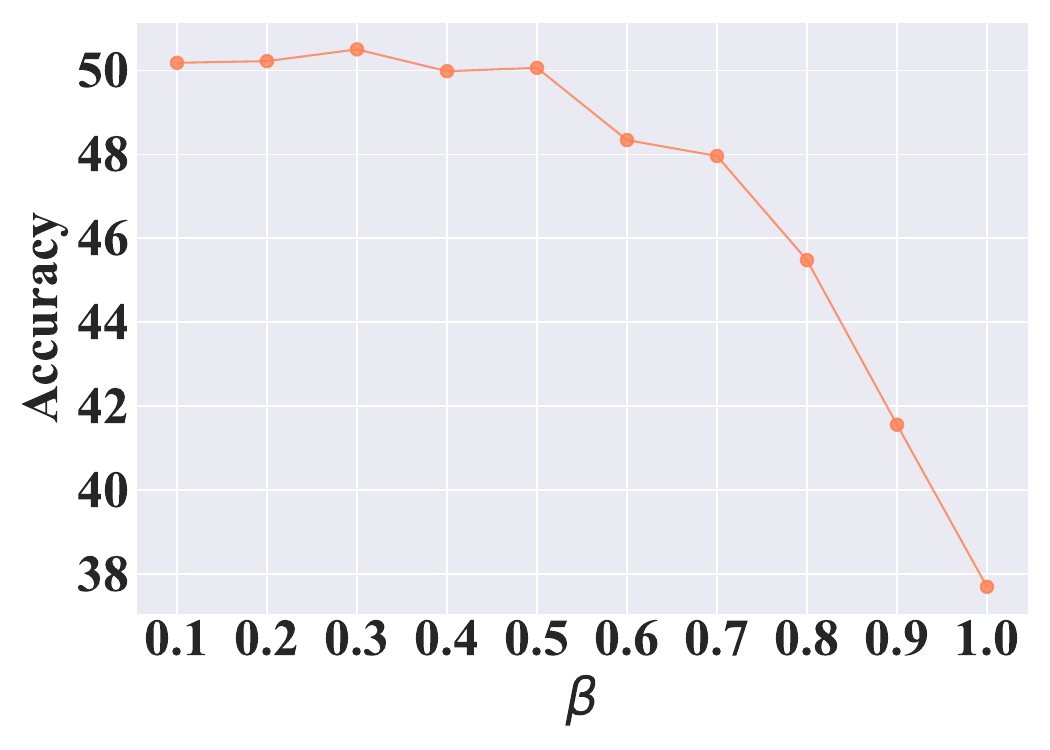}
    \vspace{-1.8em}
    \caption{Cora (Asymmetric $\phi = 0.5$).}
    \label{fig:ablation_study_beta_Cora_asymm}
  \end{subfigure}
  % \newline
  \begin{subfigure}{0.24\textwidth}
    \includegraphics[width=\textwidth]{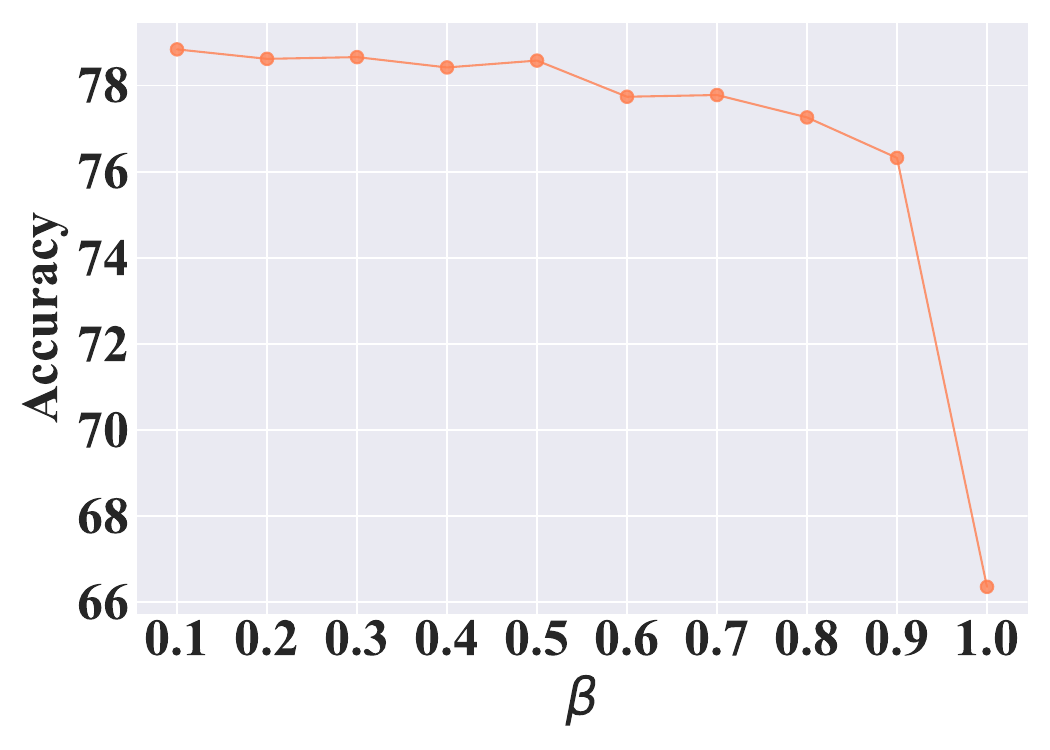}
    \vspace{-1.8em}
    \caption{Cora (Symmetric $\phi = 0.5$).}
     \label{fig:ablation_study_beta_Cora_symm}
  \end{subfigure}
  \begin{subfigure}{0.24\textwidth}
    \includegraphics[width=\textwidth]{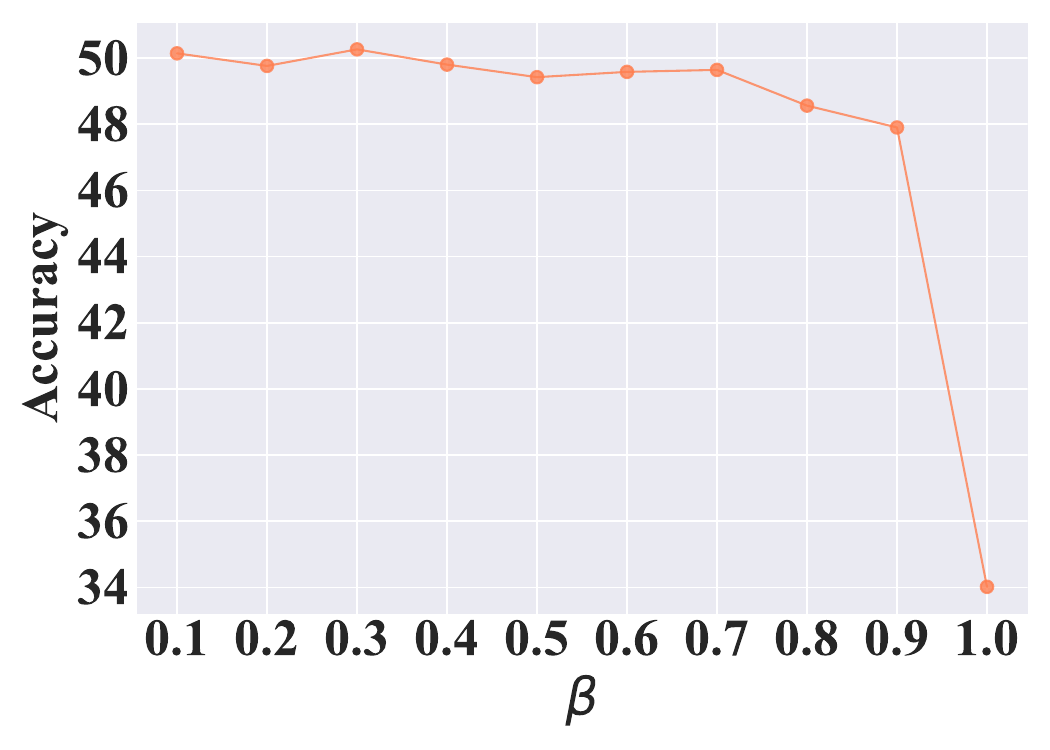}
    \vspace{-1.8em}
    \caption{CiteSeer (Asymmetric $\phi = 0.5$).}
    \label{fig:ablation_study_beta_CiteSeer_asymm}
  \end{subfigure}
  \begin{subfigure}{0.24\textwidth}
    \includegraphics[width=\textwidth]{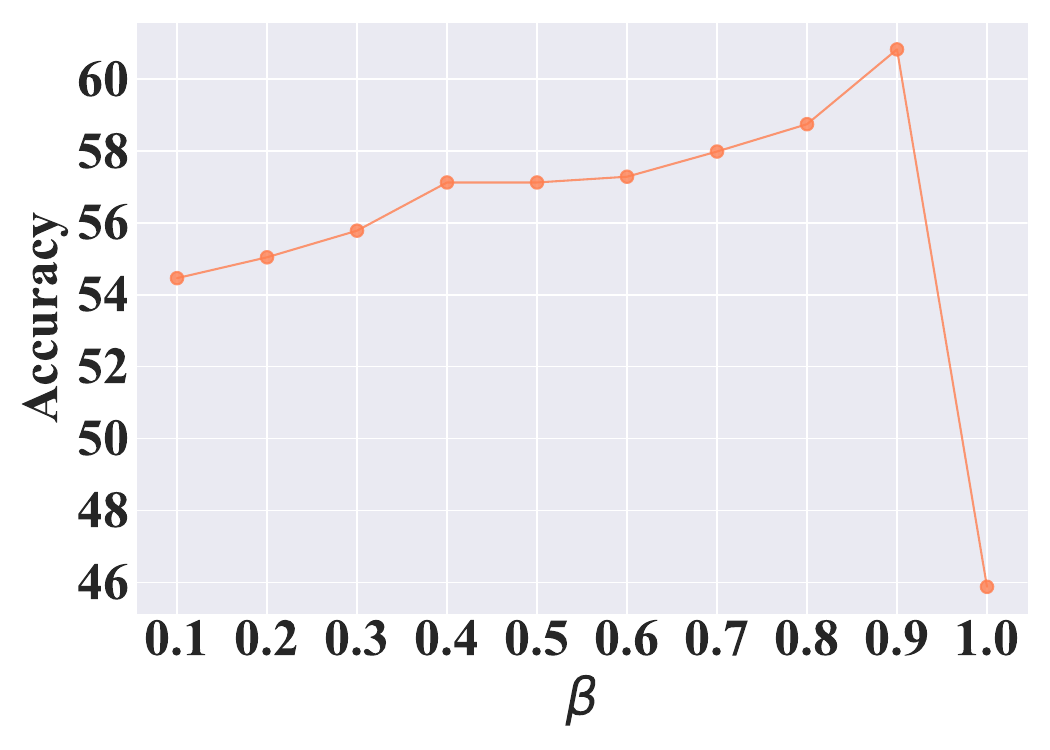}
    \vspace{-1.8em}
    \caption{CiteSeer (Symmetric $\phi = 0.5$).}
     \label{fig:ablation_study_beta_CiteSeer_symm}
  \end{subfigure}
  \vspace{-1em}
  \caption{ Test accuracy of ERASE with different $\beta$s on Cora (a, b) and CiteSeer (c, d) respectively.}
  \vspace{-1.8em}
  \label{fig:ablation_study_beta}
\end{figure}

\begin{table}[H]\centering
    \small
    \setlength\tabcolsep{6.5pt}
\begin{tabular}{c|ccccc|ccccc}
\toprule
\multirow{2}{*}{\textbf{Cora}}     & \multicolumn{5}{c|}{Asymmetric}                         & \multicolumn{5}{c}{Symmetric}                                    \\ \cline{2-11} 
                                   & 0.1            & 0.2   & 0.3   & 0.4   & 0.5            & 0.1   & 0.2            & 0.3            & 0.4   & 0.5            \\ \hline
$\beta$=0.1                        & 81.00          & 80.22 & 79.04 & 73.72 & 50.18          & 81.32 & 81.90          & 81.10          & 79.98 & \textbf{78.84} \\
$\beta$=0.2                        & 81.08          & 80.16 & 79.00 & 74.08 & 50.22          & 81.26 & \textbf{82.18} & 81.24          & 79.82 & 78.62          \\
$\beta$=0.3                        & 81.20          & 80.24 & 78.94 & 73.86 & \textbf{50.50} & 81.28 & 81.88          & 81.30          & 79.90 & 78.66          \\
$\beta$=0.4                        & 81.24          & 80.14 & 79.16 & 74.06 & 49.98          & 81.42 & 81.76          & 81.28          & 79.76 & 78.42          \\
$\beta$=0.5                        & 80.98          & 80.10 & 78.92 & 74.02 & 50.06          & 81.22 & 81.08          & \textbf{81.38} & 79.70 & 78.58          \\
$\beta$=0.6                        & 81.42          & 80.12 & 79.58 & 75.00 & 48.34          & 81.22 & 80.84          & 80.90          & 79.70 & 77.74          \\
$\beta$=0.7                        & 81.44          & 80.32 & 80.18 & 74.94 & 47.96          & 81.08 & 81.04          & 80.46          & 79.92 & 77.78          \\
$\beta$=0.8 &
  \textbf{81.46} &
  80.48 &
  80.28 &
  \textbf{75.12} &
  45.48 &
  81.20 &
  81.18 &
  80.88 &
  \textbf{80.10} &
  77.26 \\
$\beta$=0.9 &
  \textbf{81.46} &
  \textbf{80.68} &
  \textbf{80.32} &
  73.46 &
  41.56 &
  \textbf{81.48} &
  81.20 &
  80.64 &
  79.62 &
  76.32 \\
$\beta$=1.0                        & 80.36          & 78.84 & 72.80 & 60.58 & 37.70          & 78.44 & 79.76          & 77.74          & 72.58 & 66.36          \\ \hline
\multirow{2}{*}{\textbf{CiteSeer}} & \multicolumn{5}{c|}{Asymmetric}                         & \multicolumn{5}{c}{Symmetric}                                    \\ \cline{2-11} 
                                   & 0.1            & 0.2   & 0.3   & 0.4   & 0.5            & 0.1   & 0.2            & 0.3            & 0.4   & 0.5            \\ \hline
$\beta$=0.1                        & 69.56          & 68.28 & 67.82 & 56.54 & 50.14          & 69.38 & 68.78          & 66.70          & 67.30 & 54.46          \\
$\beta$=0.2                        & 69.62          & 68.20 & 68.06 & 56.98 & 49.76          & 69.22 & 68.68          & 67.28          & 67.74 & 55.04          \\
$\beta$=0.3                        & 69.44          & 68.52 & 68.44 & 57.46 & \textbf{50.26} & 69.28 & 68.96          & 67.40          & 67.84 & 55.78          \\
$\beta$=0.4                        & 69.58          & 68.76 & 68.46 & 58.28 & 49.80          & 69.52 & 68.56          & 67.54          & 67.84 & 57.12          \\
$\beta$=0.5                        & 69.50          & 68.96 & 68.84 & 58.66 & 49.42          & 69.64 & 68.68          & 68.00          & 68.14 & 57.12          \\
$\beta$=0.6                        & 69.96          & 69.14 & 69.02 & 58.32 & 49.58          & 69.80 & 69.10          & 68.34          & 68.22 & 57.28          \\
$\beta$=0.7                        & \textbf{70.06} & 69.30 & 68.94 & 58.92 & 49.64          & 70.20 & 69.16          & 68.92          & 67.90 & 57.98          \\
$\beta$=0.8 &
  69.92 &
  69.28 &
  \textbf{69.16} &
  \textbf{59.16} &
  48.56 &
  70.60 &
  69.40 &
  68.76 &
  68.18 &
  58.74 \\
$\beta$=0.9 &
  69.90 &
  \textbf{69.72} &
  69.14 &
  58.48 &
  47.90 &
  \textbf{70.74} &
  \textbf{69.74} &
  \textbf{69.52} &
  \textbf{69.50} &
  \textbf{60.82} \\
$\beta$=1.0                        & 67.86          & 61.10 & 56.04 & 50.78 & 34.02          & 66.46 & 66.62          & 64.12          & 63.16 & 45.88        \\ \bottomrule 
\end{tabular}
  \vspace{-1em}
\caption{Test accuracy for ERASE with differnt $\beta$s and noise rate. The mean value of 5 runs is displayed.}
  \vspace{-1.5em}
\label{tab:semantic weight}
\end{table}

\subsection{Ablation on Different Backbones}
\label{appendix:backbone}

As shown in~\cref{tab:detailed comparison of backbones}, ERASE enjoys different backbones and even the GCN backbone shows competitive performance against baselines in~\cref{tab:baseline_main_results} which indicates good generalization of the ERASE algorithm.
\vspace{0.4em}
\begin{table*}[h]
\centering
\setlength{\tabcolsep}{2.8pt}
\renewcommand\arraystretch{1.0}
\footnotesize
\vspace{-1.2em}
\begin{tabular}{c|ccccc|ccccc}
\toprule
\textbf{Cora}     & Asym-0.1     & Asym-0.2     & Asym-0.3     & Asym-0.4     & Asym-0.5     & Sym-0.1      & Sym-0.2      & Sym-0.3      & Sym-0.4      & Sym-0.5      \\ \hline
GCN               & 80.19 (0.95) & 79.92 (0.72) & 78.74 (1.31) & \textbf{75.53 (0.81)} & 47.31 (2.37) & 80.69 (0.78) & 81.14 (0.91) & 80.37 (1.33) & 80.07 (1.20) & \textbf{78.19 (1.29)} \\
GraphSAGE         & \textbf{81.61 (0.80)} & \textbf{80.88 (0.70)} & 79.14 (1.71) & 72.63 (3.03) & 45.12 (2.90) & 81.01 (0.99) & 80.90 (0.94) & 80.38 (1.23) & 78.53 (1.72) & 75.75 (0.77) \\ \hline

\rowcolor[HTML]{E7FBFF}
GAT               & 81.43 (0.90) & 80.46 (1.00) & \textbf{79.52 (1.13)} & 75.36 (2.32) & \textbf{48.00 (2.52)} & \textbf{81.58(0.80)}  & \textbf{81.97 (0.79)} & \textbf{81.61 (0.95)} & \textbf{80.13 (1.07)} & 78.01 (1.05) \\ \hline
\specialrule{0em}{0.5pt}{0.5pt}
\hline
\textbf{CiteSeer} & Asym-0.1     & Asym-0.2     & Asym-0.3     & Asym-0.4     & Asym-0.5     & Sym-0.1      & Sym-0.2      & Sym-0.3      & Sym-0.4      & Sym-0.5      \\
GCN               & \textbf{70.98 (1.30)} & 69.51 (2.26) & 68.79 (3.47) & 58.97 (5.58) & 48.19 (2.37) & \textbf{71.38 (1.41)} & \textbf{70.58 (1.83)} & \textbf{70.69 (2.47)} & \textbf{69.64 (2.07)} & 58.88 (5.04) \\
GraphSAGE         & 67.60 (4.27) & 65.72 (4.48) & 66.12 (3.39) & 52.61 (5.61) & 43.39 (2.26) & 68.59 (3.60) & 68.26 (3.18) & 67.47 (4.03) & 52.61 (5.61) & \textbf{59.17 (2.22)} \\\hline
\rowcolor[HTML]{E7FBFF}
GAT               & 70.70 (1.60) & \textbf{69.91 (1.79)} & \textbf{69.45 (1.84)} & \textbf{59.32 (4.69)} & \textbf{49.62 (2.20)} & 70.81 (1.34) & 69.85 (2.82) & 69.08 (2.98) & 68.37 (3.03) & 58.56 (3.30) \\ \hline
\specialrule{0em}{0.5pt}{0.5pt}
\hline
\textbf{PubMed}   & Asym-0.1     & Asym-0.2     & Asym-0.3     & Asym-0.4     & Asym-0.5     & Sym-0.1      & Sym-0.2      & Sym-0.3      & Sym-0.4      & Sym-0.5      \\
GCN               & 76.17 (1.37) & 76.10 (1.59) & 74.70 (1.84) & 68.83(2.31)  & 58.32 (4.26) & 77.22 (0.84) & 76.80 (1.12) & 74.64 (1.11) & 72.40 (1.33) & 69.30 (0.98) \\
GraphSAGE         & 77.98 (0.97) & \textbf{77.24 (1.75)} & 76.18 (1.42) & 70.33 (3.58) & 59.91 (2.17) & 77.65 (1.91) & 77.53 (1.50) & 75.34 (1.01) & 72.38 (1.13) & 69.64 (0.71) \\\hline
\rowcolor[HTML]{E7FBFF}
GAT               & \textbf{78.01 (1.10)} & 76.70 (1.07) & \textbf{76.71 (1.55)} & \textbf{73.00 (2.62)} & \textbf{61.46 (2.90)} & \textbf{78.16 (0.88)} & \textbf{77.86 (1.07)} & \textbf{75.60 (0.85)} & \textbf{72.72 (0.79)} & \textbf{70.72 (1.08)} \\ \bottomrule
\end{tabular}
\vspace{-1.3em}
\caption{Detailed comparison with different backbones in test accuracy ($\pm $ std) $\%$ with asymmetric and symmetric noise.}
\vspace{-1em}
\label{tab:detailed comparison of backbones}
\end{table*}

\newpage
\subsection{Technical Design on Two-stage Framework}
\label{sec:ablation_two_stage}
The two-stage framework is introduced for learning representations and node classification respectively. In the first stage, we train a graph encoder and obtain error-resilient representations. 
According to Lemma 1~(empirically supported in~\ref{sec:Abalation_study}), learned representations can be directly applied with a na\"ive \underline{\textit{linear}} classifier~(LogReg) to perform the classification. To verify the impact of the two-stage design, one-stage classification with linear classifier~(LogReg) for comparison is shown in~\cref{sec:Abalation_study}. Results show the necessity of the two-stage design.

\section{More Visualization Results}
\label{appendix:MoreVisRes}

\subsection{Comparison on the Cosine Similarity Matrix} 
We compare the cosine similarity of the representations learned by cross-entropy~(CE) and ERASE in~\cref{fig:cos_mat_comparison}. Results show that ERASE learns more orthogonal and robust representations between classes than CE.

\begin{figure*}[!ht]
    \captionsetup{}
  \centering
  \begin{subfigure}{0.49\textwidth}
  \centering
    \includegraphics[width=\textwidth]{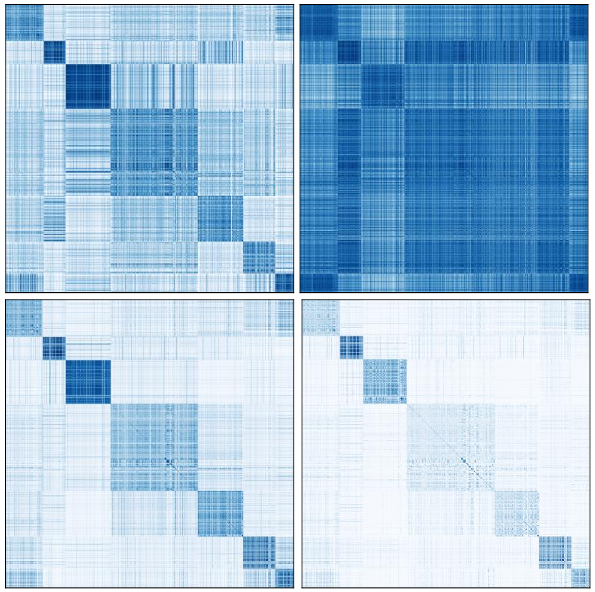}
    \caption{
  \centering Results of Cora obtained by CE~(above) and ERASE~(below). Asymmetric $\phi = 0.3$~(left) and 0.5~(right).}
  \end{subfigure}
    \centering
  \begin{subfigure}{0.49\textwidth}
  \centering
    \includegraphics[width=\textwidth]{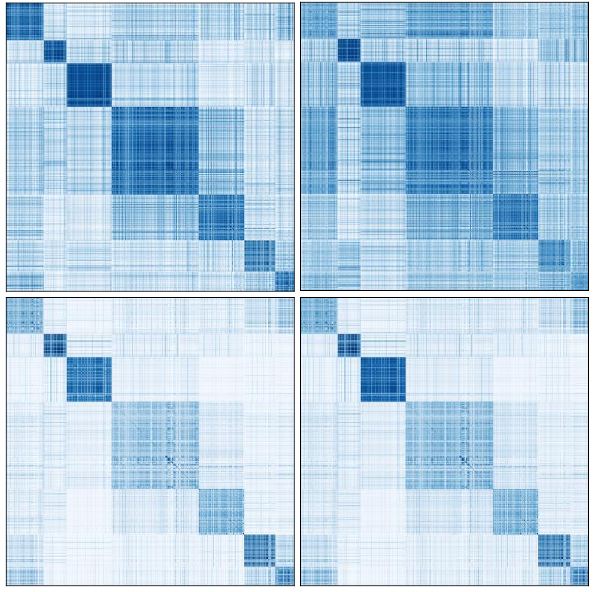}
    \caption{
  \centering Results of Cora obtained by CE~(above) and ERASE~(below). Symmetric $\phi = 0.3$~(left) and 0.5~(right).}
  \end{subfigure}
  % \newline
  \centering
  \begin{subfigure}{0.49\textwidth}
  \centering
    \includegraphics[width=\textwidth]{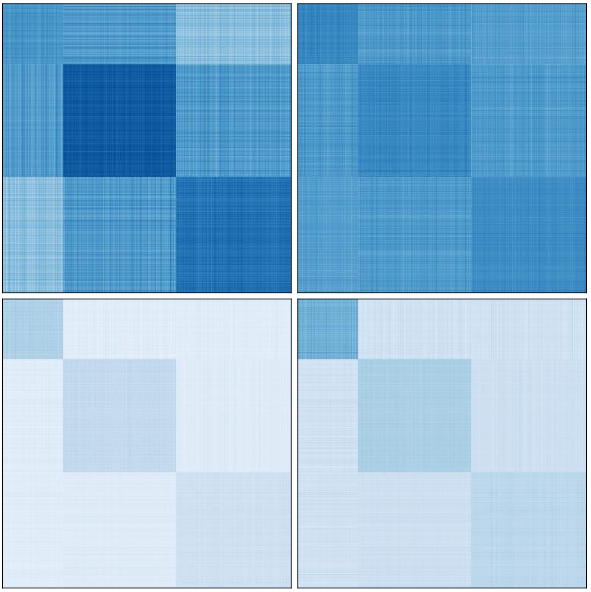}
    \caption{
  \centering Results of PubMed obtained by CE~(above) and ERASE~(below). Asymmetric $\phi = 0.3$~(left) and 0.5~(right).}
  \end{subfigure}
    \centering
  \begin{subfigure}{0.49\textwidth}
  \centering
    \includegraphics[width=\textwidth]{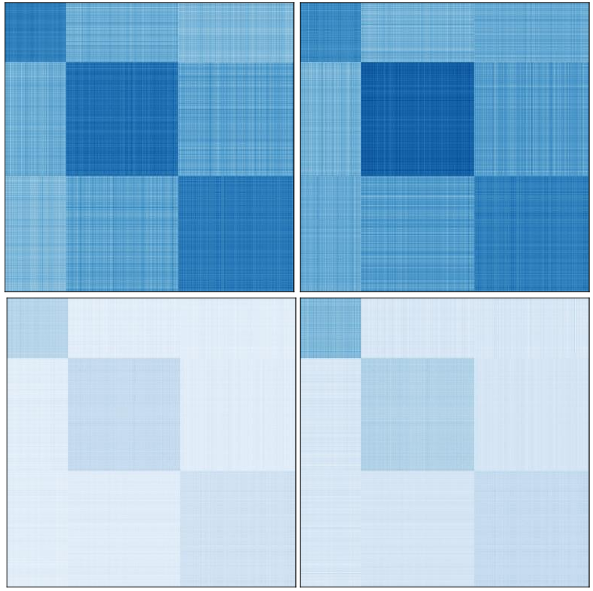}
    \caption{
  \centering Results of PubMed obtained by CE~(above) and ERASE~(below). Symmetric $\phi = 0.3$~(left) and 0.5~(right).}
  \end{subfigure}
    \vspace{-0.7em} 
  \caption{Comparison on the cosine similarity matrix.}
    \vspace{-1.1em}
  \label{fig:cos_mat_comparison}
\end{figure*}

\clearpage
\subsection{Pairwise Comparison on PCA Visualization.}
To demonstrate the orthogonal relationships between classes better, we perform PCA dimensionality reduction for visualization. Note that PCA is a linear dimensionality reduction algorithm. The results in~\cref{fig:pairwise_pca} indicate that the representations learned by ERASE show obvious orthogonal relationships between classes.

\begin{figure*}[!ht]
    \captionsetup{}
  \centering
  \begin{subfigure}{\textwidth}
  \centering
    \includegraphics[width=0.95\textwidth]{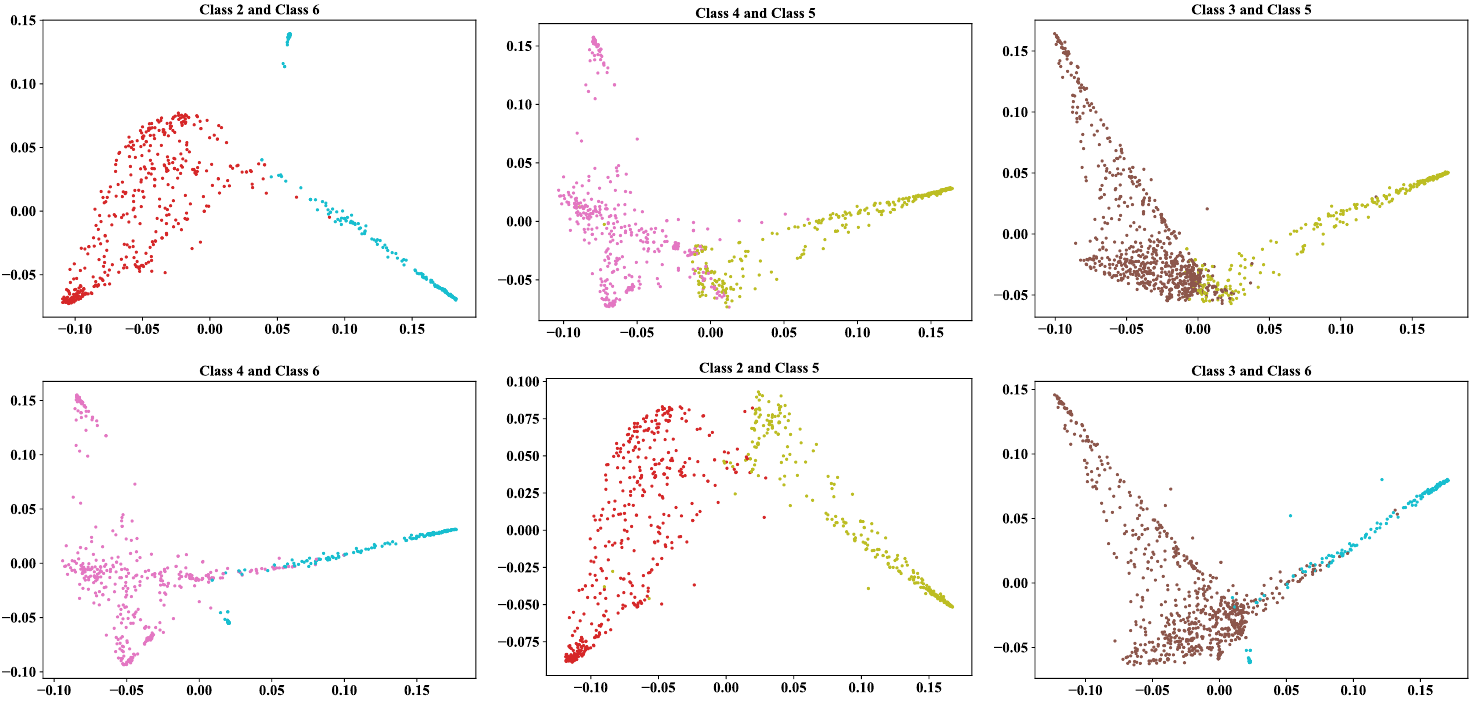}
    \vspace{-0.2em} 
    \caption{
  \centering Results of Cora. Asymmetric noise $\phi =0.3$.}
    \vspace{1em} 
  \end{subfigure}
  % \newline
  \centering
    \begin{subfigure}{\textwidth}
  \centering
    \includegraphics[width=0.95\textwidth]{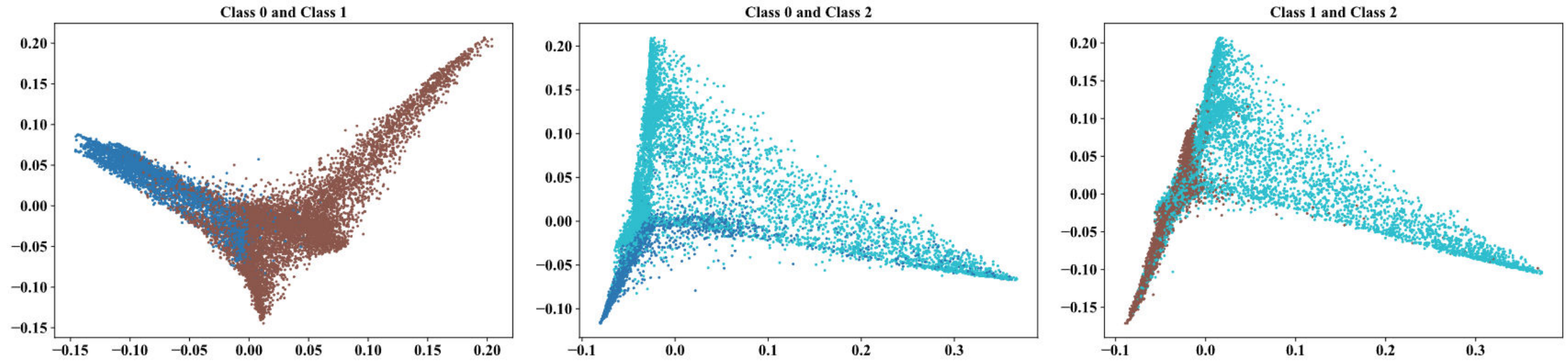}
    \vspace{-0.2em} 
    \caption{
  \centering Results of PubMed. Asymmetric noise $\phi =0.3$.}
  \end{subfigure}
    \vspace{-0.7em} 
  \caption{Pairwise PCA visualization between classes.}
    % \vspace{-2.1em}
  \label{fig:pairwise_pca}
\end{figure*}

\clearpage
\subsection{Visualization results of t-SNE.}
 We also visualized the learned representations obtained by cross-entropy (CE) and ERASE under noise via t-SNE on Cora, CiteSeer, and PubMed. The results in~\cref{fig:visualize_tsne} show that the representations obtained by ERASE are more discriminative and robust against label noise than CE.

 \vspace{-1em}
\begin{figure*}[!ht]
\captionsetup{}
    \centering
\begin{subfigure}{0.45\textwidth}
  \centering
    \includegraphics[width=\textwidth]{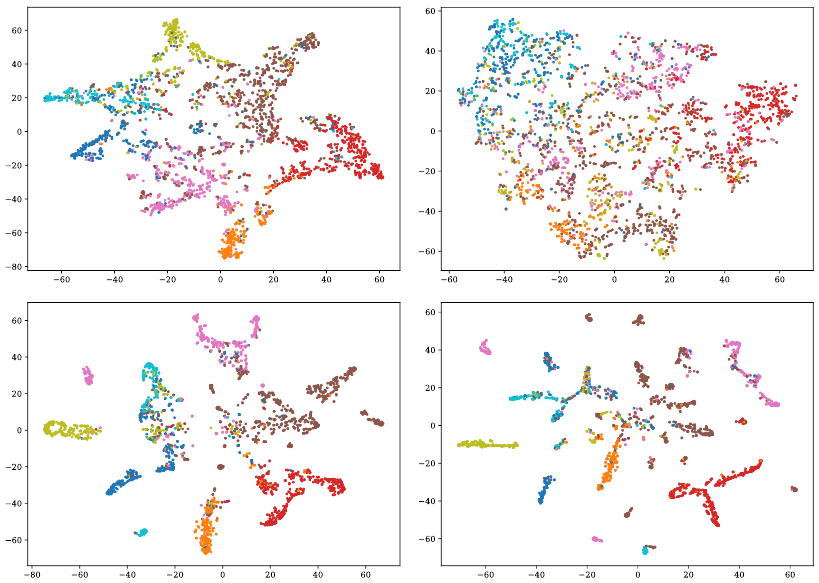}
    \vspace{-0.5em} 
    \caption{
  \centering  Results of Cora obtained by CE~(above) and ERASE~(below). Asymmetric $\phi = 0.3$~(left) and 0.5~(right).}
  \end{subfigure}
  % \newline
  \centering
    \begin{subfigure}{0.45\textwidth}
  \centering
    \includegraphics[width=\textwidth]{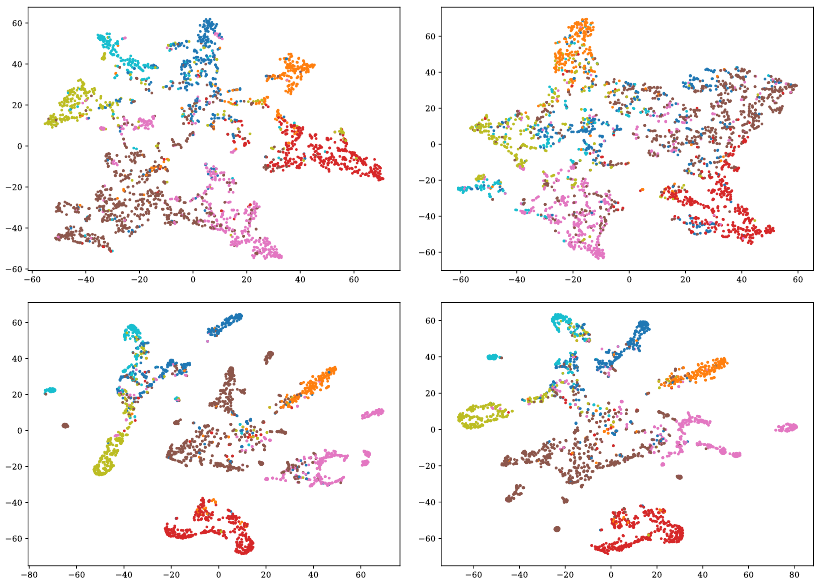}
    \vspace{-0.5em} 
    \caption{
  \centering Results of Cora obtained by CE~(above) and ERASE~(below). Symmetric $\phi = 0.3$~(left) and 0.5~(right).}
  \end{subfigure}
    \centering
    \begin{subfigure}{0.45\textwidth}
  \centering
    \includegraphics[width=\textwidth]{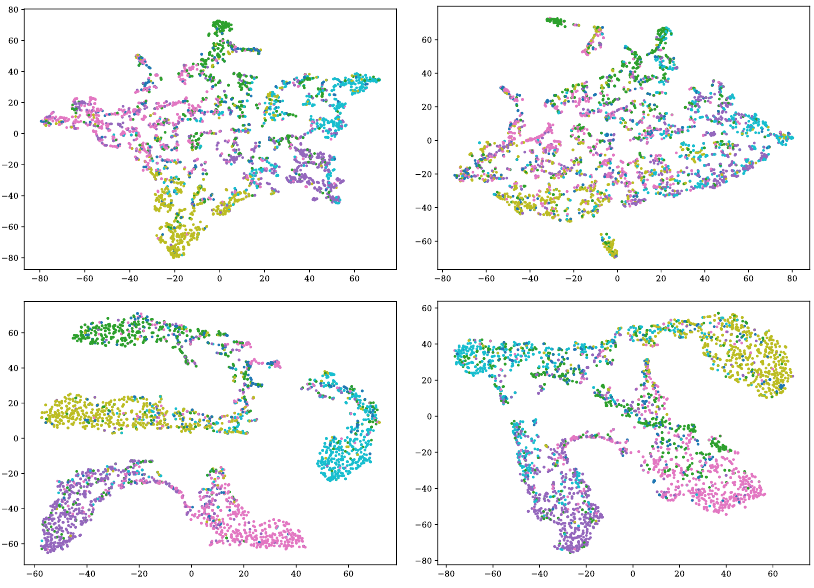}
    \vspace{-0.5em} 
    \caption{
  \centering Results of CiteSeer obtained by CE~(above) and ERASE~(below). Asymmetric $\phi = 0.3$~(left) and 0.5~(right).}
  \end{subfigure}
      \centering
    \begin{subfigure}{0.45\textwidth}
  \centering
    \includegraphics[width=\textwidth]{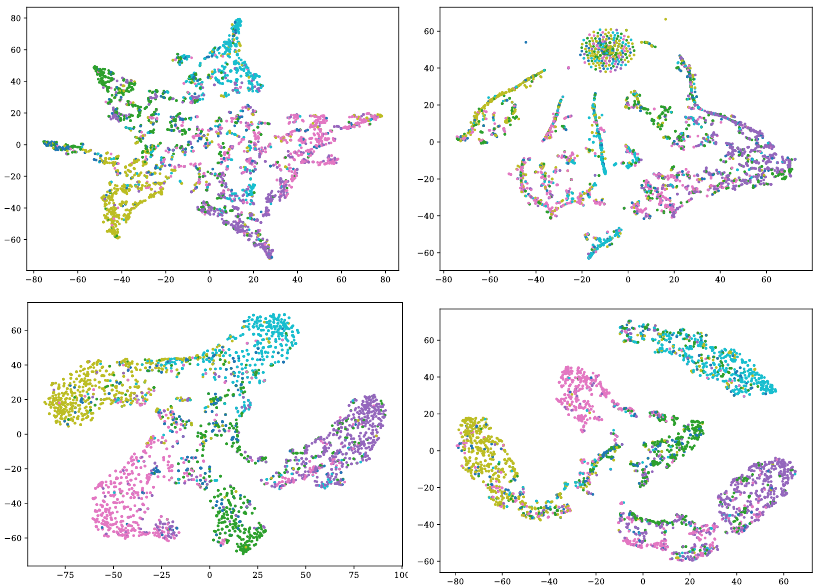}
    \vspace{-0.5em} 
    \caption{
  \centering Results of CiteSeer obtained by CE~(above) and ERASE~(below). Symmetric $\phi = 0.3$~(left) and 0.5~(right).}
  \end{subfigure}
        \centering
    \begin{subfigure}{0.45\textwidth}
  \centering
    \includegraphics[width=\textwidth]{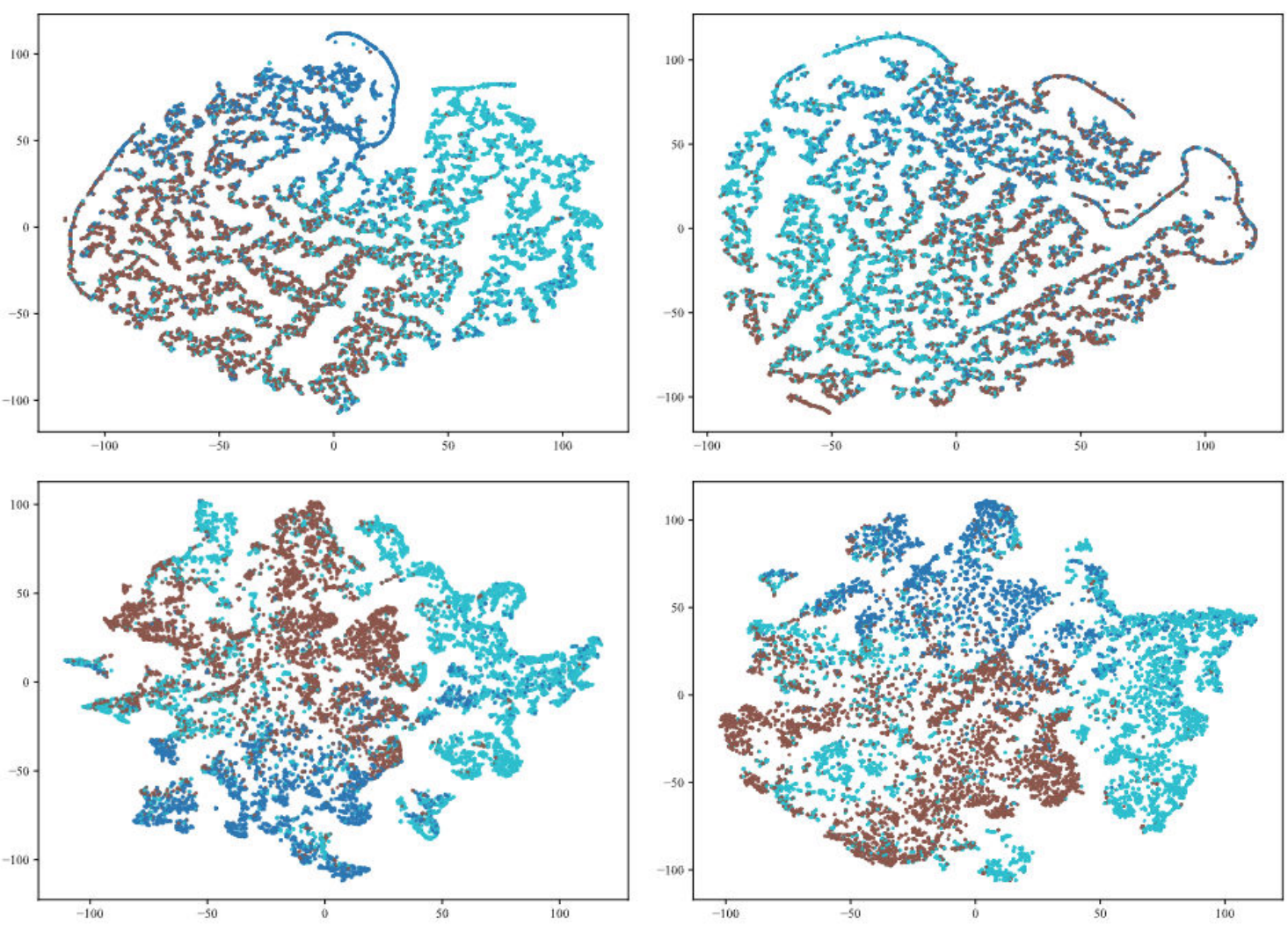}
    \vspace{-0.5em} 
    \caption{
  \centering Results of PubMed obtained by CE~(above) and ERASE~(below). Asymmetric $\phi = 0.3$~(left) and 0.5~(right).}
  \end{subfigure}
      \begin{subfigure}{0.45\textwidth}
  \centering
    \includegraphics[width=\textwidth]{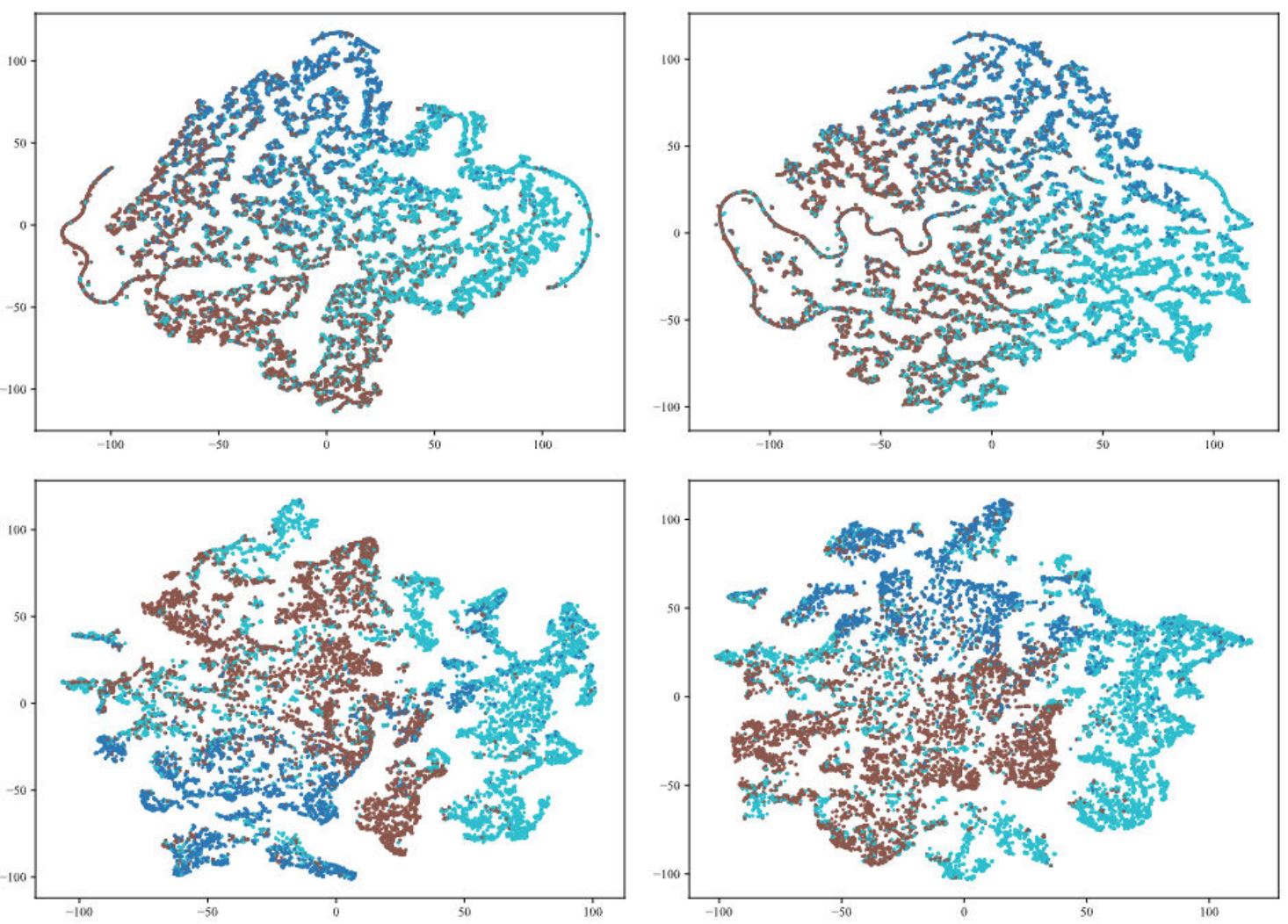}
    \vspace{-0.5em} 
    \caption{
  \centering Results of PubMed obtained by CE~(above) and ERASE~(below). Symmetric $\phi = 0.3$~(left) and 0.5~(right).}
  \end{subfigure}
    \vspace{-1em} 
    \caption{Comparison on t-SNE visualization.}
    \label{fig:visualize_tsne}
\end{figure*}

\newpage

\clearpage
\section{Case Study: Does ERASE Perform Well with Less Labeled Datasets?}
\label{sec:case_study}

As ERASE is robust to the label noise of the dataset, we are interested in whether ERASE could perform well when the datasets are less labeled. Therefore, we prune the size of the training set and compared ERASE with CE and PI-GNN. Results in~\cref{fig:case_study} show that ERASE performs better than baselines even when the datasets are less labeled which shows that ERASE is also promising to be a kind of data-efficient training fashion. 

\begin{figure}[H]
    \captionsetup{}
  \centering
  \begin{subfigure}{0.4\textwidth}
    \includegraphics[width=\textwidth]{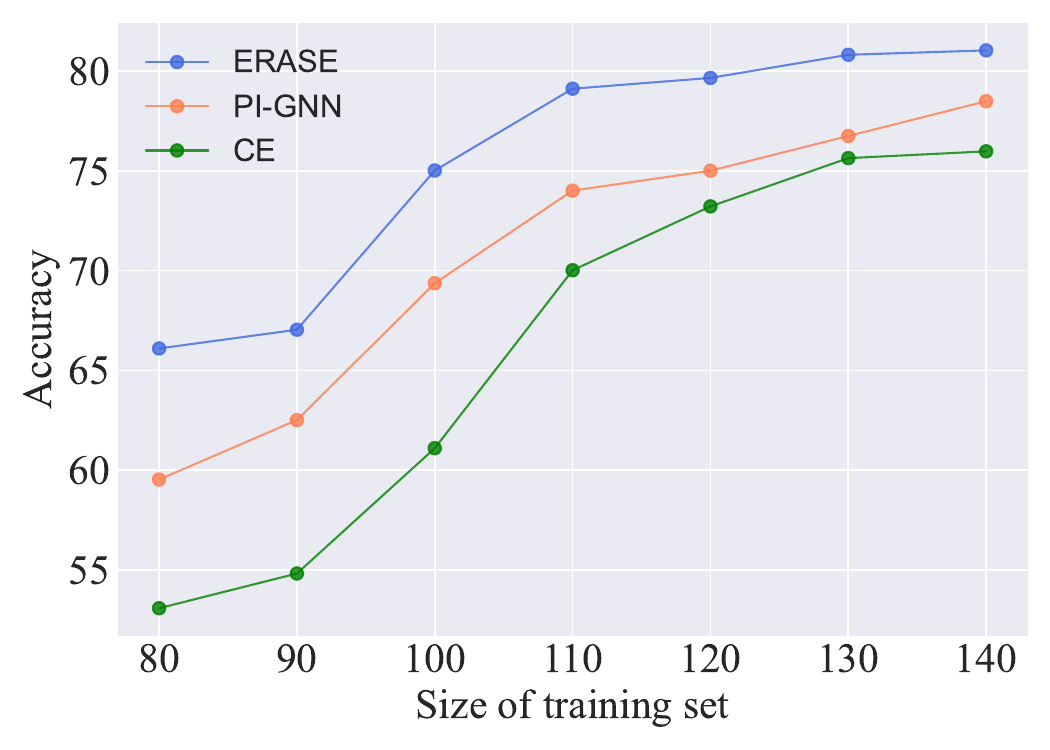}
    \caption{Results of Cora.}
  \end{subfigure}
  % \newline
  \begin{subfigure}{0.4\textwidth}
    \includegraphics[width=\textwidth]{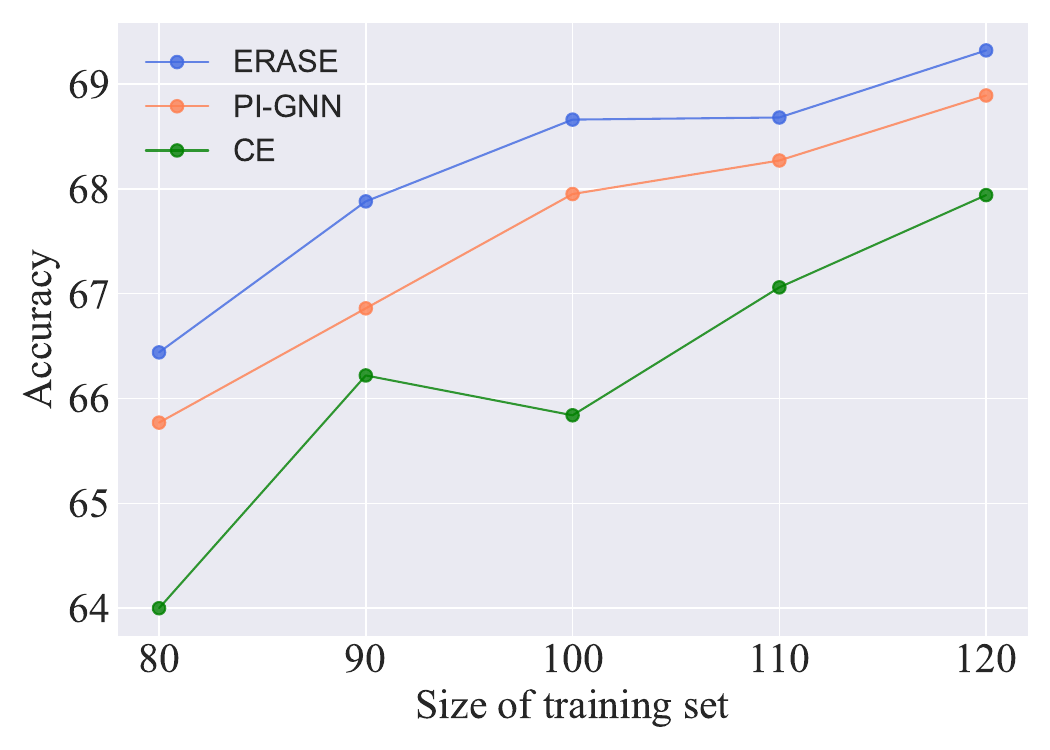}
    \caption{Results of CiteSeer.}
  \end{subfigure}
  \caption{Test accuracy with different sizes of training set on Cora and CiteSeer.}
  \label{fig:case_study}
\end{figure}


\end{document}